\documentclass[letterpaper]{article}
\usepackage{jmlr2e}
\usepackage{times}
\usepackage{helvet}
\usepackage{courier}
\usepackage{pdfsync}

\usepackage{graphicx} 
\usepackage{subfigure}
\usepackage{sidecap}
\usepackage{tabularx}
\usepackage{pbox}
\usepackage{csquotes}
\usepackage{booktabs}

\usepackage{amsmath, amssymb}

\usepackage{algorithm}
\usepackage[noend]{algpseudocode}

\input{Definitions}

\newcommand{\nsamples}{T}
\newcommand{\sampiter}{t}
\newcommand{\regwgt}{\alpha}
\newcommand{\ldim}{k}
\newcommand{\ydim}{m}
\newcommand{\xdim}{d}
\newcommand{\hscale}{s}

\newcommand{\stepsize}{\eta}
\newcommand{\nuh}{\nu_{\textrm{H}}}
\newcommand{\nud}{\nu_{\textrm{D}}}

\newcommand{\avec}{\mathbf{a}}
\newcommand{\bvec}{\mathbf{b}}
\newcommand{\dvec}{\mathbf{d}}

\newcommand{\hvec}{\mathbf{h}}
\newcommand{\uvec}{\mathbf{u}}
\newcommand{\vvec}{\mathbf{v}}
\newcommand{\xvec}{\mathbf{x}}

\newcommand{\zvec}{\mathbf{z}}

\newcommand{\Amat}{\mathbf{A}}
\newcommand{\Bmat}{\mathbf{B}}
\newcommand{\Cmat}{\mathbf{C}}
\newcommand{\Diagmat}{{\boldsymbol{\Lambda}}}
\newcommand{\Dalt}{\bar{\mathbf{D}}}
\newcommand{\Dmat}{\mathbf{D}}
\newcommand{\Dmatone}{\mathbf{D}^{(1)}}
\newcommand{\Dmattwo}{\mathbf{D}^{(2)}}

\newcommand{\Gmat}{\mathbf{G}}
\newcommand{\Halt}{\bar{\mathbf{H}}}
\newcommand{\Hmat}{\mathbf{H}}
\newcommand{\joint}{\mathbf{Z}}

\newcommand{\Mmat}{\mathbf{M}}
\newcommand{\Qmat}{\mathbf{Q}}
\newcommand{\Smat}{\mathbf{S}}
\newcommand{\Umat}{\mathbf{U}}
\newcommand{\Vmat}{\mathbf{V}}
\newcommand{\Wmat}{\mathbf{W}}
\newcommand{\Xmat}{\mathbf{X}}
\newcommand{\Ymat}{\mathbf{Y}}

\newcommand{\Bmatopt}{\bar{\Bmat}}

\newcommand{\Gammamat}{\boldsymbol{\Gamma}}
\newcommand{\Lambdamat}{\boldsymbol{\Lambda}}

\newcommand{\Sigmamat}{\boldsymbol{\Sigma}}

 \newcommand{\Ds}{{\Dmat^{s}}}
 \newcommand{\Dp}{{\Dmat^{p}}}
 
  \newcommand{\Hs}{{\Hmat^{s}}}
 \newcommand{\Hp}{{\Hmat^{p}}}

\newcommand{\trace}[1]{\tr\left(#1\right)}
\newcommand{\jointloss}{\mathrm{L}} 
 \newcommand{\hstack}{\text{ hstack}}
 \newcommand{\vstack}{\text{ vstack}}

\newcommand{\Lx}{L_x}
\newcommand{\regd}{{\mathrm{R}_{\mathrm{\scriptsize D}}}}
\newcommand{\regh}{{\mathrm{R}_{\mathrm{\scriptsize H}}}}
\newcommand{\regz}{{\mathrm{R}_{\mathrm{\scriptsize Z}}}}
\newcommand{\regzf}[1]{\regz(#1)}

\newcommand{\fcol}{f_c}
\newcommand{\frow}{f_r}

\newcommand{\Qmatopt}{\bar{\Qmat}}

\newcommand{\inv}{{\raisebox{.2ex}{$\scriptscriptstyle-1$}}}
\newcommand{\tinv}{{\raisebox{.2ex}{$\scriptscriptstyle-\top$}}}

\newcommand{\genrankset}{\Dmat \in \RR^{\xdim \times \ldim},\Hmat \in \RR^{\ldim \times \nsamples}}
\newcommand{\rrankset}{\Dmat \in \RR^{\xdim \times \xdim},\Hmat \in \RR^{\xdim \times \nsamples}}

 \newcommand{\pairvar}{\Ymat}

 \newcommand{\closs}{g}
\newcommand{\rdim}{r}

\newcommand{\seqiter}{i}

 \newcommand{\pboxspace}{\vspace{0.1cm}}

\renewcommand{\cite}[1]{\citep{#1}}
\newcommand{\defeq}{=}

\newcommand{\RFM}{DLM}
\newcommand{\RFMs}{DLMs}

\jmlrheadingsubmission{1}{2017}{0}{08/17}{00}{Lei Le and Martha White}
\ShortHeadings{Identifying global optimality for dictionary learning}{Le and White}
\firstpageno{1}

\begin{document}

\title{Identifying global optimality for dictionary learning}

\author{\name Lei Le \email leile@indiana.edu \\ \addr Department of Computer Science\\ Indiana University\\Bloomington, IN 47405 \AND \name Martha White \email martha@indiana.edu \\ \addr Department of Computer Science\\ Indiana University\\ Bloomington, IN 47405}

\editor{}

\maketitle

\begin{abstract}
Learning new representations of input observations in machine learning is often tackled using a factorization of the data. For many such problems, including sparse coding and matrix completion, learning these factorizations can be difficult, in terms of efficiency and to guarantee that the solution is a global minimum. Recently, a general class of objectives have been introduced---which we term induced dictionary learning models (\RFMs)---that have an induced convex form that enables global optimization. Though attractive theoretically, this induced form is impractical, particularly for large or growing datasets. In this work, we investigate the use of practical alternating minimization algorithms for induced \RFMs, that ensure convergence to global optima. We characterize the stationary points of these models,
and, using these insights, highlight practical choices for the objectives. We then provide theoretical and empirical evidence that alternating minimization, from a random initialization, converges to global minima for a large subclass of induced \RFMs. In particular, we take advantage of the existence of the (potentially unknown) convex induced form, to identify when stationary points are global minima for the dictionary learning objective. We then provide an empirical investigation into practical optimization choices for using alternating minimization for induced \RFMs, for both batch and stochastic gradient descent. 
\end{abstract}


\section{Introduction}

Dictionary learning models (\RFMs) are broadly used
for unsupervised learning and representation learning,
in the form of nonlinear dimensionality reduction, sparse coding,
matrix completion and many matrix factorization-based approaches. 
The general form consists of factorizing an input matrix $\Xmat$
into a dictionary $\Dmat$ and a representation (or coefficients) $\Hmat$,
potentially transformed with a nonlinear transfer $f$: $\Xmat \approx f(\Dmat \Hmat)$ \cite{singh2008unified,white2014thesis}. 
Figure \ref{fig:subspace} depicts a common setting of obtaining a low-dimensional representation.
As another example, shown in Figure \ref{fig:sparse}, in sparse coding, the input observation $\xvec$ (a column of $\Xmat$) is re-represented by sparse coefficients $\hvec$ (a column in $\Hmat$), which linearly combines a small subset of dictionary elements (columns in $\Dmat$). The motivation for sparse coding is based on mimicking the representation mechanism observed within the mammalian cortex \cite{olshausen1997sparse}. In addition to this biological motivation, sparse coding has been demonstrated to be useful for image processing, with impressive denoising \cite{elad2006image} and classification results \cite{raina2007self,mairal2009supervised,jiang2011learning}.

\begin{figure}[tp]
  \centering
    \includegraphics[width=0.95\textwidth]{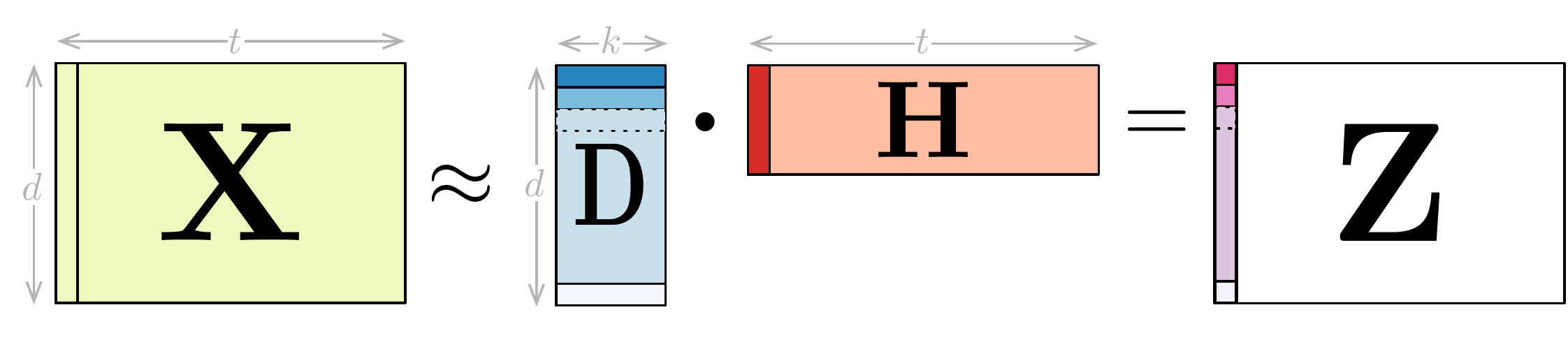}
    \caption{An example of the general form for Dictionary Learning Models (DLMs), where the inner dimension $\ldim < \xdim$, results in dimensionality reduction. The red column in $\Hmat$ corresponds to a low-dimensional re-representation of the first sample (column) in $\Xmat$. The dot product of the re-representation $\Hmat$ with the dictionary $\Dmat$ provides the approximation $\joint = \Dmat \Hmat$, where the dot product of the red column with $\Dmat$ provides the first sample (column) in $\joint$ as the approximation to the first sample (column) in $\Xmat$.}\label{fig:subspace}
\end{figure}

\RFMs\ are attractive for unsupervised learning and representation learning, because in addition to their
generality and simplicity, they are amenable to incremental estimation. Consequently, estimation can be performed on large datasets, and models can be updated with new data. These models are amenable to incremental estimation because only one of the variables---typically $\Hmat$---grows with the number of samples. The other variable---typically $\Dmat$---can be feasibly maintained incrementally,
and sufficiently summarizes the solution. The ability to learn these models incrementally across samples additional provides the standard benefits from using stochastic gradient descent, including efficiency and being more reactive to non-stationary data.
 Despite these advantages, incremental estimation of \RFMs\ has been under-explored. 

 
The main reason for the restriction is the difficultly in optimal estimation of \RFMs,
even in the batch setting.
The main difficultly arises from the fact that the optimization is over bilinear parameters
$\joint = \Dmat \Hmat$, which interact to create a non-convex optimization.
There have been impressive gains in recent years, in terms of obtaining optimal solutions
to this non-convex problem. An important development has been the identification of induced \RFMs:
dictionary learning models that admit a convex regularizer directly on $\joint$, given the regularizers on $\Dmat$ and $\Hmat$. 
Consequently, by performing the optimization directly on $\joint$ instead of the factors, using the induced regularizer on $\joint$, 
the optimal solution can be obtained for the factor model. This insight has led to several convex reformulations, including for
relaxed rank exponential family PCA \cite{bach2008convex,zhang2011convex}, 
multi-view learning \cite{white2012convex},
(semi-)supervised dictionary learning \cite{goldberg2010transduction,zhang2011convex} and
autoregressive moving average models \cite{white2015optimal}.
This insight, however, is practically limited because the convex regularizer only has a known form for a small subset of induced \RFMs\ and further the optimization is no longer amenable to incremental estimation
because $\joint$ grows with the size of the data. Nonetheless, intuitively because this class enables convex reformulations, even though they are not computable, it identifies induced \RFMs\ as a promising class to study to obtain global solutions for representation learning. 

In that direction, there has been significant effort to understand the optimization surface of these non-convex objectives, including identifying that all local minima are global minima and proving that there are no deleterious saddlepoints (i.e., degenerate saddlepoints).
There has been quite a bit work illustrating that alternating minimizations on the factors
produces global models, including
low-rank matrix completion with least-squares losses \cite{mardani2013decentralized,jain2013low,gunasekar2013noisy},
a thresholded sparse coding algorithm \cite{agarwal2017aclustering,agarwal2014learning},
semi-definite low-rank optimization, where the two factors are the same $\joint = \Dmat \Dmat^\top$
\cite{aspremont2004direct,burer2005local,bach2008convex,journee2010low,zhang2012accelerated,mirzazadeh2015scalable}
and for a generalized factorization setting that includes (rectified) linear neural networks \cite{haeffele2015global,kawaguchi2016deep}. 
These approaches provide important seminal results, but still only characterize a small subset of induced \RFMs,
particularly requiring rank-deficient (i.e., low rank) solutions or specialized initialization strategies.

For incremental estimation, the guarantees become further limited. There are specialized incremental algorithms for principal components analysis \cite{warmuth2008randomized,feng2013online} and partial least-squares \cite{arora2012stochastic}. Results characterizing the optimization surface \cite{haeffele2015global} provide insights on the optimality of stochastic gradient descent; however, as mentioned, such characterizations have been largely limited to rank-deficient solutions. For other settings,
little is known about how to obtain optimal incremental \RFM\  estimation and, particularly for incremental sparse coding, several approaches have settled for local solutions \cite{mairal2009online,mairal2009supervised,mairal2010online}. Yet, as we argue, there is more room to take advantage of the induced form for these models, to identify and theoretically characterize tractable objectives.


In this work, we provide both theoretical and empirical evidence that gradient descent on the non-convex, factored form for induced \RFMs\ produces global solutions for a broad class of induced \RFMs.
We make the following conjecture, that leads to simple
and effective global optimization algorithms for induced \RFMs.
\begin{displayquote}
Alternating minimization for induced dictionary learning models, with a full rank random initialization,
converges to a globally optimal solution. 
\end{displayquote}
The key contributions of this paper are to provide evidence for this conjecture
and to develop effective optimization techniques based on alternating minimization. 
We first expand the set of induced \RFMs, and propose several practical modifications
to previous objectives, particularly to enable incremental estimation. 
We prove a novel result that a large subclass of induced \RFMs\ satisfy the property that every local minimum
is a global minimum, and further demonstrate that a subclass does not have degenerate saddle-points,
providing compelling evidence that alternating minimization will successfully find global minima. 
The key novelty in our theoretical results is a new approach to characterizing the overcomplete setting,
for full rank solutions $\Dmat, \Hmat$, as opposed to the more studied rank-deficient setting. 
We provide empirical support for the above conjecture, in particular illustrating that small deviations to outside the class of induced \RFMs\ no longer
have the property that alternating minimization produces global solutions. 
With this justification for alternating minimization, 
we then develop algorithms that converge faster than vanilla implementations of alternating minimization,
for both batch and stochastic gradient descent.\footnote{This terminology is used for this work,
but note that stochastic gradient descent
is also called stochastic approximation,
and batch gradient descent is 
called sample average approximation.}

In addition to providing evidence for this conjecture, 
the theoretical results in this work provide a new methodology for analyzing
stationary points for general dictionary learning problems.
With analysis conducted under this more general class of models,
we automatically obtain novel results for specific settings of interest, including matrix completion, sparse coding
and dictionary learning with elastic net regularization. The current assumptions require differentiability,
which restricts the results to smooth versions of these problems. Nonetheless, it provides first steps
towards a unified analysis that, with extensions to non-smooth functions, could address many of these settings simultaneously. 

The paper is organized as follows. We first motivate the types of dictionary learning problems
that can be specified as induced \RFMs\ (Section 2) and then discuss how to generalize and improve
these objectives (Section 3). Then we prove the main theoretical result (Section 4) and further
provide empirical evidence for the conjecture (Section 5), highlighting some practical optimization choices for the batch setting.  
We then propose several incremental algorithms to optimize induced \RFMs,
to make it feasible to learn these models for large datasets (Section 6).
Finally, we summarize a large body of related work (Section 7) and conclude
with a discussion on current limitations of this work and important next steps. 


\section{Dictionary Learning Models}\label{sec_fms}

In this section, we introduce induced \RFMs\ and provide several examples of objectives that are within the class.
We focus on examples that demonstrate how this optimization formalism can be used to obtain different representation properties
as well as more modern uses to help unify recently introduced models under the formalism. 
Suitable objectives for these models are summarized in Table \ref{table_objectives},
which include subspace dictionary learning, matrix completion, sparse coding and elastic-net dictionary learning. 
See \citep[Section 3.4]{white2014thesis}
for a more complete overview of
the representational capabilities of the larger class of \RFMs. 
The goal of this section is to provide intuition for the uses and generality of induced \RFMs.

We begin by defining the class of induced \RFMs.
Let $\Xmat \in \RR^{n \times \nsamples}$ consist of
$\nsamples$ samples of $n$ features.
The goal is to learn a bilinear factorization $\joint = \Dmat \Hmat$
where $\Dmat \in \RR^{\xdim \times \ldim}$ and $\Hmat \in \RR^{\ldim \times \nsamples}$
for given\footnote{Typically $\xdim = n$, such as for principal components analysis;
however, in some setting, 
 $\xdim > n$ is useful, such as for autoregressive moving average models \cite{white2015optimal}.} $\xdim, \ldim$,
given any convex loss $\jointloss: \RR^{\xdim \times \nsamples} \rightarrow \RR$.
For example, the loss could be the sum of squared errors between the reconstruction and the input,
\begin{equation*}
\jointloss(\joint, \Xmat) = \sum_{\sampiter = 1}^\nsamples \Lx(\joint_{:\sampiter}, \Xmat_{:\sampiter}) \hspace{1.0cm} \text{ for } \ \ \ \Lx(\joint_{:\sampiter}, \Xmat_{:\sampiter}) = \| \joint_{:\sampiter} - \Xmat_{:\sampiter} \|_2^2
.
\end{equation*}
More generally, this convex loss will typically
 be defined based on a transfer (or activation function) $f: \RR \rightarrow \RR$
and the corresponding matching loss $\Lx(\joint_{:\sampiter}, \Xmat_{:\sampiter})$ that is convex in the first argument
(see \cite{helmbold1996worst}). This loss evaluates the difference between $f(\joint_{:\sampiter})$ and $\Xmat_{:\sampiter}$,
with overloaded definition that $f(\joint_{:\sampiter})$ or $f(\joint)$ is the transfer $f$ applied to each entry of the input vector or matrix.
For example, $f$ could be an identity function, resulting in no transformation, and $\Lx(\joint_{:\sampiter}, \Xmat_{:\sampiter}) = \| \joint_{:\sampiter} - \Xmat_{:\sampiter} \|_2^2$. As another example,
$f$ could be the sigmoid function and $\Lx(\joint_{:\sampiter}, \Xmat_{:\sampiter})$ the cross-entropy. 

Additionally, the \RFM\ objective is characterized by the choice of regularizers. Properties on $\Dmat$ and $\Hmat$ are encoded using convex regularizers $\regd: \RR^{\xdim \times \ldim} \rightarrow \RR$ and $\regh: \RR^{\ldim \times \nsamples} \rightarrow \RR$. As we describe in the cases studies in the remainder of this section, common choices for these regularizers are $\ell_2$ norms---for low-rank solutions---and $\ell_1$ norms---for sparse solutions. 

The general optimization for FMs corresponds to
\begin{align}
\min_{\Dmat \in \RR^{\xdim \times \ldim}, \ \Hmat \in \RR^{\ldim \times \nsamples}}  \jointloss(\Dmat\Hmat) + \regd(\Dmat) + \regh(\Hmat)
\label{eq_general_rfm}
\end{align}
Induced \RFMs\ correspond to a subclass of this more general objective, where there are additional conditions on these regularizers. 
The condition is that there is a convex induced regularizer $\regz$ on the product $\joint = \Dmat \Hmat$, that satisfies
\begin{equation}
\regzf{\joint} =  
\frac{1}{2} \min_{\substack{\Dmat \in \RR^{\xdim \times \ldim}, \ \Hmat \in \RR^{\ldim \times \nsamples}\\\joint = \Dmat \Hmat}}  \regd(\Dmat) + \regh(\Hmat)
.
\label{eq_inducednorm}
\end{equation}
This property is potentially useful because the resulting optimization over $\joint$ is convex, providing a convex
reformulation of the non-convex \RFM\ objective. This reformulation has been extensively used for learning lower-dimensional
representations, with the trace norm (aka nuclear norm). Further, \citet{bach2008convex} proved that, if $\ldim$ is sufficiently large (potentially infinite), then there exists an induced regularizer when
\begin{align*}
\regd(\Dmat) &= \regwgt \sum_{i=1}^\ldim \|\Dmat_{:,i}\|^2_c \\
  \regh(\Hmat) &= \regwgt\sum_{i=1}^\ldim \|\Hmat_{i,:}\|^2_{r}
\end{align*}
for column and row norm regularizers $\| \cdot \|_c$ and $\| \cdot \|_r$, with
regularization parameter $\regwgt \ge 0$. This result was further generalized by \citet{haeffele2015global} and we
provide a further generalization in Proposition \ref{prop_generalization}.

The direct use of this convex reformulation, however, is limited for two reasons.
First, an explicit form for the induced regularizer $\regz$ is only known for a small number of settings. 
Second, only for a further subset of these, is there a clear way to extract the corresponding optimal
$\Dmat, \Hmat$ from the learned $\joint = \Dmat \Hmat$. Intuitively, however, this property identifies a class
of well-behaved problems. We will show that it is preferable to directly learn $\Dmat, \Hmat$, with a simple gradient descent
approach on the non-convex factored objective, and simply use the existence of this convex reformulation to identify promising, tractable objectives.

%

Given this general definition of induced \RFMs, we now turn to providing more specific examples. 
For each of the examples of induced \RFMs, we will highlight if they have a known induced regularizer. 
The purpose of this summary is to unify existing representation learning approaches
under induced \RFMs, as well as provide examples of how to 
encode properties on the learned representation $\Hmat$ using the relatively
simple induced \RFM\  formalism.

\subsection{Subspace dictionary learning, including nonlinear dimensionality reduction}

A common goal in representation learning is to to obtain a lower-dimensional
representation of input $\Xmat$.
Learning such a low-dimensional representation using an induced \RFM\   objective 
has typically been formulated~\citep{bach2008convex,zhang2011convex} by
setting $\| \cdot \|_c = \| \cdot \|_2 = \| \cdot \|_r$,
and $\jointloss(\joint) = \sum_{\sampiter=1}^\nsamples \Lx(\joint_{:\sampiter}, \Xmat_{:\sampiter})$.
The corresponding
induced convex reformulation is
\begin{align}
\min_{\genrankset} & \jointloss(\Dmat\Hmat) + \tfrac{\regwgt}{2} \frobsq{\Dmat} + \tfrac{\regwgt}{2} \frobsq{\Hmat} \label{alt_simple} \\
&=
\min_{\joint \in \RR^{\xdim \times \nsamples}} \jointloss(\joint) + \regwgt \tracenorm{\joint}  \ \ \ \ \ \  \  \triangleright \text{ if } \ldim \ge \xdim
\label{opt_simple}
\end{align}
where the Frobenius norm is defined as $\| \Dmat \|_F^2 = \sum_{i=1}^\ldim \| \Dmat_{:,i} \|_2^2$
and the trace norm (or nuclear norm) is defined as $\tracenorm{\joint} = \sum_{i=1}^{\min(\xdim,\nsamples)} \sigma_i(\joint)$.
This objective corresponds to a relaxed rank formulation of principal components analysis (PCA). 
To obtain the standard PCA objective \cite{xu2009optimal}, $\regwgt$ is set to zero and $\ldim < \xdim$ is set to the desired rank,
with $\Lx(\joint_{:\sampiter}, \Xmat_{:\sampiter}) = \frobsq{\joint_{:\sampiter} - \Xmat_{:\sampiter}}$.
The resulting $\joint$ corresponds to a low rank approximation to $\Xmat$, that consists of the top $\ldim$ singular values and vectors of $\Xmat$: given singular value decomposition $\Xmat = \Umat \Sigmamat \Vmat^\top$, the resulting $\joint = \Umat \Sigmamat_\ldim \Vmat^\top$ where $\Sigmamat_\ldim$ is a diagonal matrix of the top $\ldim$ singular values of $\Sigmamat$ and the rest set to zero.  
For the relaxed rank setting, with $\regwgt > 0$, the optimal solution is $\joint = \Umat (\Sigmamat - \regwgt)_+ \Vmat^\top$, where $(\cdot)_+$ truncates negative values at zero. For this case, instead of fixing $\ldim$ to a smaller rank, 
the rank is implicitly restricted by the induced regularizer on $\joint$, and for the factored form, by the regularizers on $\Dmat$ and $\Hmat$. 

To further understand why these chosen column and row norms result in this dimensionality reduction effect,
consider the induced norm. The trace norm is known to be a tight convex relaxation of rank. The optimization over
$\joint$, for large enough $\regwgt$, selects a lower rank $\joint = \Dmat \Hmat = \Umat (\Sigmamat - \regwgt)_+ \Vmat^\top$ as described. 
Further, it is well known that the above formulation can be equivalently written as a constrained form with a $(2,1)$-block norm on $\Hmat$:
\begin{align*}
\min_{\genrankset} & \jointloss(\Dmat\Hmat) + \tfrac{\regwgt}{2} \frobsq{\Dmat} + \tfrac{\regwgt}{2} \frobsq{\Hmat}  
&=
\min_{\substack{\Hmat \in \RR^{\ldim \times \nsamples},\Dmat \in \RR^{\xdim \times \ldim}\\ \|\Dmat_{:i} \|_2 \le 1}}  \jointloss(\Dmat\Hmat) + \regwgt \sum_{i=1}^\ldim \| \Hmat_{i:} \|_2
\end{align*}
The $(2,1)$-block norm constitutes a group-sparse regularizer on $\Hmat$,
where the sum of the $\ell_2$ norms encourages entire rows of $\Hmat$ 
to be zero \cite{argyriou2008convex,white2014thesis}. Once an entire row of $\Hmat$ is set to zero, the rank is reduced by
one, and the implicit $\ldim$ is actually one less\footnote{
Note that the summed form in  \eqref{alt_simple} does not necessarily enforce zeroed rows of $\Dmat$ and $\Hmat$,
though it does guarantee equivalently low-rank solutions. To see why,
consider the singular value 
decomposition of $\Dmat = \Umat \Sigmamat \Vmat^\top$. 
We get an equivalent solution with $\Dmat \Vmat$ and $\Vmat^\top \Hmat$, where $\Vmat \Vmat^\top = \eye$ and so 
$\Dmat \Vmat\Vmat^\top \Hmat = \Dmat \Hmat$. The regularization values remain unchanged because the Frobenius norm is invariant under orthonormal matrices.
The new solution $\Dmat \Vmat = \Umat \Sigmamat$ does in fact only have $k$ non-zero columns, because $\Sigmamat$ only has $k$ non-zero entries
on the diagonal. The optimization, however, has no preference to select the form of $\Dmat$ with or without $\Vmat$ and so may not prefer the solution with zeroed columns in $\Dmat$ and zeroed rows in $\Hmat$.}.

These subspace dictionary learning objectives encompass a wide-range of dimensionality reduction
approaches, including principal components analysis, canonical correlation analysis, partial least-squares
and nonlinear dimensionality reduction approaches such as Isomap and non-linear embeddings \cite{white2014thesis}.
To obtain non-linear dimensionality reduction, the inputs $\xvec$ are first transformed with kernels,
and then the same subspace dictionary learning optimization used above. 
Similarly to PCA, we can additionally obtain relaxed rank versions of each of these nonlinear dimensionality approaches. 

\subsection{Matrix completion}

The matrix completion problem, formulated as a low-rank completion problem \cite{candes2009exact}, is an instance of subspace dictionary learning. 
The matrix completion problem is
often used for collaborative filtering, where the goal is to infer rankings or information about a user using
a small amount of labeled information from other users. For example, for Netflix ratings, the goal is to complete
a matrix of ratings, with $\xdim$ users as rows and $\nsamples$ movies as columns. 
Each $\Xmat_{ij}$ corresponds to a rating for user $i$ and movie $j$, where only a subset of such ratings will be available. 
The idea behind finding a low-rank $\Dmat$ and $\Hmat$ to
factorize the known components of $\Xmat$ is that there is some latent structure that explains the ratings.

A common objective for matrix completion is
\begin{align*}
\sum_{\text{observed } (i,j)} (\Xmat_{ij} - \Dmat_{i:} \Hmat_{:j})^2 + \regwgt \frobsq{\Dmat} +  \regwgt \frobsq{\Hmat}
\end{align*}
which is an instance of the subspace objective in \eqref{alt_simple}. There are numerous variations on this basic objective to improve performance.
For example, \citet{salakhutdinov2010collaborative} use weighted norms for non-uniform sampling, giving 
\begin{align*}
\sum_{\text{observed } (i,j)} (\Xmat_{ij} - \Dmat_{i:} \Hmat_{:j})^2 + \regwgt \frobsq{\Lambdamat_D\Dmat} +  \regwgt \frobsq{\Hmat\Lambdamat_H}
\end{align*}
where $\Lambdamat_D \in \RR^{\xdim \times \xdim}$ is a positive diagonal matrix that reweights rows of $\Dmat$
and 
where $\Lambdamat_H \in \RR^{\nsamples \times \nsamples}$ is a positive diagonal matrix that reweights columns of $\Hmat$.
With a change of variables, this can equivalently be written as a minimization over
\begin{align*}
\sum_{\text{observed } (i,j)} (\Xmat_{ij} - \Lambdamat_D(i,i)^{-1}\Lambdamat_H(j,j)^{-1}\Dmat_{i:} \Hmat_{:j})^2 + \regwgt \frobsq{\Dmat} +  \regwgt \frobsq{\Hmat}
\end{align*}
where 
\begin{align*}
\jointloss(\Dmat \Hmat) &= \sum_{\text{observed } (i,j)} (\Xmat_{ij} - \Lambdamat_D(i,i)^{-1}\Lambdamat_H(j,j)^{-1} (\Dmat \Hmat)_{ij})^2\\
&= \sum_{\text{observed } (i,j)} (\Xmat_{ij} - \Lambdamat_D(i,i)^{-1}\Lambdamat_H(j,j)^{-1}\Dmat_{i:} \Hmat_{:j})^2
.
\end{align*}
As with subspace learning, $\ldim$ is set less than $\xdim$, and $\regwgt$ is tuned to obtain a lower-dimensional $\Dmat, \Hmat$. Each user has $\ldim$ explanatory variables, where the column $\Dmat_{:l}$ corresponds to the value of that variable for all the users and $\Hmat_{l:}$ corresponds to the value for that variable for all the movies. Though these latent variables may have no interpreation, intuitively one could imagine that they describe general properties of users and movies. For example, column $l$ in $\Dmat$ could correspond to ``likes romance" and the corresponding row in $\Hmat$ could be ``movie has romance". The dot product $\Xmat_{ij} = \Dmat_{i:} \Hmat_{:j}$ would include $\Dmat_{il} \Hmat_{lj}$. If the user does not like romance, or the movie does not have romance, these coefficients might be zero and will not contribute to the dot product; else, if the movie has romance and the user likes romance, the coefficients are likely non-negligible and can be predictive of a high rating.  Even though $\Xmat_{ij}$ is only know for a small subset of $\Xmat$, we learn the properties $\Dmat_{:l}$ and $\Hmat_{l:}$ for all users and movies based on the information that is available. Once we learn these properties, we can complete the unknown ratings using $\Xmat_{ij} = \Dmat_{i:} \Hmat_{:j}$.

\subsection{Sparse coding}

\begin{figure}[tp]
  \centering
    \includegraphics[width=0.8\textwidth]{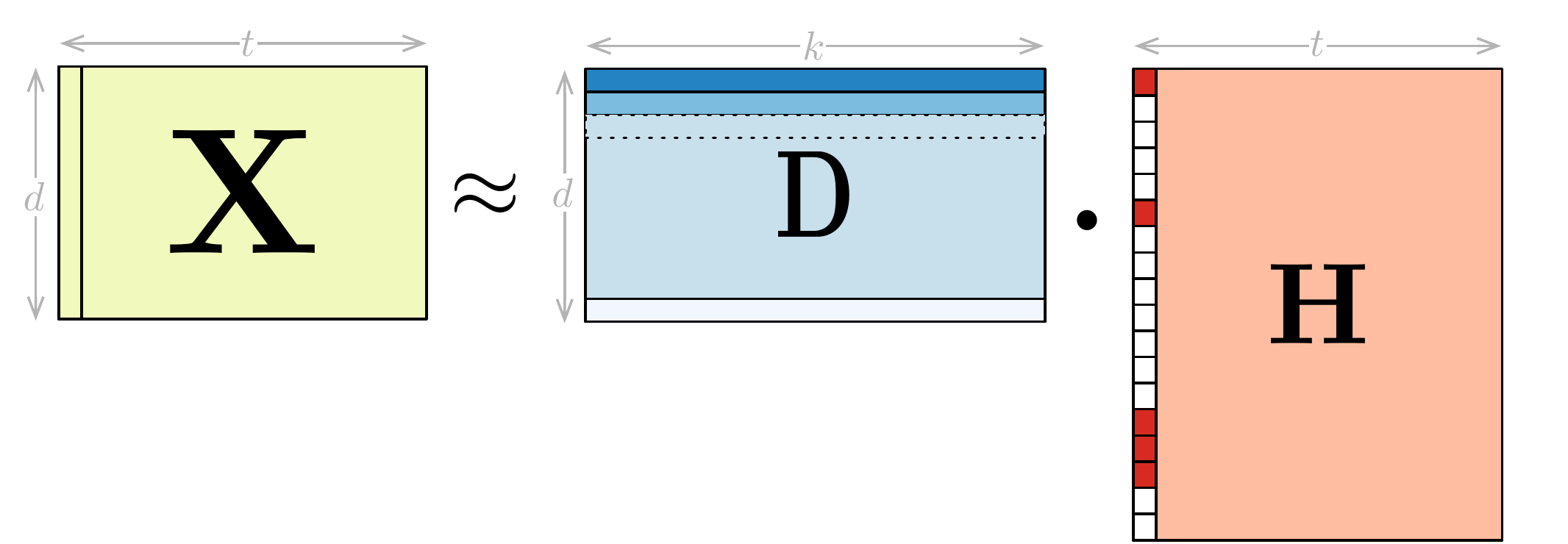}
    \caption{Sparse coding, with an overcomplete dictionary, $\ldim > \xdim$, and sparse re-representation. The red column in $\Hmat$ corresponds to a high-dimensional, sparse re-representation of the first sample (column) in $\Xmat$. The dictionary atoms (rows) in $\Dmat$ intuitively correspond to prototypical examples.}\label{fig:sparse}
\end{figure}

The aim in sparse coding is to re-represent the inputs with a sparse set of coefficients, that linearly weight a subset of representative
dictionary atoms (columns of $\Dmat$). Sparse coding was originally introduced based on observed representations in the mammalian brain \citep{olshausen1997sparse}, and has since proven practically useful particularly in image processing, where dictionary atoms can encode edges and other modular components within images. A similar such premise underlies many representations. Consider an extreme form of sparse coding, where only one element in $\Hmat_{:j}$ is active; this represents an indicator to the bin for the given input, and can be thought of as clustering or aggregation. Sparse coding can be seen as selecting a small number of key properties to re-represent the input; similar inputs should have similar---or at least overlapping---active properties. 

A sparse representation is typically obtained by using an $\ell_1$ norm on $\Hmat$, with induced \RFM\ objective\footnote{This objective is slightly different from an alternative formulation of
sparse coding where a specified level of sparsity is given, for which alternating minimization has also been explored \cite{agarwal2017aclustering}.}
\begin{equation}
\min_{\genrankset}  \jointloss(\Dmat \Hmat) + \tfrac{\regwgt}{2} \sum_{i=1}^\ldim \|\Dmat_{:,i}\|^2_{q} + \tfrac{\regwgt}{2} \sum_{i=1}^\ldim \|\Hmat_{i,:}\|^2_{1}
\label{eq_sparseobj}
.
\end{equation}
where
$\|\joint^\top\|_{q,1} = \sum_{i} \|\joint_{:,i}\|_{q}$
can have any $q \ge 1$.
This objective has a known convex induced regularizer \citep[Proposition 2]{zhang2011convex}),
as long as $\ldim \ge \nsamples$
\begin{equation*}
\eqref{eq_sparseobj}
=
\min_{\joint \in \RR^{\xdim \times \nsamples}} \jointloss(\joint) + \regwgt  \|\joint^\top\|_{q,1} \ \ \ \ \ \  \  \triangleright \text{ if } \ldim \ge\nsamples
.
\end{equation*}
The resulting global solution, however, requires $k = \nsamples$ and
corresponds to memorizing normalized observations:
$\Dmat_{:,i} = \joint_{:,i}/ \| \joint_{:,i}\|_q$ and diagonal $\Hmat_{i,i} = \| \joint_{:,i}\|_q$.
To remedy this issue with induced \RFMs\ and sparse learning, 
\citet{bach2008convex} proposed to combine the subspace and sparse regularizers,
described in the next subsection. In this work, we highlight that this objective may actually have a convex induced regularizer
for smaller $\ldim$, even though the form of this induced regularizer may not be known. 

%

\subsection{Elastic-net dictionary learning}
We can interpolate between both subspace and sparse regularizers (also called elastic net regularization \cite{zou2005regularization}),
using parameters $\nud, \nuh \in [0,1]$
\begin{align*}
&\min_{\Dmat \in \RR^{\xdim \times \ldim},\Hmat \in \RR^{\ldim \times \nsamples}} \jointloss(\Dmat \Hmat) + \tfrac{\regwgt}{2} \sum_{i=1}^\ldim \Big[ \nud \| \Dmat_{:i} \|_2^2  + (1-\nud) \|\Dmat_{:,i}\|^2_{1} + \nuh \|\Hmat_{i:} \|_2^2 +  (1-\nuh) \|\Hmat_{i,:}\|^2_{1} \Big]
\end{align*}
with elastic net norm\footnote{In the elastic net introduced by \cite{zou2005regularization}, the $\ell_1$ norm is not squared. We square the $\ell_1$ norm to ensure we have a valid norm, as
$\| \vvec \| = \sqrt{ \nu \| \vvec\|_2^2 + (1-\nu) \| \vvec\|_1}$
is not a norm. The purpose of the combination, however, is the same and so we still call it an elastic net regularizer.} $\| \vvec \| = \sqrt{ \nu \| \vvec\|_2^2 + (1-\nu) \| \vvec\|_1^2 }$. 
This regularizer should prefer a more compact sparse representation $\Hmat$. Previous results indicate improved stability and compactness \cite{zou2005regularization}; we verify these properties for representation learning in Section \ref{sec_empirical},
showing that $\ldim$ can be smaller with the elastic net objective compared to the sparse coding objective.

The elastic-net objective is an instance of an induced \RFM,
because the elastic net norm is a valid column and row norm (see Proposition \ref{prop_convex} in Appendix \ref{app_regularizers})
and so we know that a corresponding induced norm on $\joint$ exists for sufficiently large $\ldim$.
However, unlike the previously described induced \RFMs, there is no known efficiently computable form for the induced regularizer.
For this reason, this regularizer remains one of the bigger open questions for induced \RFMs. It is additionally a motivator for using 
the existence of the induced regularizer, without requiring a known form.

Though we write the above having the elastic net norm on both factors,
this is not necessary. An elastic net norm could be used on one factor, and a completely different norm on another factor. 
We simply write the above to introduce notation, and because several common settings are obtained with different 
 $\nud$ and $\nuh$. For example, for $\nud = 1$, we have an $\ell_2$ regularizer on the columns of $\Dmat$,
 and an elastic net norm on the rows of $\Hmat$, which is the setting addressed by \citet{bach2008convex}.

\subsection{Supervised dictionary learning}

\begin{figure}[tp]
  \centering
    \includegraphics[width=0.7\textwidth]{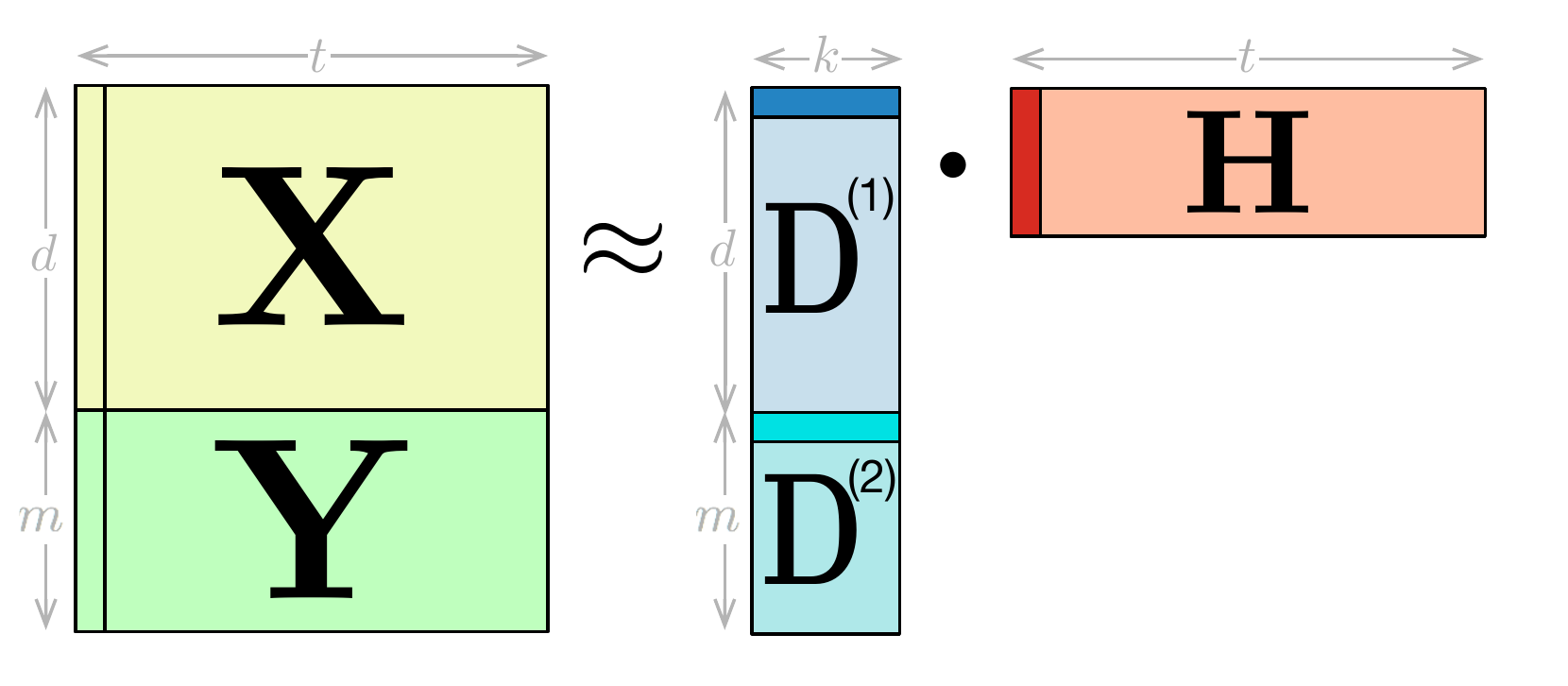}
    \caption{Supervised dictionary learning, with a separate dictionary for the input $\Xmat$ and the target $\Ymat$. This example depicts $\ldim < \xdim$, giving dimensionality reduction, though other choices such as a higher-dimensional, sparse representation could have also been used.}\label{fig:supervised}
\end{figure}

The objectives so far have focused on unsupervised learning;
however, for many cases, the goal is to improve supervised learning,
which can be elegantly incorporated under \RFMs.
A typical strategy is to learn a new representation in an unsupervised way,
and then use that representation to learn a supervised predictor given labeled samples.
These two stages can be combined into one objective,
where the new factorized representation is learned for $\Xmat$
while also using it for supervised learning, given some labels $\Ymat \in \RR^{\ydim \times \nsamples}$.
For simplicity we will assume that we have labels corresponding to
each sample (column) in $\Xmat$; this can be relaxed to a semi-supervised setting by ignoring missing entries in the supervised loss component \cite{zhang2011convex}.

The supervised \RFM\   is \citep{goldberg2010transduction,zhang2011convex,white2012convex}
\begin{align}
&\argmin_{\rrankset} \jointloss\left(\mathvec{\Dmatone}{\Dmattwo}\Hmat \right)
+ \tfrac{\regwgt}{2} \sum_{i=1}^\ldim \max( \| \Dmatone_{:i} \|^2, \| \Dmattwo_{:i} \|^2)+ \tfrac{\regwgt}{2} \sum_{i=1}^\ldim \| \Hmat_{i,:}  \|_r^2 
.
\label{eq_maxf}
\end{align}
$\Dmat$ is now partitioned into two components $\Dmat = \inlinevec{\Dmatone}{\Dmattwo}$,
with $\Dmatone \in \RR^{\xdim \times \ldim}$ and $\Dmattwo \in \RR^{\ydim \times \ldim}$,
but there is a shared representation $\Hmat$. 
The loss can be defined for a regression setting as 
\begin{align*}
\jointloss\left(\mathvec{\Dmatone}{\Dmattwo}\Hmat \right) 
&= \left \| \mathvec{\Dmatone}{\Dmattwo}\Hmat  - \mathvec{\Xmat}{\Ymat}  \right\|_F^2\\
&= \sum_{\sampiter=1}^\nsamples \left\| \mathvec{\Dmatone}{\Dmattwo}\Hmat_{:\sampiter}  - \mathvec{\Xmat_{:\sampiter}}{\Ymat_{:\sampiter}}  \right\|_2^2\\
&= \sum_{\sampiter=1}^\nsamples \| \Dmatone\Hmat_{:\sampiter}  - \Xmat_{:\sampiter} \|_2^2 + \| \Dmattwo\Hmat_{:\sampiter}  - \Ymat_{:\sampiter} \|_2^2
\end{align*}
or for classification as
\begin{align*}
\jointloss\left(\mathvec{\Dmatone}{\Dmattwo}\Hmat \right) 
&= \sum_{\sampiter=1}^\nsamples \| \Dmatone\Hmat_{:\sampiter}  - \Xmat_{:\sampiter} \|_2^2 + \text{cross-entropy}\left(\Dmattwo\Hmat_{:\sampiter}, \Ymat_{:\sampiter} \right)
.
\end{align*}
The regularizer uses the max, to ensure that each component is regularized separately.
Again because $\| \Dmat_{:i} \|_c = \max( \| \Dmatone_{:i} \|^2, \| \Dmattwo_{:i} \|^2)$
is a valid norm, we know that an induced regularizer exists for sufficiently large $\ldim$. 

There is a known form for the induced regularizer on $\joint =\Dmat \Hmat= \inlinevec{\Dmat^{(1)}}{\Dmat^{(2)}} \Hmat$ for the subspace setting, with the $\ell_2$ norm \citep{white2012convex}
\begin{align*}
\regzf{\joint} &= \tfrac{1}{2}\min_{\substack{\Dmat, \Hmat \\\Dmat \Hmat = \joint}} \sum_{i=1}^\ldim \max( \| \Dmat^{(1)}_{:i} \|_2^2, \| \Dmat^{(2)}_{:i} \|_2^2)+ \frobsq{\Hmat} \\
&= \max_{0 \le \eta \le 1} \tracenorm{\Diagmat_\eta \joint}
 \ \ \ \ \text{where} \ \ \Diagmat_\eta = \left[\begin{array}{cc} 
\!\!\! \sqrt{\eta} \eye_{\xdim} & \zerovec\\
 \zerovec & \sqrt{1-\eta} \eye_{\ydim} \!\!\! \end{array} \right].
\end{align*}
The diagonal matrix $\Diagmat_\eta$ reweights the rows of
$\joint$ in two blocks corresponding to the two components $\Dmatone$ and $\Dmattwo$.

\subsection{Robust objectives}
All of the above objectives can accommodate robust alternatives.
This includes using robust convex losses---such as the Huber loss---as well an incorporating
an additional variable $\Smat$ that represents the noise \cite{candes2011robust,xu2010robust,zhang2011convex}.
For sparse noise, for example, such as in robust PCA \cite{candes2011robust}, we can impose an $\ell_1$ regularizer on $\Smat$ and define the loss
\begin{align*}
\jointloss(\Dmat \Hmat) = \min_{\Smat \in \RR^{\xdim \times \nsamples}}\sum_{j=1}^\nsamples L_x(\Dmat \Hmat_{:j} + \Smat_{:j}, \Xmat_{:j}) + \regwgt_s \| \Smat \|_{1,1} 
\end{align*}
for some regularization parameter $\regwgt_s > 0$ that enforces the level of sparsity in the learned
noise $\Smat$. This loss $\jointloss(\cdot)$ is convex in $\joint = \Dmat \Hmat$, 
because the composition of a convex loss and an affine function on two variables (the sum of the two variables $\joint_{:j} + \Smat_{:j}$) is jointly convex in those two variables. 
For robust subspace dictionary learning, if the data is corrupted by sparse noise, then this loss can be used
to remove the sparse noise and still learn the lower-dimensional latent structure. 


\section{Exploring improved objectives}\label{sec_objectives}

The objectives defined in the previous section demonstrate the generality of the induced \RFM\ class;
in this section, we highlight some simple modifications and generalizations beyond these more widely-used induced \RFMs. 
We show that different choices for the objective, that are equivalent in terms of model specification,
can have important ramifications for optimization---particularly for incremental estimation.
Our goal is to identify practical objective choices. Though we cannot make definitive claims for the correct choices, 
we highlight some of the equivalent options and provide a summary table of what we believe are preferable objectives.
Additionally, we aim to generalize the set of widely-used induced \RFMs, since the variety of objectives that have induced regularizers
is much broader than has been the focus for most objectives using matrix factorization. 


\newcommand{\regzk}{\mathrm{R}_\ldim}

\subsection{Generalized induced form}
Most of the work on induced \RFMs\ has assumed that the regularizers are norms. As we show below, however,
this restriction is not necessary: the regularizers can be any non-negative, centered convex functions $\fcol: \RR^\xdim \rightarrow \RR^+$
and $\frow: \RR^\nsamples \rightarrow \RR^+$. \citet{haeffele2015global} have a similar generalization, in Proposition 11,
but additionally require that the regularizers be positively homogenous\footnote{They use the term ``positive semi-definite function" to mean non-negative, centered function.}. They tackle a more general setting, with the product of more than two variables; for specifically \RFMs, we can obtain a stronger result. 

Define 
\begin{align}
\regzk(\joint) \defeq \min_{\Dmat \in \RR^{\xdim \times \ldim}, \ \Hmat \in \RR^{\ldim \times \nsamples}: \joint = \Dmat \Hmat} \sum_{i=1}^\ldim \left( \fcol^2(\Dmat_{:i})  +  \frow^2(\Hmat_{i:}) \right)  
\label{eq_regk}
\end{align}

\begin{proposition}\label{prop_generalization}
Assume $\fcol: \RR^\xdim \rightarrow \RR^+$
and $\frow: \RR^\nsamples \rightarrow \RR^+$ are convex functions, that are non-negative and centered: $\frow(\zerovec) = 0 = \fcol(\zerovec)$.
For all $\joint \in \RR^{\xdim \times \nsamples}$, the limit $\regz(\joint) = \lim_{\ldim \rightarrow \infty} \regzk(\joint)$ exists
and $\regz$ is a convex function.
\end{proposition}
\begin{proof}
For a given $\joint$, $\regzk(\joint)$ is nonnegative because of the squares on the function $f_r$ and $f_c$. 
Consequently, it's value cannot be pushed to negative $\infty$.
Because it is a minimum of these non-negative functions, which can only have more flexibility with increasing $\ldim$,
$\regzk(\joint)$ is non-increasing with $\ldim$. Therefore, it has a finite, non-negative limit as $\ldim$ tends to infinity.

We now show $\regz$ is convex. Take any $\joint_1, \joint_2 \in \RR^{\xdim \times \nsamples}$, $\eta \in [0,1]$ and
any $\epsilon > 0$. For large enough $\ldim$, there exists $\epsilon$-optimal decompositions $\Dmat^{(1)}, \Hmat^{(1)}$ and $\Dmat^{(2)}, \Hmat^{(2)}$
for $\joint_1$ and $\joint_2$ respectively
\begin{align*}
\regz(\joint_1) &\ge \sum_{i=1}^\ldim \left( \fcol^2(\Dmat^{(1)}_{:i})  +  \frow^2(\Hmat^{(1)}_{i:}) \right) - \epsilon\\
\regz(\joint_2) &\ge \sum_{i=1}^\ldim \left( \fcol^2(\Dmat^{(2)}_{:i})  +  \frow^2(\Hmat^{(2)}_{i:}) \right)  -\epsilon 
\end{align*}
 Then for $\Dmat =  [\sqrt{\eta} \Dmat^{(1)} \ \ \  \sqrt{1-\eta} \Dmat^{(2)}] $ and 
 $\Hmat =   \inlinevec{\sqrt{\eta}\Hmat^{(1)}}{ \sqrt{1-\eta} \Hmat^{(2)}}$,
 $$\joint = \eta \joint_1 + (1-\eta) \joint_2 = \Dmat \Hmat.$$
 If $\fcol$ is non-negative, then for any $\dvec \in \RR^\ldim$, because $\fcol(\zerovec) = 0$
  \begin{align*}
 \fcol(\eta \dvec) = \fcol(\eta \dvec + (1-\eta) \zerovec) \le \eta \fcol(\dvec) + (1-\eta) \fcol(\zerovec) = \eta \fcol(\dvec)
 .
 \end{align*}
Because $\fcol$ is non-negative, we can square both sides and maintain the inequality: $\fcol^2(\eta \dvec) \le \eta^2 \fcol^2(\dvec)$. 
 This is similarly true for $\frow$. 
Now we get
\begin{align*}
\regz(\joint) &\le  \sum_{i=1}^{2\ldim} \left( \fcol^2(\Dmat_{:i})  +  \frow^2(\Hmat_{i:}) \right)  \\
&=  \sum_{i=1}^\ldim \left( \fcol^2(\sqrt{\eta} \Dmat^{(1)}_{:i})  +  \frow^2(\sqrt{\eta}\Hmat^{(1)}_{i:}) \right)  + \sum_{i=1}^{\ldim} \left( \fcol^2(\sqrt{1-\eta}\Dmat^{(2)}_{:i})  +  \frow^2(\sqrt{1-\eta}\Hmat^{(2)}_{i:}) \right)  \\
&\le  \sum_{i=1}^\ldim \left( \eta \fcol^2(\Dmat^{(1)}_{:i})  +  \eta \frow^2(\Hmat^{(1)}_{i:}) \right)  + \sum_{i=1}^{\ldim} \left( (1-\eta) \fcol^2(\Dmat^{(2)}_{:i})  +  (1-\eta)\frow^2(\Hmat^{(2)}_{i:}) \right)  \\
&\le  \eta \sum_{i=1}^\ldim \left( \fcol^2(\Dmat^{(1)}_{:i})  +  \frow^2(\Hmat^{(1)}_{i:}) \right)  + (1-\eta) \sum_{i=1}^{\ldim} \left( \fcol^2(\Dmat^{(2)}_{:i})  +  \frow^2(\Hmat^{(2)}_{i:}) \right)  \\
&\le  \eta (\regz(\joint_1) + \epsilon) +  (1-\eta) (\regz(\joint_2) + \epsilon) \\
&=  \eta \regz(\joint_1) + (1-\eta) \regz(\joint_2) + \epsilon 
\end{align*}
Letting $\epsilon$ go to zero gives the desired result that $\regz$ is convex. 
\end{proof}

This generalized form includes many regularizers that were previously not possible. 
Some examples include
\begin{enumerate}
\item smoothed approximations to $\ell_1$, that are no longer norms, such as the pseudo-Huber loss \cite{fountoulakis2013asecond}:
$\fcol(\dvec) = \sum_{i=1}^\xdim \left( \sqrt{\mu^2 + \dvec_i^2} - \mu \right)$ for some $\mu > 0$. The pseudo-Huber loss is twice differentiable
and approaches $\ell_1$ as $\mu \rightarrow 0$ (see \citep[Figure 1]{fountoulakis2013asecond}).
\item the sum of the squares of any non-negative centered convex functions $g_i$, with $\frow^2 = g_1^2 + g_2^2 + \ldots g_m^2$. This is allowed
because $\frow(\dvec) = \sqrt{g_1^2(\dvec) + \ldots g_m^2(\dvec)}$ is convex (see Proposition \ref{prop_convex} in Appendix \ref{app_regularizers}). 
\item the smoothed elastic net norm, where the $\ell_1$ component is replaced by the pseudo-Huber loss (see Proposition \ref{prop_convex} in Appendix \ref{app_regularizers}). 
\item regularizers that act on partitions of the column of $\Dmat$ or rows of $\Hmat$ (see Corollary \ref{cor_regularizer_separate} in Appendix \ref{app_regularizers}). 
For example, for supervised learning, $\Dmatone$ and $\Dmattwo$ could now have different regularizers. This is appropriate as they serve different purpose:
one for unsupervised recovery and the other for supervised learning. 
\end{enumerate}
This generalization is particularly important for the result that alternating minimization provides optimal solutions,
as we will need differentiable regularizers.
The generalization beyond norms, to any convex function, enables the use of smoothed versions
of non-smooth regularizers. 

The generalization to any non-negative convex function significantly expands the space of potential
regularizers; however, many regularizers are not designed to then also be squared. Future work
will investigate how squaring a given non-negative convex function affects the properties intended to be
encoded by that regularizer. We provide one insight into the equivalence of stationary points
for a squared versus non-squared form, in Proposition \ref{prop_equivalence}.

\subsection{Relationships between multiple forms of the induced regularizer}

We have discussed only a summed form for the regularizer on the factors; however, this is not the only option. 
The induced regularizer was introduced with multiple
forms \cite{bach2008convex,white2014thesis}, including
the summed form
\begin{align*}
\regzk(\joint) = \frac{1}{2} \min_{\substack{\Dmat \in \RR^{\xdim \times \ldim}, \ \Hmat \in \RR^{\ldim \times \nsamples} \\ \joint = \Dmat \Hmat}}
 \sum_{i=1}^\ldim \| \Dmat_{:i} \|_c^2 + \sum_{i=1}^\ldim \| \Hmat_{i:} \|_r^2
\end{align*}
the producted form
\begin{align*}
\regzk(\joint) =\min_{\substack{\Dmat \in \RR^{\xdim \times \ldim}, \ \Hmat \in \RR^{\ldim \times \nsamples} \\ \joint = \Dmat \Hmat}}
\sum_{i=1}^\ldim \| \Dmat_{:i} \|_c \| \Hmat_{i:} \|_r
\end{align*}
and the constrained forms
\begin{align*}
\regzk(\joint) &= \min_{\substack{\Dmat \in \RR^{\xdim \times \ldim}, \ \Hmat \in \RR^{\ldim \times \nsamples} \\ \joint = \Dmat \Hmat, \ \| \Dmat_{:i} \|_c \le 1}} 
 \sum_{i=1}^\ldim \| \Hmat_{i:} \|_r\\
\regzk(\joint) &=  \min_{\substack{\Dmat \in \RR^{\xdim \times \ldim}, \ \Hmat \in \RR^{\ldim \times \nsamples} \\ \joint = \Dmat \Hmat, \ \| \Hmat_{i:} \|_r \le 1}} 
 \sum_{i=1}^\ldim \| \Dmat_{i:} \|_c\\
\end{align*}
These equivalent definitions are only at a global minimum 
of the above objectives. 

A natural question, therefore, is if these objectives have different properties away from the global minimum. 
For example, it is conceivable that these different objectives will have a different optimization surface, and different sets of
stationary points. Understanding the relationships between the stationary points for these different
forms can be important for understanding the ramifications of selecting
one form or the other. If there is an equivalence, for example, then any of
the forms can be selected. If this is the case, secondary criteria can be used to
select the form, including choosing the form that provides the most
numerical stability or even the form that enables the simplest computation
of gradients. 

We prove that the set of stationary points of the squared form
and the producted form are in fact equivalent. We further show that the
set of stationary points of the constrained form constitute a super-set
of those of the squared and producted forms. 
We define the pairs of stationary points 
to be equivalent if their product is equivalent, as in the below definition.

\begin{definition}
The pairs of stationary points $\Dmat_1, \Hmat_1$ and $\Dmat_2, \Hmat_2$
are \textbf{equivalent stationary points} if $\Dmat_1\Hmat_1=\Dmat_2 \Hmat_2$,
ensuring that the resulting induced variables $\joint_1 = \joint_2$. 
The sets of stationary points for 
two objectives are \textbf{equivalent} if for every point
in one set, there is an equivalent stationary point in the other set. 
\end{definition}

\begin{proposition}\label{prop_equivalence}
The stationary points are equivalent for the summed form
\begin{align*}
\jointloss(\Dmat \Hmat) + \frac{\alpha}{2} \sum_{i=1}^\ldim \| \Dmat_{:i} \|_c^2 + \frac{\alpha}{2} \sum_{i=1}^\ldim \| \Hmat_{i:} \|_r^2
\end{align*}
and the producted form
\begin{align*}
\jointloss(\Dmat \Hmat) + \alpha \sum_{i=1}^\ldim \| \Dmat_{:i} \|_c \| \Hmat_{i:} \|_c
\end{align*}
The stationary points of the summed and producted forms 
are also stationary points of the constrained form.
\end{proposition}
\begin{proofsketch}
See Appendix \ref{app_prop_equivalence}
for the full proof.
The proof follows from taking a stationary point from each optimization,
and reweighting with a diagonal matrix to obtain a stationary
point in the other form. 
\end{proofsketch}

This result applies only to normed regularizers, and does elucidate the relationship between the three forms for non-negative centered functions $\fcol$ and $\frow$. The result in Proposition \ref{prop_generalization} showing the existence of the induced regularizer could have been equivalently proven with the producted or constrained forms, for general $\fcol, \frow$.  In contrast to the norm case, however, it is unclear if the resulting
$\regz$ are the same for each of the forms; it is only the case that they each have a convex induced regularizer. 

Though we cannot make strong claims about preferences based on this result, we nonetheless advocate for the summed form.
Proposition \ref{prop_equivalence} provides some justification that the summed or producted forms could be equivalently chosen, in terms of stationary points. The constrained form could have more stationary points, which is not preferable because it creates the potential for more poor local minima or saddlepoints. 
In preliminary experiments, we found the producted form to be less stable. 
Moreover, gradient computations and updates are simpler for the summed form. 
Therefore, we move forward with this preference, and all further results are focused on the summed form. 

\subsection{Scaling with samples}\label{sec_samples}

An important oversight in previous specifications is an explicit normalization by the number of samples.
The magnitude of the regularizer on $\Hmat$ grows with samples---
since $\Hmat \in \RR^{\ldim \times \nsamples}$ grows with samples---whereas the regularizer on $\Dmat \in \RR^{\xdim \times \ldim}$
does not. Usually these regularizers are equally weighted by $\regwgt$;
preferably, $\Hmat$ should be scaled with samples.
For example, the loss is commonly an average error, 
e.g., $\jointloss(\Dmat \Hmat) = \tfrac{1}{\nsamples} \frobsq{\Dmat \Hmat - \Xmat}$.
Similarly, the regularizer on $\Hmat$ should be averaged,
to give a more balanced optimization
$$\tfrac{1}{\nsamples} \frobsq{\Dmat \Hmat - \Xmat} + \tfrac{\regwgt}{2} \frobsq{\Dmat} + \tfrac{\regwgt}{2\nsamples} \frobsq{\Hmat}.$$
In general, it is clearly useful to be able to normalize $\Hmat$ separately
from $\Dmat$. Though this modification seems trivial, it is not immediately obvious
from the previous convex reformulations using norm regularizers \citep{bach2008convex,zhang2011convex}
nor is it obvious how it affects the regularization weight on the induced regularizer.
This modification is particularly important for incremental estimation, using stochastic gradient descent, 
and so we explicitly characterize this relationship. 

\begin{proposition}\label{prop_scale}
Given any norms $\| \cdot \|_c$, $\| \cdot \|_r$ 
and scalar $s > 0$, if $\ldim$ is sufficiently large (in the worst case as large as $\xdim \nsamples$)
\begin{align*}
&\min_{\substack{\Dmat \in \RR^{\xdim \times \ldim}, \ \Hmat \in \RR^{\ldim \times \nsamples}}}
\jointloss(\Dmat\Hmat)+
\frac{\regwgt}{2} \sum_{i=1}^\ldim \| \Dmat_{:,i} \|_c^2 + \frac{\regwgt}{2\hscale^2}\sum_{i=1}^\ldim \| \Hmat_{i,:} \|_r^2\\
&=
\min_{\joint \in \RR^{\xdim \times \nsamples}} \jointloss(\joint) + \frac{\regwgt}{s} \regz(\joint)
.
\end{align*}
where $\regz(\joint) = \min_{\Dmat \in \RR^{\xdim \times \ldim}, \ \Hmat \in \RR^{\ldim \times \nsamples}: \joint = \Dmat \Hmat} \sum_{i=1}^\ldim  \|\Dmat_{:i} \|_c^2  +  \| \Hmat_{i:} \|_r^2$ is convex.    
\end{proposition}
\begin{proof}
$\regz$ exists and is convex by Proposition \ref{prop_generalization}, and additionally because norms are positively homogenous, we know that $\regz$ is convex for $\ldim$ no greater than $\xdim \nsamples$ \citep[Proposition 11]{haeffele2015global}.

To show the connection between regularization weights for the factored and induced forms, we will instead prove that
\begin{align}
&\min_{\substack{\Dmat \in \RR^{\xdim \times \ldim}, \ \Hmat \in \RR^{\ldim \times \nsamples}}}
\jointloss(\Dmat\Hmat)+\frac{\regwgt\hscale}{2} \sum_{i=1}^\ldim \| \Dmat_{:,i} \|_c^2  + \frac{\regwgt}{2\hscale}\sum_{i=1}^\ldim \| \Hmat_{i,:} \|_r^2 \label{eq_scalep1}\\
&=
\min_{\substack{\Dmat \in \RR^{\xdim \times \ldim}, \ \Hmat \in \RR^{\ldim \times \nsamples}}}
\jointloss(\Dmat\Hmat)+\frac{\regwgt}{2} \sum_{i=1}^\ldim \|\Dmat_{:,i}\|_c^2  + \frac{\regwgt}{2}\sum_{i=1}^\ldim \| \Hmat_{i,:}\|_r^2
\label{eq_noscalep1}
\end{align}
where we already know that $ \eqref{eq_noscalep1} = \min_{\joint \in \RR^{\xdim \times \nsamples}}
\jointloss(\joint) + \regwgt \regz(\joint)$.
Using this, we can choose regularizer weight $\frac{\regwgt}{s}$ to get the desired result.

Take any $\Dmat^*, \Hmat^*$ that are minimizers of \eqref{eq_noscalep1}. 
Assume that $\Dmat^* / \sqrt{\hscale}$ and $\sqrt{\hscale}\Hmat^*$ are 
not minimizers of \eqref{eq_scalep1}.
Then there exists $\Dalt$ and $\Halt$ such that
\begin{align}
 \jointloss(\Dalt \Halt) + \frac{\regwgt\hscale}{2}\sum_{i=1}^\ldim \| \Dalt_{:,i} \|_c^2 + \frac{\regwgt}{2\hscale}\sum_{i=1}^\ldim \| \Halt_{i,:}\|_r^2
&< 
 \jointloss(\Dmat^* \Hmat^*) +  \frac{\regwgt\hscale}{2}\sum_{i=1}^\ldim \| \Dmat^*_{:,i}/\sqrt{\hscale} \|_c^2 + \frac{\regwgt}{2\hscale}\sum_{i=1}^\ldim \|\sqrt{\hscale}\Hmat^*_{i,:}\|_r^2 \nonumber\\
 &=
 \jointloss(\Dmat^* \Hmat^*) +  \frac{\regwgt}{2}\sum_{i=1}^\ldim \| \Dmat^*_{:,i} \|_c^2 + \frac{\regwgt}{2}\sum_{i=1}^\ldim \| \Hmat^*_{i,:} \|_r^2 \label{eq_contradiction}
\end{align}
where the second inequality is due to the fact that $\| \dvec/\sqrt{s}\|_c^2 = \| \dvec \|_c^2/s$ and similarly for $\| \cdot \|_r$. 
The strict inequality in \eqref{eq_contradiction}
is a contradiction of the fact that $\Dmat^*$ and $\Hmat^*$ are the minimizers of \eqref{eq_noscalep1}. 
Therefore,
$\Dmat^* / \sqrt{\hscale}$ and $\sqrt{\hscale}\Hmat^*$ are minimizers of \eqref{eq_scalep1}.

Similarly, if $\Dmat^*, \Hmat^*$ are minimizers of \eqref{eq_scalep1}, 
then $\sqrt{\hscale}\Dmat^*$ and $\Hmat^*/\sqrt{\hscale}$ are minimizers of \eqref{eq_noscalep1}. 
Therefore, the minimum value for \eqref{eq_scalep1} and \eqref{eq_noscalep1}
is equal.
\end{proof}

This proposition is largely subsumed by Proposition \ref{prop_generalization}, but serves to highlight the relationship 
between the regularization parameters for the factored form and the induced form.  
For example, for subspace dictionary learning, a natural choice is to scale $\| \Hmat \|_F^2$ with $\tfrac{1}{\nsamples}$, to correspond to the scaling on the loss $\tfrac{1}{\nsamples}\| \Dmat \Hmat - \Xmat \|_F^2$. This implies that when using the induced convex form, the corresponding scale on the trace norm should be $\sqrt{\nsamples}$. For more general $f_c, f_r$ chosen, however, this relationship is no longer clear. We can still scale the regularizer on $\Hmat$ with $\hscale > 0$, because the resulting $f_r/\hscale^2$ still satisfies the
conditions of Proposition \ref{prop_generalization} and so a convex induced regularizer $\regz$ is guaranteed
to exist. However, the impact of the choice of $\hscale$ on the regularization weight in front of $\regz$---and in fact the impact on $\regz$ itself---is no longer clear. 

\newcommand{\reglittleh}{\text{R}_h}

\subsection{Regularizers decoupled across samples}\label{sec_decoupled}

To enable incremental estimation, where each sample is processed incrementally we need to design objectives
that factor across samples. This requires that the regularizer on $\Hmat$ factors across the columns of $\Hmat$: $\regh(\Hmat) = \sum_{\sampiter = 1}^\nsamples \reglittleh(\Hmat_{:\sampiter})$ for some function $\reglittleh: \RR^\ldim \rightarrow \RR$.
The formulation for identified induced \RFMs, however, requires that the regularizer factors across columns: 
$\regh(\Hmat) = \sum_{i = 1}^\ldim \frow(\Hmat_{i:})$. In this section, we discuss options to select regularizers that satisfy both.  

For the incremental setting, we need to define the objective used for each sample. We assume that $\jointloss(\Dmat \Hmat) = \frac{1}{\nsamples}\sum_{\sampiter=1}^\nsamples L_x(\Dmat \Hmat_{:\sampiter}, \Xmat_{:\sampiter})$, for  i.i.d. samples $\xvec_1, \ldots, \xvec_\nsamples$. Because $\Hmat$ grows with the number of samples, we only maintain $\Dmat$ explicitly. To do so, similarly to previous work on incremental dictionary learning \cite{bottou1998online,mairal2009online}, we consider the following objective
\begin{align}
l_\sampiter(\Dmat) = \big(\min_{\hvec} L_x(\Dmat \hvec, \xvec_\sampiter) + \reglittleh(\hvec) \big)+ \sum_{i=1}^\ldim \fcol^2(\Dmat_{:i})
\end{align}
Each stochastic gradient descent step consists of stepping in the direction $-\nabla l_\sampiter(\Dmat)$ for each sample $\sampiter$. 
If $l_\sampiter(\Dmat)$ is an unbiased estimate of $\jointloss(\Dmat \Hmat) + \regh(\Hmat) + \regd(\Dmat)$,
then standard theoretical results imply that stochastic gradient descent will converge to a stationary point of the batch objective.

Unfortunately, if $\regh(\Hmat)$ cannot be written as $\sum_{\sampiter = 1}^\nsamples \reglittleh(\Hmat_{:\sampiter})$,
then $l_\sampiter(\Dmat)$ is not necessarily unbiased.
To see why, consider the setting with $f_r^2(\cdot) = \| \cdot \|_2^2$, giving
\begin{equation*} 
\regh(\Hmat) = \frobsq{\Hmat} = \sum_{i=1}^\ldim \|\Hmat_{i:} \|_2^2 = \sum_{\sampiter=1}^\nsamples \|\Hmat_{:\sampiter} \|_2^2
.
\end{equation*}
where the regularizer decomposes across columns. Then we get
\begin{align*}
\mathbb{E}[l_t(\Dmat)] &= \tfrac{1}{\nsamples} \sum_{\sampiter=1}^\nsamples \mathbb{E}[l_t(\Dmat)]\\
&= \tfrac{1}{\nsamples} \sum_{\sampiter=1}^\nsamples \mathbb{E}[\min_{\hvec} L_x(\Dmat \hvec, \xvec_t) + \tfrac{\regwgt}{2} \|\hvec \|_2^2] \\
&= \mathbb{E}\left[ \tfrac{1}{\nsamples} \sum_{\sampiter=1}^\nsamples \min_{\Hmat_{:\sampiter}} \left( L_x(\Dmat \Hmat_{:\sampiter}, \Xmat_{:\sampiter}) + \tfrac{\regwgt}{2}  \| \Hmat_{:\sampiter} \|_2^2 \right) \right] \\
&= \mathbb{E}\left[\min_\Hmat \tfrac{1}{\nsamples} \sum_{\sampiter=1}^\nsamples \left( L_x(\Dmat \Hmat_{:\sampiter}, \Xmat_{:\sampiter}) + \tfrac{\regwgt}{2} \| \Hmat_{:\sampiter} \|_2^2 \right) \right] \\
&= \mathbb{E}\left[\min_\Hmat \left(\jointloss(\Dmat\Hmat) + \tfrac{\regwgt}{2\nsamples}  \sum_{\sampiter=1}^\nsamples \| \Hmat_{:\sampiter} \|_2^2 \right) \right] \\
&= \mathbb{E}\left[\min_\Hmat \left(\jointloss(\Dmat\Hmat) + \tfrac{\regwgt}{2\nsamples}  \sum_{i=1}^\ldim \| \Hmat_{i:} \|_2^2 \right) \right]
.
\end{align*}
The equivalence occurs because we can swap $i$ and $\sampiter$ in the last step. 
In the sparse setting, on the other hand, with $\frow^2(\cdot) = \| \cdot \|_1^2$, we can no longer swap $i$ and $\sampiter$ because
\begin{align*}
\sum_{i=1}^\ldim \|\Hmat_{i:} \|_1^2 = \sum_{i=1}^\ldim \left(\sum_{\sampiter=1}^\nsamples |\Hmat_{ij} | \right)^2 
\neq \sum_{\sampiter=1}^\nsamples \|\Hmat_{:\sampiter} \|_1^2
\end{align*}
Therefore, $l_\sampiter(\Dmat)$ could be a biased estimate
of the expected loss. A natural alternative is to use $\| \cdot \|_1$ without squaring, since 
\begin{equation*}
\sum_{i=1}^\ldim \|\Hmat_{i:} \|_1 =  \sum_{\sampiter=1}^\nsamples \|\Hmat_{:\sampiter} \|_1
.
\end{equation*}
Technically, however, this no longer precisely fits into the formalism, because $\frow$ must be squared in the summed form
and $\frow = \sqrt{\ell_1}$---which would give $\frow^2 = \ell_1$---is no longer a convex function. 

For incremental estimation, therefore, we need to more carefully select the regularizers.
Fortunately, the flexibility of the induced \RFM\ formalism provides some recourse. For sparse coding, for example,
the producted or constrained forms both use $\ell_1$ instead of $\ell_1^2$.
In previous work on incremental sparse coding \cite{mairal2010online},
the constrained form was used
\begin{align*}
\min_{\substack{\Dmat \in \RR^{\xdim \times \ldim}, \ \Hmat \in \RR^{\ldim \times \nsamples} \\ \| \Dmat_{:i} \|_c \le 1}} \jointloss(\Dmat \Hmat)  + \regwgt \sum_{i=1}^\ldim \| \Hmat_{i:} \|_1
.
\end{align*}
The producted form also avoids the square on $\ell_1$, and additionally does not require a constrained optimization to be solved.\footnote{The required projection for more complicated regularizers on $\Dmat$ can significantly impact computation \cite{hazan2012projection}.}
Additionally, though not theoretically shown, we have empirically found that using $\ell_1$ instead of $\ell_1^2$ within the summed form also gives global solutions.
As discussed later in this document, we hypothesize that because $\frow = \sqrt{\ell_1}$ is close to being a convex function,
it enjoys similar global optimality properties as $\frow$ that are convex.

\paragraph{Remark:}
Once $\Dmat$ is learned, 
out-of-sample prediction for unsupervised learning is done by solving the following objective, for a new sample $\xvec$
\begin{align*}
\min_{\hvec} L_x(\Dmat \hvec, \xvec) + \regwgt \reglittleh(\hvec)
.
\end{align*}
%
Therefore, even for the batch setting, it makes sense to select the
objective that decouples the columns of $\Hmat$, for more effective out-of-sample prediction. 

{
 \def\arraystretch{2} 
 \setlength\tabcolsep{0.3cm} 
  \begin{table*}[htp!]
\begin{center}
\begin{small}
\begin{sc}
\begin{tabular}{| l || c |}
\hline  
Setting & Batch Loss  \\
\hline
Subspace 
    & $\tfrac{1}{\nsamples} \jointloss(\Dmat \Hmat)  + \tfrac{\regwgt}{2 \nsamples} \| \Hmat \|_F^2 +  \tfrac{\regwgt}{2} \| \Dmat \|_F^2$ \\
    \hline
Matrix completion 
    & $ \frac{1}{\# \text{observed}}\sum_{\text{observed } (i,j)} L_x(\Dmat_{i:} \Hmat_{:j}, \Xmat_{ij})  + \tfrac{\regwgt}{2 \nsamples} \| \Hmat \|_F^2 +  \tfrac{\regwgt}{2 \xdim} \| \Dmat \|_F^2$ \\
\hline
Sparse coding ($q \ge 1$)
    & $\tfrac{1}{\nsamples} \jointloss(\Dmat \Hmat)  + \tfrac{\regwgt}{2} \| \Dmat \|_F^2 + \tfrac{\regwgt}{2\nsamples} \sum_{i=1}^\ldim \| \Hmat_{i:} \|_1$ \\
\hline
Elastic net (set $\nuh = 1$)
    & $\tfrac{1}{\nsamples} \jointloss(\Dmat \Hmat)  + \tfrac{\regwgt}{\nsamples} \frobsq{\Hmat}
    + \regwgt\nud \frobsq{\Dmat} + \regwgt(1-\nud)\sum_{i=1}^\ldim \|\Dmat_{:,i}\|^2_{1}$ \\
\hline
Elastic net (set $\nuh = 0$)
    & $\tfrac{1}{\nsamples} \jointloss(\Dmat \Hmat)  + \tfrac{\regwgt}{\nsamples} \|\Hmat\|_{1,1}
    + \regwgt\nud \frobsq{\Dmat} + \regwgt(1-\nud)\sum_{i=1}^\ldim \|\Dmat_{:,i}\|^2_{1}$\\
\hline
Supervised dict. learning 
    & $\tfrac{1}{\nsamples} L\left(\inlinevec{\Dmatone}{\Dmattwo} \Hmat, \inlinevec{\Xmat}{\Ymat}\right) + \tfrac{\regwgt}{\nsamples} \|\Hmat\|_{1,1}
    + \regwgt \sum_{i=1}^\ldim \max\left( \| \Dmatone_{:i}\|_c^2, \|\Dmattwo_{:i} \|_c^2\right)$\ \\
\hline
\end{tabular}
\end{sc}
\end{small}
\caption {Our proposed preferred objectives for batch and stochastic gradient descent for induced \RFMs.
The reasons are given for these objectives throughout Section \ref{sec_objectives},
based mainly on obtaining unbiased objectives for incremental estimation and from empirical investigation into
their stability and optimality properties.
For the supervised representation learning objectives,
the column norm on $\Dmat$ can be any norm,
including the elastic net norm: $\nud \frobsq{\Dmat} + (1-\nud)\sum_{i=1}^\ldim \|\Dmat_{:,i}\|^2_{1}$. 
} \label{table_objectives} 
\end{center}
\end{table*}
}

In this section, we discussed how the objective can be specified in multiple ways, with
similar or equivalent modeling properties. In Table \ref{table_objectives},
we summarize what we believe are effective choices for the induced \RFMs\ defined in Section \ref{sec_fms}.
In the next sections, we demonstrate
theoretically and empirically that alternating minimization on induced \RFM\ objectives produces global solutions. In particular, we also provide evidence that moving
outside the class of induced \RFMs\ loses this property.
The result is actually hopeful: we can globally optimize a wide-range
of representation learning problems, with an appropriately chosen objective. 

\section{Local minima are global minima for a subclass of induced \RFMs}

Our main theoretical result is to show that,
for an appropriately chosen inner dimension $\ldim$, 
local minima are actually global minima.
The key novelty over previous work is to characterize the overcomplete setting, with $\ldim > \xdim$,
with full-rank solutions, whereas previous work has generally analyzed rank-deficient solutions.
Combined with previous results on rank-deficiency, there is compelling evidence that even
though the \RFM\   objective is nonconvex, alternating minimization between $\Hmat$ and $\Dmat$ 
should converge to a global solution.
Later, in Section \ref{sec_empirical} we illustrate empirically that this global optimality result additionally holds for induced \RFMs\ not covered by the theory,
but that it does not hold for two slight modifications that take the objective out of the class of induced \RFMs.
These theoretical and empirical insights constitute a significant step towards the conjecture
proposed in this work, that alternating minimization for induced \RFMs\ produces global solutions. 

\subsection{Previous work and moving to the full rank setting}

There has been significant effort towards understanding local minima for \RFMs, and related objectives. 
The most related and most general result to-date has been given by \citet{haeffele2015global}, which mostly characterizes the rank-deficient setting.
We give the definition of rank-deficiency more formally here.
\begin{definition}
A stationary point $(\Dalt, \Halt)$ with $\Dalt \in \RR^{\xdim \times \ldim}, \Halt \in \RR^{\ldim \times \nsamples}$ is \textbf{rank-deficient} 
if $\ldim$ is strictly greater than the rank of both $\Dalt$ and $\Halt$. 
\end{definition}
For a rank-deficient $(\Dalt, \Halt)$, one of these variables may still be full rank. For example, if $\xdim < \ldim < \nsamples$, then
$\Dalt$ could be full rank (rank $\xdim$); however, $\Halt$ cannot be full rank, as that would violate the definition of rank-deficiency. Therefore, rank-deficiency describes the rank of the pair, rather than the ranks of each individual matrix, reflecting that at least one of the variables must be reduced-rank. 

The theoretical properties of stationary points for the rank-deficient settings has been characterized for a surprisingly broad set of problems, which mostly encompasses induced \RFMs\footnote{The result by \citet{haeffele2015global} mostly includes  induced \RFMs, but requires positive homogeneity for the regularizers, whereas we only require that the regularizers be centered, convex functions.}. 
\citet{haeffele2015global} show that, when $\ldim < \nsamples$, if $\Dmat, \Hmat$ are rank-deficient local minima, and an entire column of $\Dmat$ or row of $\Hmat$ is zero, then that local minimum is a global minimum. Several other works have found similar properties for the rank-deficient setting (see Section \ref{sec_related}), though none for such a general setting. 

The full-rank setting, however, where $\Dmat$ and $\Hmat$ are both full rank, is not as well understood. For sparse coding, the results do not require rank deficiency, but characterize a different objective and require a careful initialization strategy to ensure convergence to global optima (again see Section \ref{sec_related}). The result by \citet{haeffele2015global} does apply to full-rank $\Dalt, \Halt$, but only when $\ldim > \nsamples$ and again still requires that an entire column of $\Dalt$ or row of $\Halt$ is zero. Because the rank-deficient setting is comparatively more thoroughly characterized, we focus our theoretical investigation on the full rank case.  

One key modification beyond previous work---to make the result meaningful for the full-rank case---is to allow smaller $\ldim$ such that $\regzk \neq \regz$. We can expand the set of of considered induced \RFMs\ by instead only requiring that $\ldim$ be sufficiently large so that $\regzk$ is convex, where 
 \begin{equation}
\regzk(\joint) \defeq 
        \min_{\substack{\Dmat \in \RR^{\xdim \times \ldim}, \ \Hmat \in \RR^{\ldim \times \nsamples}\\ \joint = \Dmat \Hmat}} \sum_{i=1}^\ldim \left( \fcol^2(\Dmat_{:i})  +  \frow^2(\Hmat_{i:}) \right) 
\tag{\ref{eq_regk}}  
\end{equation}
As shown in Proposition \ref{prop_generalization}, there exists a sufficiently large $\ldim^*$ such that $\text{R}_{\ldim^*} = \regz$ and so is guaranteed to be convex. However, $\regzk$ may be convex for smaller $\ldim$ and empirically we find that this seems to be the case. 
This is contrary to common wisdom that for these models we require $\ldim$ to be as large as the induced rank to obtain global solutions (see \cite{bach2008convex,zhang2011convex}).
For example, for sparse coding, the true induced rank is $\ldim^* = \nsamples$, where an optimal solution consists of memorizing the training samples. In practice, however, one would almost definitely set $\ldim < \ldim^*$.

This generalization is key because otherwise it is likely that the result would only apply to impractically large $\ldim$. For the rank-deficient setting, by definition, $\ldim$ was already sufficiently large because the stationary point did not use that additional parametrization. For the full-rank setting, however, it is feasible that increasing $\ldim$ and adding more parameters could further decrease the objective. For example, for sparse coding, it is possible that by increasing $\ldim$ towards $\ldim^* = \nsamples$, the objective could be further reduced. However, it could still be the case that, for $\ldim$ smaller than $\ldim^*$, a full-rank local minimum $\Dmat, \Hmat$ is a global minimum. Stationary points with $\ldim$ smaller than $\ldim^*$ would be unlikely to be stationary points for the induced $\regz$, but could be stationary points of $\regzk$. Because the proof uses the induced form---with $\regzk$---to characterize the stationary points of the factored form, if we only considered $\regzk = \regz$ for the full-rank setting, we would severely limit the scope of the result.  

\subsection{Theoretical result for the induced form}

We first provide the more general result, for optimality of alternating minimization on the induced form. 
We characterize the optimization surface for the minimization over $\Dmat$ and $\Hmat$ for the induced form
\begin{equation*}
\min_{\Dmat \in \RR^{\xdim \times \ldim}, \ \Hmat \in \RR^{\ldim \times \nsamples}} \jointloss(\Dmat \Hmat) + \regzk(\Dmat \Hmat)
.
\end{equation*}
We show that, despite the fact that this is a nonconvex optimization over $\Dmat, \Hmat$, that all full rank stationary points are actually global minima. This is a surprisingly strong result. Despite the fact that we do not directly optimize the induced form, this result may help to explain why empirically alternating minimization on the factored form typically produces global solutions.  This result---more so than the more limited result for the factored form---may suggest that the optimization surface for induced \RFMs\ is well-behaved. 

Because the induced problem may not be differentiable, due to $\regzk$,
we turn to more general definitions of directional derivatives \cite{ivanov2015second}. 
The lower Hadamard directional derivative exists for any continuous function; for a real-valued function $g: \RR^{\xdim \times \nsamples} \rightarrow \RR$, it is defined as
\begin{align*}
D_Z g (\joint)[\Umat] = \lim_{t \downarrow 0} \inf_{\Umat' \rightarrow \Umat} t^\inv [f(\joint + t \Umat') - f(\joint)]
.
\end{align*}
These directional derivatives become the standard Frechet derivatives, when those exist,
as they do for differentiable functions.
\begin{theorem}[Non-differentiable losses]\label{thm_convex}
Let $g: \RR^{\xdim \times \nsamples} \rightarrow \RR$ be any real-valued, continuous convex function, where the lower Hadamard directional derivative exists.
Let $(\Dalt,\Halt)$ be a stationary point of
\begin{align*}
\min_{\genrankset} g(\Dmat\Hmat).
\end{align*}
%
If \textbf{either}
\begin{enumerate} 
\item $\Dalt$ is full rank and $\ldim \ge \min(\xdim, \nsamples)$ \textbf{or}
\item $\Halt$ is full rank and $\ldim \ge \min(\xdim, \nsamples)$
\textbf{or} 
\item $(\Dalt, \Halt)$ is a local minimum and is rank-deficient (i.e., 
$r = \max(\rank(\Dalt), \rank(\Halt)) < \ldim$ )
\end{enumerate}
then $(\Dalt,\Halt)$ is a global minimum.
\end{theorem}
\begin{proof}
We need to show two cases. For the overcomplete setting, we consider full rank $\Dalt$,
the argument is similar for full rank $\Halt$. 
Notice that
\begin{align*}
D_H g(\Dalt \Halt)[d\Hmat] \ge 0
\end{align*}
by definition, since $\Dalt, \Halt$ is a stationary point and $g$ is convex, making $g(\Dalt \cdot)$ convex in $\Hmat$. By the chain rule
\begin{align*}
D_H g(\Dalt \Halt)[d\Hmat] = D_Z g(\Dalt \Halt)[\Dalt d\Hmat]. 
\end{align*}
The chain rule follows, because $g(\Dalt (\Halt + t d\Hmat)) = g(\Dalt \Halt + t \Dalt d\Hmat)$.
Because $\Dalt$ is full rank, $\Dalt d\Hmat$ constitutes all possible directions $d\joint$. 
Therefore, in all directions, $D_Z g(\Dalt \Halt)[d\joint] \ge 0$; since this is a convex function,
$\Dalt \Halt$ must be a global minimum.
 
For the second setting, where $\min(\rank(\Dalt), \rank(\Halt)) = r < \min(\xdim,\ldim)$, we show that local optimality implies global optimality.
The key is to use the fact that $\Dalt$ has singular values that are zero. 
The first directional derivative in direction $(d\Dmat, d\Hmat)$, for the stacked variable $\Ymat = (\Dmat, \Hmat)$ is 
\begin{align*}
D_Y g(\Dalt \Halt)[(d\Dmat, d\Hmat)] = D_Z g(\Dalt \Halt)[\Dalt d\Hmat + d\Dmat\Halt]
.
\end{align*}
The second directional derivative, by the chain rule \citep[Theorem 1.127]{mordukhovich2006variational}, is
\begin{align*}
D_Y^2 g(\Dalt \Halt)[(d\Dmat, d\Hmat), (d\Dmat, d\Hmat)] 
&= D_Y (D_Z g(\Dalt \Halt)[\Dalt d\Hmat + d\Dmat\Halt] )[(d\Dmat, d\Hmat)]\\
&= 2 D_Z g(\Dalt \Halt)[d\Dmat d\Hmat] + D_Z^2 g(\Dalt \Halt)[\Dalt d\Hmat + d\Dmat\Halt,\Dalt d\Hmat + d\Dmat\Halt]
.
\end{align*}
Let $\Dalt = \Umat_1 \Sigmamat_1 \Vmat_1^\top$ and $\Halt = \Umat_2 \Sigmamat_2 \Vmat_2^\top$. 
Any directions $d\Dmat, d\Hmat$ can be written
\begin{align}
d\Hmat &= \Vmat_1 [\hvec_1, \ldots, \hvec_{\ldim}]^\top \label{eq_directions}\\
d \Dmat &= [\dvec_1, \ldots, \dvec_{\ldim}] \Umat_2^\top \nonumber
\end{align}
for vectors $\hvec_i \in \RR^\nsamples, \dvec_i \in \RR^\xdim$. 

We consider the more specific choice $d\Dmat_r, d\Hmat_r$ with $\hvec_1, \ldots, \hvec_{\rdim} = \zerovec$ and 
$\dvec_1, \ldots, \dvec_{\rdim} = \zerovec$.
The rotation in $d\Hmat_r$ results in 
$\Dalt d\Hmat_r = \Umat \Sigmamat_1 [\zerovec, \ldots, \zerovec, \hvec_{\rdim+1}, \ldots, \hvec_{\ldim}]^\top = \zerovec$.
Similarly, $d \Dmat_r \Halt = \zerovec$. Therefore, $D^2_\joint g(\joint)[\Dalt d\Hmat_r + d\Dmat_r\Halt,\Dalt d\Hmat_r + d\Dmat_r\Halt] = D^2_\joint g(\joint)[\zerovec,\zerovec] = \zerovec$.
Therefore, for all such $d\Dmat_r, d\Hmat_r$, 
    \begin{equation*}
D_Z g(\Dalt \Halt)[d\Dmat_r d\Hmat_r]  = D_Y^2 g(\Dalt \Halt)[(d\Dmat, d\Hmat), (d\Dmat, d\Hmat)]  \ge 0
  \end{equation*}
where the second directional derivate is greater than zero because $\Dalt,\Halt$ is a local minimum.

Using this, we can show that 
$\zerovec \in \partial g(\Dalt \Halt)$, where $\partial g(\Dalt \Halt)$
is the subdifferential of $g$ at $\bar\joint = \Dalt \Halt$. 
Because $g$ is convex, we know that $D_Z g(\bar\joint)[d\joint] = \sup_{\Gmat \in \partial g(\bar\joint)} \trace{\Gmat^\top d\joint}$
with the set $\partial g(\bar\joint)$ convex. 
Consider the simpler setting where $\dvec_{\rdim+1}, \hvec_{\rdim+1}$ can be any $\ldim$-dimensional vectors,
and we set the other vectors in the directions $d\Dmat_r$ and $d\Hmat_r$ to zero. 
If $\partial g(\Dalt \Halt)$ is a singleton $\{\Gmat \}$, then we get, for directions $\hvec_{\rdim+1} = \hvec$ 
and $\dvec_{\rdim+1} = -\Gmat \hvec$ that $0 \le \trace{\dvec_{\rdim+1}^\top \Gmat \hvec_{\rdim+1} } = 
-\trace{\hvec^\top \Gmat^\top\Gmat \hvec}$ which implies that $\Gmat = \zerovec$.
More generally, let $\Gmat_1$ and $\Gmat_2$ correspond to these subderivatives for $d\Dmat_r$
and $-d\Dmat_r$. Therefore, $\trace{\Gmat_1^\top d\Dmat_r d\Hmat_r}  \ge 0$ and 
$-\trace{\Gmat_2^\top d\Dmat_r d\Hmat_r} \ge 0$. 
If $\Gmat_2 = -\Gmat_1$, by convexity of the subdifferential, $\zerovec \in \partial g(\Dalt \Halt)$. 
More generally, if the set does not contain both $\Gmat_1$ and $-\Gmat_1$, then the point
that has highest product with $-d\Dmat_r$ must be the projection to the set from $-\Gmat_1$,
and so $\Gmat_2 = (1-t) \Gmat_1$, i.e., $\Gmat_2$ is between $\Gmat_1$ and $-\Gmat_1$ for $t \in [0,2]$. 
This implies that $-\trace{\Gmat_2^\top d\Dmat_r d\Hmat_r} = (t-1)\trace{\Gmat_1^\top d\Dmat_r d\Hmat_r} \ge 0$
if and only if $t \ge 1$. This would again imply $\zerovec \in \partial g(\Dalt \Halt)$, by convexity of the subdifferential.

Therefore, since $\zerovec \in \partial g(\Dalt \Halt)$ 
and because $g$ is convex, this implies $\Dalt, \Halt$ is a global minimum. 
\end{proof}

This result shows that the stationary points of the induced form are well-behaved. 
This can be seen by setting $g(\joint) = \jointloss(\joint) + \regzk(\joint)$, with $\ldim$ sufficiently large to ensure that $\regzk$ is convex.
The regularizer $\regzk$ may not be smooth (e.g., trace norm), but it will be continuous, and the theorem only requires continuity for $g$.  For the overcomplete setting, any stationary point with full rank $\Dalt$ or $\Halt$ is guaranteed to be a global minimum. 
For the rank-deficient setting, if this point is a local minimum, then it is a global minimum. 
This result, therefore, characterizes even the saddlepoints for the full-rank setting---since it characterizes all stationary points.

\subsection{Theoretical result for the factored form}

We now provide a more restricted result for the factored form, which is the objective we will optimize in practice. 
We provide a novel proof strategy to characterize the stationary points of the factored form using the results from the induced form.
We highlight a necessary condition---the induced regularization property---for optimality for the factored form, given in Definition \ref{def_irp}, and, given this condition, we can prove an analogously strong result to Theorem \ref{thm_convex}. We then show a general setting for the factored form that guarantees that this condition is satisfied for local minima, only requiring that the regularizers be strictly convex. We expect that more induced \RFMs\ can be shown to satisfy the induced regularization property, and so leverage the general result in Theorem \ref{thm_convex}. 

We formally characterize this result in the remainder of this section. We use the following assumption. 
 \begin{assumption}
   The loss $\jointloss: \RR^{\xdim \times \nsamples} \rightarrow \RR$ and the regularizers $\fcol:\RR^\xdim \rightarrow \RR^+, \frow: \RR^\nsamples \rightarrow \RR^+$ are all differentiable convex functions. Further, as required for induced \RFMs, $\frow$ and $\fcol$ are centered, non-negative functions. 
\end{assumption}
This assumption does not allow non-smooth regularizers, such as the $\ell_1$ or elastic net regularizer.
However, due to the generalization in Proposition \ref{prop_generalization} to any non-negative convex $\fcol$ and $\frow$,
we can use smooth approximations to these regularizers. The theory applies to these smooth approximations 
even for parameter selections that make them arbitrarily close to the non-smooth regularizer. Intuitively, this suggests that the results
extend to the non-smooth setting; we leave this generalization to future work.

We introduce a condition that is necessary for obtaining optimality for the full rank case.
We state this property as an explicit condition to highlight it as a key property
for the proof. We then demonstrate one setting where this assumption is satisfied, in Proposition \ref{prop_strongly}. 
We hope that by identifying this condition as key, it can help further identify
more \RFMs\ that satisfy this condition, for the setting with full rank solutions. 
\begin{definition}[Induced regularization property]\label{def_irp}
A stationary point $\Dmat,\Hmat$ of the objective $\{\jointloss(\Dmat\Hmat) + \regd(\Dmat) + \regh(\Hmat)\}$ satisfies the \textit{induced regularization property} if
\begin{align}
\displaystyle \regd(\Dmat) + \regh(\Hmat) = \min_{\substack{\Dmat_k \in \RR^{\xdim \times \ldim}, \Hmat_k \in \RR^{\ldim \times \nsamples}\\ \Dmat \Hmat = \Dmat_k \Hmat_k}}  \regd(\Dmat_k) + \regh(\Hmat_k) \label{eq_key}
\end{align}
and each possible solution $\Dmat_k, \Hmat_k$ is itself a stationary point of the objective. 
\end{definition}
This property essentially states that the point $\Dmat, \Hmat$ cannot be reweighted to create $\Dmat_k, \Hmat_k$
that has same $\joint = \Dmat_k\Hmat_k = \Dmat \Hmat$ but with a lower objective value. If such a $\Dmat_k, \Hmat_k$ existed, 
they could only have a lower objective value because of smaller regularization terms, which would imply that $\Dmat, \Hmat$ could not be a global minimum. 
Further, even if the regularization term is equivalent for all $\Dmat_k, \Hmat_k$, if any of the equivalent $\Dmat_k, \Hmat_k$ are not stationary points, then we would obtain a descent direction to further decrease the objective. This would again
mean that $\Dmat, \Hmat$ is not a global minimum. The induced regularization property, therefore, is a necessary condition. 

We provide our main result for the factored form in Theorem \ref{thm_irfm}. We focus the characterization on full rank stationary points. For the rank-deficient setting, we can leverage the results\footnote{Their general factorization objective does not fully cover all of the induced \RFMs\ considered here, because they require positively homogenous regularizers, rather than just convexity.} by \citet{haeffele2015global}, to show that local minima are global minima, similarly to Condition 3 in Theorem \ref{thm_convex}. We therefore focus the characterization in Theorem \ref{thm_irfm} on full rank stationary points. 
\begin{theorem}[General regularizers on $\Dmat, \Hmat$]\label{thm_irfm}
Let $(\Dalt,\Halt)$ be a stationary point of
%
\begin{align}
\min_{\genrankset} \jointloss(\Dmat\Hmat) + \regd(\Dmat) + \regh(\Hmat)
\label{eq_generic}
\end{align}
where $\regd(\Dmat) = \sum_{i=1}^k \fcol^2(\Dmat_{:i} )$ and 
$\regh(\Hmat) = \sum_{i=1}^k \frow^2(\Hmat_{i:} )$ and the loss and regularizers satisfy the conditions of Assumption 1.
Let $\regzk$ be the induced norm given $\fcol$ and $\frow$, as defined in Equation \eqref{eq_regk}.
If 
 \begin{enumerate}[itemindent=0.5cm]
 \item $\regzk$ is convex
 \item Both $\Dalt$ and $\Halt$ are full rank 
 \item $(\Dalt,\Halt)$ satisfies the induced regularization property (Definition \ref{def_irp})
 \end{enumerate}
then $(\Dalt, \Halt)$ is a globally optimal solution to \eqref{eq_generic}.
\end{theorem}
\begin{proof}
We would like to relate $(\Dalt, \Halt)$ to the stationary points of the induced problem
\begin{align}
\argmin_{\genrankset} \jointloss(\Dmat \Hmat) +  \regzk(\Dmat \Hmat)
.
\label{eq_inducedk}
\end{align}
According to Theorem \ref{thm_convex}, if we can show that $\Dalt, \Halt$ are stationary points of \eqref{eq_inducedk},
then we know they are global minima. Under the induced-regularization property, the optimal solution to \eqref{eq_inducedk} corresponds
to the optimal solution of \eqref{eq_generic}, and so we would correspondingly know this is a global minima
for \eqref{eq_generic} and be done. 

To relate these optimizations, we rewrite the optimization in \eqref{eq_inducedk} for full rank $\Dmat, \Hmat$ using the factored form
with an additional optimization over $\Qmat = \inlinevec{\Amat}{\Bmat} \in \RR^{2\ldim \times \ldim}$
\begin{align*}
\regzk(\Dmat \Hmat)
&= \frac{1}{2} \min_{\substack{\Dmat_k \in \RR^{\xdim \times \ldim}, \Hmat_k \in \RR^{\ldim \times \nsamples}\\ \Dmat \Hmat = \Dmat_k \Hmat_k}}  \regd(\Dmat_k) + \regh(\Hmat_k) \nonumber\\ 
&= \frac{1}{2} \min_{\Amat, \Bmat \in \RR^{\ldim \times \ldim}: \Dmat \Amat \Bmat \Hmat = \Dmat \Hmat}  \regd(\Dmat \Amat) + \regh(\Bmat\Hmat)
.
\end{align*}
The second equality follows from the fact that a minimal such pair $\Dmat_k, \Hmat_k$ can always be written 
$\Dmat_k = \Dmat \Amat$ and $\Hmat_k = \Bmat \Hmat$, because $\Dmat$ and $\Hmat$ are both full rank. 
%
Consequently, we will instead consider
\begin{align}
\jointloss(\Dmat \Hmat) + \regzk(\Dmat\Hmat)
&=
\min_{\Amat, \Bmat \in \RR^{\ldim \times \ldim}: \Dmat \Amat \Bmat \Hmat = \Dmat \Hmat} \jointloss(\Dmat \Hmat) +  \regd(\Dmat \Amat) + \regh(\Bmat\Hmat)
.
\label{eq_inducedaltq}
\end{align}
%

%
For a differentiable function $f$, 
the directional derivative is simpler and can be written as
\begin{align*}
D_Y f (\Ymat)[\Umat] = \lim_{t \downarrow 0} t^\inv [f(\Ymat + t \Umat) - f(\Ymat)]
\end{align*}

We use Danskin's theorem to characterize the stationary point for the induced problem.
We provide a more specific, simpler proof here, by taking advantage of some of the properties of our function. For a nice overview on the more general Danskin's theorem,
see \citep[Theorem 4.1]{bonnans1998optimization} and \citep[Theorem 1.29]{guler2010foundations}.
Let $\pairvar_0 = [\Dalt; \Halt^\top] \in \RR^{(\xdim+\nsamples) \times \ldim}$ and for $\Ymat = [\Dmat; \Hmat^\top]$ define
\begin{equation*}
J\left(\Ymat, \Qmat = \inlinevec{\Amat}{\Bmat}\right) = \jointloss(\Dmat \Hmat) +  \regd(\Dmat \Amat) + \regh(\Bmat\Hmat)
.
\end{equation*}
Pick an arbitrary direction $\Umat = [d\Dmat; d\Hmat^\top] \neq 0$ and 
let $\{\Ymat_{\seqiter}\}_{\seqiter=1}^\infty$, $\Ymat_{\seqiter} = \Ymat_0 + t_{\seqiter} \Umat$, $t_{\seqiter} \ge 0$ be a sequence
converging to $t_0 = 0$ (so $\Ymat_\seqiter \rightarrow \Ymat_0$). Define 
\newcommand{\cset}{C}
\begin{align*}
\cset(\Ymat_0) &= \left\{\inlinevec{\Amat}{\Bmat} \in \RR^{2\ldim \times \ldim}: \Dmat \Amat \Bmat \Hmat = \Dmat \Hmat \right\}\\
S(\Ymat_{\seqiter}) &= \argmin_{\Qmat \in \cset(\Ymat_\seqiter)} J(\Ymat_\seqiter, \Qmat)
.
\end{align*}

Since $\Dalt, \Halt$ are finite, we can assume that the set of possible $\Qmat$ is compact. This is because we can pick a sufficiently large ball around $\Dalt, \Halt$ to encompass the sequence $\Ymat_i$, and so restrict the definition to a compact set of $\Dmat, \Hmat$. For this compact set, in the optimization, the regularizers are non-negative, centered and convex and so the optimization will select bounded $\Amat, \Bmat$ for all $\Dmat, \Hmat$. The optimization over $\Qmat$ is guaranteed to have at least one solution, and
so $S(\Ymat_{\seqiter})$ is not empty. By boundedness of the set over $\Qmat$ and because $J(\Ymat, \Qmat)$ is continuous in $\Qmat$, for each $\Qmat_0 \in S(\Ymat_0)$, there exists a sequence $\Qmat_i \in S(\Ymat_i)$ such that $\Qmat_i \rightarrow \Qmat_0$.

Defining
\begin{equation*}
\bar{J}(\Ymat) = \min_{\Qmat \in \cset(\Ymat)} J(\Ymat, \Qmat)
\end{equation*}
we get that     
\begin{align*}
\frac{\bar{J}(\Ymat_{\seqiter}) - \bar{J}(\Ymat_0)}{t_{\seqiter}} 
&= \frac{J(\Ymat_{\seqiter}, \Qmat_{\seqiter}) - J(\Ymat_0, \Qmat_0)}{t_{\seqiter}} \\
&= \frac{J(\Ymat_{\seqiter}, \Qmat_{\seqiter}) - J(\Ymat_{\seqiter}, \Qmat_0)}{t_{\seqiter}} + \frac{J(\Ymat_{\seqiter}, \Qmat_0) - J(\Ymat_0, \Qmat_0)}{t_{\seqiter}} \\
&\le  \frac{J(\Ymat_{\seqiter}, \Qmat_0) - J(\Ymat_0, \Qmat_0)}{t_{\seqiter}}\\
&= D_Y J(\Ymat_0 + t_{\seqiter}' \Umat, \Qmat_0) [\Umat] 
\end{align*}
where the first inequality follows from $J(\Ymat_{\seqiter}, \Qmat_{\seqiter}) \le J(\Ymat_{\seqiter}, \Qmat_0)$
and the last equality from the mean value theorem, where there is some $0 \le t'_\seqiter \le t_\seqiter$
that provides a point between $\Ymat_{\seqiter}$ and $\Ymat_0$.
Taking the limit superior, we obtain
\begin{align*}
\lim_{\seqiter \rightarrow \infty} \sup\frac{\bar{J}(\Ymat_{\seqiter}) - \bar{J}(\Ymat_0)}{t_\seqiter} 
&\le D_Y J(\Ymat_0, \Qmat_0)[\Umat]
\end{align*}
Since this holds for all $\Qmat_0 \in S(\Ymat_0)$, then we must have 
\begin{align*}
\lim_{\seqiter \rightarrow \infty} \sup\frac{\bar{J}(\Ymat_{\seqiter}) - \bar{J}(\Ymat_0)}{t_\seqiter} 
&\le \min_{\Qmat \in S(\Ymat_0)}D_Y J(\Ymat_0, \Qmat)[\Umat]
\end{align*}
For the other direction, for each $\Umat$, 
there is convergent sequence $\{\Qmat_i\}$ to some $\Qmat_0 \in S(\Ymat_0)$, where $\Qmat_0$ is likely different
depending on $\Umat$.
Then 
\begin{align*}
\frac{\bar{J}(\Ymat_{\seqiter}) - \bar{J}(\Ymat_0)}{t_{\seqiter}} 
&= \frac{J(\Ymat_{\seqiter}, \Qmat_{\seqiter}) - J(\Ymat_0, \Qmat_0)}{t_{\seqiter}} \\
&= \frac{J(\Ymat_{\seqiter}, \Qmat_{\seqiter}) - J(\Ymat_{0}, \Qmat_\seqiter)}{t_{\seqiter}} + \frac{J(\Ymat_{0}, \Qmat_\seqiter) - J(\Ymat_0, \Qmat_0)}{t_{\seqiter}} \\
&\ge \frac{J(\Ymat_{\seqiter}, \Qmat_{\seqiter}) - J(\Ymat_{0}, \Qmat_\seqiter)}{t_{\seqiter}}\\
&= D_Y J(\Ymat_0 + t_{\seqiter}' \Umat, \Qmat_\seqiter)[\Umat] 
\end{align*}
for some $0 \le t'_\seqiter \le t_\seqiter$. 
Because $D_Y J$ is continuous in both arguments, 
$\lim_{\seqiter \rightarrow \infty} \inf D_Y J(\Ymat_0 + t_{\seqiter}' \Umat, \Qmat_\seqiter)[\Umat] 
= D_Y J(\Ymat_0, \Qmat_0) [\Umat]$.  
Taking now the limit inferior of the left as well, 
\begin{align*}
\lim_{\seqiter \rightarrow \infty} \inf \frac{\bar{J}(\Ymat_{\seqiter}) - \bar{J}(\Ymat_0)}{t_\seqiter} 
&\ge D_Y J(\Ymat_0,\Qmat_0)[\Umat] \\
&\ge \min_{\Qmat \in S(\Ymat_0)}D_Y J(\Ymat_0, \Qmat)[\Umat]
.
\end{align*}
Therefore, the limit inferior and superior converge to the same point, and we have
that the directional derivative for $\bar{J}(\Ymat_0)$ exists and 
\begin{equation*}
D_Y\bar{J}(\Ymat_0)[\Umat] = \min_{\Qmat \in S(\Ymat_0)} D_Y J(\Ymat_0, \Qmat)[\Umat]
.
\end{equation*} 
Because $\Ymat_0$ satisfies the induced regularization property, regardless of $\Qmat \in S(\Ymat_0)$ the resulting
$\Dmat_0 \Amat, \Bmat \Hmat_0$ is a stationary point and so $\zerovec \in \partial \bar{J}(\Ymat_0)$.
Therefore, $\Ymat_0$ is a stationary point of $\bar{J}$, which corresponds to the objective in \eqref{eq_inducedaltq}.
%

By showing that $\Dalt, \Halt$ is a stationary point of \eqref{eq_inducedaltq} for optimal $\Amat,\Bmat$, 
this shows that $\Dalt, \Halt$ is a stationary point of \eqref{eq_inducedk}. This is because in a small
ball around $\Dalt, \Halt$, the objective $ \jointloss(\Dalt \Halt) + \regzk(\Dalt\Halt)$ is equal to $ \jointloss(\Dalt \Halt) +  \min_{\inlinevec{\Amat}{\Bmat} \in \cset(\Dalt, \Halt)} \regd(\Dalt \Amat) + \regh(\Bmat\Halt)$, as the matrices
remain full rank in a small ball around $\Dalt, \Halt$. 
By Theorem \ref{thm_convex},  
$(\Dalt, \Halt)$ is a global minimum of \eqref{eq_inducedk}
and so must also be a global minimum of \eqref{eq_generic}.
\end{proof}

Now we show one setting that satisfies the induced regularization property.\footnote{We have derived one other
setting that removes the requirement on the loss; however, the conditions are more difficult to verify
and so we only include this result in the appendix (see Appendix \ref{app_prop}). We hope, nonetheless, that it can provide a useful step for future results showing the induced regularization property.} We require the regularizers to be strictly convex and the loss to be Holder continuous. 
The $\ell_2$, the elastic net and some of the smooth approximations, including the pseudo-Huber loss, all result in strictly convex regularizers on $\Dmat$ and $\Hmat$. The following result, therefore, demonstrates that full-rank local minima for a broad class of induced \RFMs\ satisfy the induced regularization property, implying that local minima are global minima. 
\begin{proposition}[Induced regularization property for strictly convex regularizers]\label{prop_strongly}
Assume that 
\begin{enumerate}
\item $\fcol$ and $\frow$ are strictly convex
\item the loss is Holder continuous: for some $m > 0$, $| \jointloss(\joint_1) -  \jointloss(\joint_2) | \le m \| \joint_1 - \joint_2 \|_F^2$
\item $\xdim \leq k \leq \nsamples$.
\end{enumerate}
If 
 \begin{enumerate}[itemindent=0.2cm]
 \vspace{-0.3cm}
 \item $(\Dalt,\Halt)$ is a local minimum of \eqref{eq_generic}
 \item $\Dalt$ and $\Halt$ are full rank
 \end{enumerate}
then $(\Dalt, \Halt)$ satisfies the induced regularization property.
\end{proposition}
\begin{proof}
First notice that $\Dalt$ cannot have a zero column. Otherwise, the corresponding row
in $\Halt$ could be set to zero. There would be a clear descent direction pushing that row
in $\Halt$ to zero (to minimize $\regh$), without changing the loss or regularizer on $\Dalt$. 
Since $\Dalt,  \Halt$ is a local minimum and $\Halt$ is full rank, this is not possible and so $\Dalt$
cannot have a zero column.

We can further constrain that $\Amat$ is also full rank. 
Notice that after some elementary column operation $\mathbf{O} \in \RR^{\ldim \times \ldim}$ on $\Dalt$, the first $\xdim$ columns of the resultant matrix $\hat{\Dmat} = \Dalt\mathbf{O}$ are full rank and the rest columns are zero. Therefore any possible $\Dmat \in \RR^{\xdim \times \ldim}$ can be written as $\hat{\Dmat}\mathbf{P}$, where $\mathbf{P}\in \RR^{\ldim \times \ldim}$. Since $\Dmat\Hmat$ are full rank, then $\Dmat$ must be full rank thus the first $\xdim$ rows of $\mathbf{P}$ are full rank, and the rest rows of $\mathbf{P}$ can be any vectors as they will not affect $\hat{\Dmat}\mathbf{P}$. Without losing any generalization, let's say for the first $\xdim$ row of $\mathbf{P}$, the first $\xdim$ columns are linearly independent. If for the last $\ldim-\xdim$ rows of $\mathbf{P}$, the first $\xdim$ columns are all zero and the rest columns are linearly independent, then $\mathbf{P}$ must be full rank. Since $\mathbf{O}$ is full rank as well, $\mathbf{O}\mathbf{P}$ is full rank. Therefore, we can assume that $\Amat$ is full rank.     

Since $\Amat = \Qmat^\inv$ for some full rank $\Qmat^\inv$, we can now consider $\Bmat$.
Notice that $\Dalt \Qmat^\inv \Bmat \Halt = \Dalt \Halt$ means we can rewrite 
\begin{equation}
\Bmat \Halt = \Qmat \Halt + \Qmat\Bmat_1 \Halt
\label{eq_nullspace}
\end{equation}
for some component $\Bmat_1 \Halt$ in the nullspace of $\Dalt$. Because $\Halt$ is full rank, this means
that $\Bmat_1$ can be written as a linear combination of the bottom $\ldim -\xdim$ right singular vectors of $\Dmat$. 

This results in a slightly different reformulation, 
\begin{align}
\regzk(\Dmat \Hmat)
&= \frac{1}{2} \min_{\substack{\Dmat_k \in \RR^{\xdim \times \ldim}, \Hmat_k \in \RR^{\ldim \times \nsamples}\\ \Dmat \Hmat = \Dmat_k \Hmat_k}}  \regd(\Dmat_k) + \regh(\Hmat_k) \nonumber\\ 
&= \frac{1}{2} \min_{\substack{\Qmat \in \RR^{\ldim \times \ldim}: \Qmat \text{ full rank}\\\Bmat_1 \in \RR^{\ldim \times \ldim}: \Bmat_1 \Hmat \in \text{nullspace}(\Dalt)}}  \regd(\Dmat \Qmat^{-1}) + \regh(\Qmat\Hmat + \Qmat\Bmat_1 \Hmat) \label{eq_frobform}
\end{align}
%
%
Let $\Qmatopt, \Bmatopt_1$ be an optimum for this reformulation. 
If $\Qmatopt, \Bmatopt_1$ strictly decreases the regularization component, 
we know $\regd(\Dalt \Qmatopt^\inv) + \regh(\Qmatopt \Halt + \Qmatopt\Bmatopt_1\Hmat) < \regd(\Dalt) + \regh(\Halt)$.
We will show that this leads to a contradiction, and so $\Qmatopt = \eye, \Bmatopt_1 = \zerovec$ must be a solution. 

Define directions 
$(1-t) \Dalt + t \Dalt \Qmatopt^\inv = \Dalt + t (\Dalt \Qmatopt^\inv - \Dalt)$ and
$(1-t) \Halt + t \Bmatopt\Halt$.
We show that the improvement in the regularizers from stepping in these directions is greater than the loss incurred from the Holder continuous loss. To see why, notice that 
\begin{align*}
\Dalt ((1-t) \eye + t \Qmatopt^\inv)((1-t) \eye + t \Bmatopt)\Halt
&= (1-t)^2 \Dalt \Halt + (1-t)t \Dalt \Bmatopt\Halt + (1-t)t \Dalt \Qmatopt^\inv \Halt + t^2 \Dalt \Qmatopt^\inv \Bmatopt\Halt\\
&= (1-2t + t^2) \Dalt \Halt + (1-t)t \Dalt \Bmatopt\Halt + (1-t)t \Dalt \Qmatopt^\inv \Halt + t^2 \Dalt \Halt\\
&= (1-2t + 2t^2) \Dalt \Halt + (1-t)t (\Dalt \Bmatopt\Halt + \Dalt \Qmatopt^\inv \Halt )\\
&= \Dalt \Halt + t (1-t) (\Dalt \Bmatopt\Halt - \Dalt \Halt + \Dalt \Qmatopt^\inv \Halt -\Dalt \Halt )
\end{align*}
and so the additional term that could increase the loss is 
\begin{align*}
s(t) = t (1-t) (\Dalt \Bmatopt\Halt - \Dalt \Halt + \Dalt \Qmatopt^\inv \Halt -\Dalt \Halt )
.
\end{align*}
%
We show in Lemma \ref{lem_strongly} in Appendix \ref{app_strongly} that $\regd$ and $\regh$ are strongly convex around $\Dalt, \Halt$ because $\Dalt, \Halt$ are full rank, i.e., there exist $m_d, m_h > 0$ such that
\begin{align}
\regd((1-t)\Dalt + t\Dmat) &\le
(1-t)\regd(\Dalt) + t\regd(\Dmat) - \frac{m_d}{2} t (1-t) \| \Dalt - \Dmat\|_F^2 \label{eq_sc_d}\\
\regh((1-t)\Halt + t\Hmat) &\le
(1-t)\regh(\Halt) + t\regh(\Hmat) - \frac{m_h}{2} t (1-t) \| \Halt - \Hmat\|_F^2  \label{eq_sc_h}
\end{align}
for any $\Dmat, \Hmat$ for a sufficiently small $t$.
Now because $\jointloss$ is Holder continuous, we get
\begin{align*}
&\jointloss(\Dalt \Halt + s(t)) - \jointloss(\Dalt \Halt) \\
&\le m \| \Dalt \Halt + s(t) - \Dalt \Halt \|_F^2\\
&= m \| s(t) \|_F^2\\
&= m t^2 (1-t)^2 \| \Dalt \Bmatopt\Halt - \Dalt \Halt + \Dalt \Qmatopt^\inv \Halt -\Dalt \Halt \|_F^2\\
&\le m t^2 (1-t)^2 \left(\| \Dalt \Bmatopt\Halt - \Dalt \Halt \|_F^2 + \|\Dalt \Qmatopt^\inv \Halt -\Dalt \Halt \|_F^2 \right)\\
&\le m t^2 (1-t)^2 \| \Dalt \|_F^2 \|\Bmatopt\Halt - \Halt \|_F^2 + m t^2 (1-t)^2 \|\Dalt \Qmatopt^\inv  - \Dalt\|_F^2 \|\Halt \|_F^2 \\
&\le c_d(t) \left[(1-t)\regd(\Dalt) + t\regd(\Dalt \Qmatopt^\inv) - \regd((1-t)\Dalt + t\Dalt \Qmatopt^\inv) \right] &&\triangleright \text{ by \eqref{eq_sc_d} and \eqref{eq_sc_h}, with}\\
& \ \ + c_h(t) \left[(1-t)\regh(\Halt) + t\regh(\Bmatopt\Halt) - \regh((1-t)\Halt + t \Bmatopt\Halt) \right] && \ \ \ c_d(t) = \frac{2mt(1-t)}{m_d}\|\Halt \|_F^2\\
&\ && \ \ \ c_h(t) = \frac{2mt(1-t)}{m_h}\|\Dalt \|_F^2\\
&< c_d(t) \left[\regd(\Dalt) - \regd((1-t)\Dalt + t\Dalt \Qmatopt^\inv) \right]  &&\triangleright \regd(\Dalt) > \regd(\Dalt \Qmatopt^\inv)\\
& \ \ + c_h(t) \left[\regh(\Halt) - \regh((1-t)\Halt + t\Bmatopt\Halt) \right]  &&\triangleright \regh(\Halt) > \regh(\Bmatopt\Halt)
\end{align*}
For a small enough $\Delta$, $c_d(t) < 1$ and $c_h(t) < 1$ for all $t < \Delta$. Therefore, for a sufficiently small $t < \Delta$
\begin{align*}
\jointloss(\Dalt \Halt + s(t)) - \jointloss(\Dalt \Halt) &< \regd(\Dalt) - \regd((1-t)\Dalt + t\Dalt \Qmatopt^\inv) \\
&+ \regh(\Halt) - \regh((1-t)\Halt + t\Bmatopt\Halt))
.
\end{align*}
We can conclude that the potential increase in the loss term from stepping in this direction is less than the reduction in 
the regularization terms. Therefore, there exists a direction from $\Dalt, \Halt$ that strictly decreases the loss, which contradicts
the fact that $\Dalt, \Halt$ is a local minimum. 
Consequently, $\Qmat = \eye$ and $\Bmatopt_1 = \zerovec$ is one solution to \eqref{eq_frobform}.  

Finally, it is possible that there could be other $\Qmatopt, \Bmatopt_1$ that do not decrease the regularizer. Because the regularizer $\regh$ is strictly convex and $\Halt$ is full rank, the solutions $\Qmatopt$ and $\Bmatopt_1$ are unique up to invariances in the regularizers (e.g., the Frobenius norm is invariant under orthonormal matrices). This further implies that $\Bmatopt_1 = \zerovec$, as the regularizers are centered functions. Consider now such a solution $\Qmat, \Bmat_1 = \zerovec$, and assume that
there is a descent direction $d\Dmat, d\Hmat$ from $\Dalt \Qmat^\inv, \Qmat \Halt$. 
Notice that 
\begin{align*}
(\Dalt \Qmat^\inv + d\Dmat)(\Qmat \Halt + d\Hmat) 
&= 
(\Dalt+ d\Dmat \Qmat) \Qmat^\inv \Qmat (\Halt + \Qmat^\inv d\Hmat)\\
&=
(\Dalt+ d\Dmat \Qmat)(\Halt + \Qmat^\inv d\Hmat)
\end{align*}
giving $d\Dmat \Qmat,  \Qmat^\inv d\Hmat$ as a possible descent direction from $\Dalt, \Halt$.
Because the solution is unique up to invariances in the regularizer, 
$\regd(\Dalt \Qmat^\inv + d\Dmat) = \regd((\Dalt + d\Dmat \Qmat )\Qmat^\inv) =  \regd(\Dalt + d\Dmat \Qmat)$
and  
$\regh(\Qmat \Halt + d\Hmat) = \regh(\Qmat (\Halt + \Qmat^\inv d\Hmat)) = \regh(\Halt + \Qmat^\inv d\Hmat)$.
If $d\Dmat, d\Hmat$ is a descent direction for $\Dalt \Qmat^\inv, \Qmat \Halt$, 
then $d\Dmat \Qmat,  \Qmat^\inv d\Hmat$ must also be a descent direction for $\Dalt, \Halt$.
This would be a contradiction, since $\Dalt, \Halt$ is a local minimum. Therefore, any solution $\Dalt \Qmatopt^\inv, \Qmatopt \Halt$
must also be a stationary point. 
 
 Therefore, $\Dalt, \Halt$ satisfies the induced regularization property, because $\regd(\Dalt) + \regh(\Halt)$ cannot
 be improved by selecting different $\Dmat_k, \Hmat_k$ that satisfy $\Dmat_k \Hmat_k = \Dalt \Halt$
 and all possible solutions are stationary points. 
\end{proof}

We can use this more general result to show that local minima are global minima for smoothed versions of 
the elastic net and the $\ell_1$ for sparse coding. The theory requires differentiable regularizers, so we can only currently
characterize smoothed approximations to these objectives. Because these smooth approximations converge to the non-smooth regularizer, according to a parameter $\mu \rightarrow 0$, 
this suggests that the theoretical result should extend to the original non-smooth formulation. The smooth elastic-net objective
has $\xdim \le \ldim \le \nsamples$ and pseudo-Huber loss within the elastic net regularizer on $\Hmat$: 
\begin{equation}
\frow^2(\hvec) = \nuh \| \hvec \|_2^2 + (1-\nuh) \left(\sum_{\sampiter=1}^\nsamples \sqrt{\mu^2 + \hvec_\sampiter^2} - \mu \right)^2
. 
\label{eq_smoothelastic}
\end{equation}
As $\mu \rightarrow 0$, this smooth regularizer converges to the elastic-net regularizer: $\nuh \| \hvec \|_2 + (1-\nuh) \| \hvec \|_1^2$. Additionally, for $\nu = 0$, the below corollary characterizes a smooth approximation to the sparse coding objective. Any strictly convex regularizer $\regd$ can be used for $\Dmat$, where a common choice is $\fcol(\cdot) = \| \cdot \|_2$ to give $\regd(\Dmat) = \frobsq{\Dmat}$. 
\begin{corollary}[Elastic-net regularizer or Sparse regularizer on $\Hmat$]\label{cor_en}
Consider the (smooth) elastic-net objective, with $\xdim \le \ldim \le \nsamples$, convex Holder continuous loss, any centered strongly convex $\fcol$ and $\regh(\Hmat) = \regwgt \sum_{i=1}^\ldim \frow^2(\Hmat_{i:})$ with $\frow$ as defined in \eqref{eq_smoothelastic}.
Then for sufficiently large $\ldim$, such that $\regzk$ is convex, if 
 \begin{enumerate}[itemindent=0.5cm]
 \item $(\Dalt,\Halt)$ is a local minimum and
 \item $\Dalt$ and $\Halt$ are full rank
 \end{enumerate}
then $(\Dalt, \Halt)$ is a global minimum. 
\end{corollary}


\section{Avoiding saddlepoints}\label{sec_saddle}


To complete the characterization of using alternating minimization for induced \RFMs,
it is important to understand issues with saddlepoints. 
For non-convex optimizations, saddlepoints are problematic,
often corresponding to large flat regions and stalling convergence. 
For the biconvex optimization for induced \RFMs, there are clear symmetries that cause multiple
equivalent solutions and saddlepoints between solutions, such as in Figure \ref{fig:optimization}.
Our empirical investigation in the next section indicates that alternating minimization for induced \RFMs\
does not get stuck in saddle points. Nonetheless, we would like
a stronger guarantee of convergence to a local minimum,
which then corresponds to a global minimum.

To do so, we can take advantage
of recent characterization of non-convex problems using the strict-saddle property \cite{ge2015escaping,sun2015when,sun2015complete,lee2016gradient},
with the most general definition given by \citet{lee2016gradient}.
The requirement is simply that either the stationary point is a local minimum or the Hessian
has at least one strictly negative eigenvalue. 
Recall that Hessians at saddle points can have both positive and negative eigenvalues, and degenerate
saddlepoints are those that have positive semi-definite Hessians with zero eigenvalues. 
Therefore, another way to state the strict-saddle property is that there are no degenerate saddle points.
 \citet{lee2016gradient} recently proved that for twice continuously differentiable functions that satisfy the strict saddle property, 
 gradient descent with a random initialization and a sufficiently small constant step-size converges to a local minimizer.
 
 We conjecture that this is the case for induced \RFMs, and believe that this is the main reason that empirically induced \RFMs\ converge stably
to local minima, and so to global minima, even without additional noise or a particular initialization.
Towards theoretical motivation for this conjecture, 
we show that a subclass of induced \RFMs\ do not have degenerate saddlepoints, in Theorem \ref{thm_weighted}.
Further, as we showed for the induced optimization in \eqref{eq_inducedk}, the result is even stronger:
 all stationary points that are full rank are guaranteed to be global minima. 
 Though these two results provide some insight into saddlepoints properties,
 an important next step is to more fully characterize the saddlepoints for induced \RFMs.

\begin{theorem} \label{thm_weighted}
For invertible $\Diagmat \in \RR^{\xdim \times \xdim}$, $\regwgt \ge 0$ and $s > 0$, 
let $(\Dalt,\Halt)$ be a stationary point of
%
\begin{align}
\min_{\genrankset} \jointloss(\Dmat\Hmat) + \tfrac{\regwgt}{2} \frobsq{\Diagmat\Dmat} + \tfrac{\regwgt}{2 \hscale^2}  \frobsq{\Hmat} 
\label{eq_localweighted}
.
\end{align}
%
If the Hessian is positive semi-definite and $(\Dalt,\Halt)$ is rank-deficient 
then $(\Dalt, \Halt)$ is a global optimum.
\end{theorem}
We provide a slightly more general theorem---Theorem \ref{thm_weighted_general}---and the proof in Appendix \ref{app_weighted}.

\begin{figure*}[tp]
  \centering
   \subfigure[Optimization surface]{\includegraphics[width=0.45\textwidth]{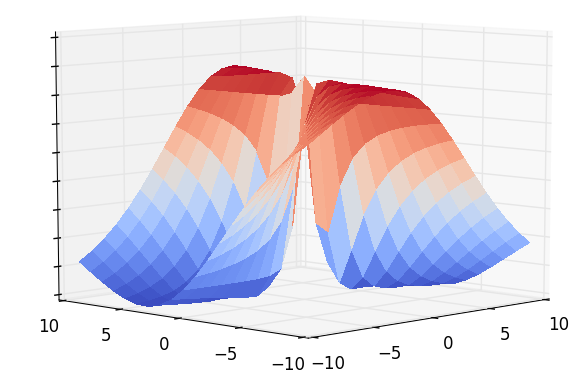}\label{fig:optimization}}
   \subfigure[Saddlepoint]{\includegraphics[width=0.45\textwidth]{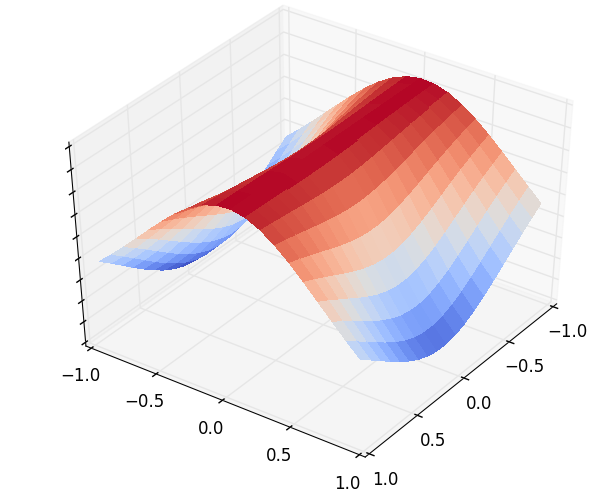}\label{fig:saddlepoint}}
 \caption{ Optimization surface for $\text{loss}(\Dmat) = \min_{\Hmat} \sum_{t=1}^\nsamples \| \Xmat_{:t} - \Dmat \Hmat_{:t} \|_2^2 + \regwgt \| \Hmat \|_F^2 + \regwgt \| \Dmat \|_F^2$ where $\Dmat \in \RR^{2 \times 1}$, $\nsamples = 20$ and $\regwgt = 0.5$. The first figure illustrates the symmetry of the problem, where there
 are two equivalent optimal $\Dmat$ in terms of the objective. The second figure is zoomed into the top of the optimization surface, indicating
 the saddle-point between these two optimal points.}
  \label{fig:optimization}
\end{figure*}

\section{Empirical evidence for optimality of alternating minimization for induced \RFMs}\label{sec_empirical}

We now empirically investigate the central hypothesis to this work, to complement the theoretical insights in the previous two sections. 
We show that for several induced \RFMs---including those not directly covered by the theoretical results---that alternating minimization seems to find global solutions. Further, we show that a couple of deviations oustide of this class do not appear to maintain this property.
This suggests that the class of induced \RFMs\ identifies a subclass of \RFMs\ that have well-behaved optimization surfaces. 
We first provide these empirical results and then conclude the section with details about the alternating minimization variants and implementation.
 
 \subsection{Optimality of alternating minimization for a variety of \RFMs}
 
To test for global optimality, the alternating minimization is started from different initial values
and differences in the solution reported. We report both differences in the objective value (Table \ref{table_batch})
and in the matrices themselves (Table \ref{table_batch_supplement}). 
The normed differences between the solutions $\joint = \Dmat \Hmat$ are reported to clarify
that similarities in objectives are due to similar solutions, rather than because the objective
itself does not change much. 
The initializations are random but of highly differing magnitudes 
to better search for different local minima and saddlepoints in which the optimization could get stuck. 
Small relative differences between objective values suggest that the
stationary points are in fact the same global minimum.

As baselines and to further elucidate if this global optimality property is characteristic of induced \RFMs,
we also test two 
modifications that take the objective
outside the class of induced \RFMs. 
The first modification is to use $\sum_{i=1}^\xdim \| \Dmat_{i:} \|_2$ as the regularizer on $\Dmat$,
which couples the columns of $\Dmat$ and so is not a valid regularizer for induced \RFMs. 
The second modification is to use the non-norm elastic net regularizer $\fcol^2(\dvec) = \nu \| \dvec \|_2^2 + (1-\nu) \| \Lambdamat \dvec \|_1$ \citep{zou2005regularization}, 
which does not satisfy the requirements of Proposition \ref{prop_generalization} for generalized induced \RFMs. 
The non-norm elastic net almost satisfies the requirements of Proposition \ref{prop_generalization}, because
$\fcol(\dvec) = \sqrt{\nu \| \dvec \|_2^2 + (1-\nu) \| \Lambdamat \dvec \|_1}$
is centered and non-negative; however, it is non-convex.
For $\Lambdamat = \eye$, however, $\fcol(\dvec)$ is almost a convex function,
with only a slight concave bow. We find, in fact, that with $\Lambdamat = \eye$, alternating minimization does seem to provide global solutions. 
However, with different choices of $\Lambdamat$, the $\fcol$ becomes more non-convex and alternating minimization is no longer
global. In the following experiments, we choose $\Lambdamat = \diag(1, 2, \ldots, \xdim)$. 


The results for the differences in objective values and in solution matrices are summarized in Table \ref{table_batch} and Table \ref{table_batch_supplement} respectively.
The induced \RFMs\ are in columns 1, 2 and 3 and the modified settings, which no longer correspond to induced \RFMs, are in columns 4, 5 and 6. 
The results are reported across many settings, including $d \in \{5, 10, 50\}$, $k \in \{3, 5, 10\}$ and 
$\regwgt \in \{0.005, 0.05, 0.5\}$, 
with a fixed sample size of $\nsamples = 100$
and least-squares $\ell_2$ loss.
The initial entries in the factors were randomly selected from unit-variance Gaussian distributions with increasing mean values $\mu \in \{0, 5, 10, 15, \ldots, 45\}$. 
For objective values the induced \RFMs\ have relative differences within 0.1\%, suggesting they are equivalent global optima, 
whereas the modified objectives have significantly larger relative differences, between 10\% to 150\%, demonstrating clearly different local minima. 
The relative differences between two solutions $\joint_1 = \Dmat_1 \Hmat_1$ and $\joint_2 = \Dmat_2 \Hmat_2$ are similarly small, 
though unsurprisingly larger than the objective values. Particularly for the elastic net, the larger the variable, the more opportunity
to select which entries in $\joint$ will be zeroed with similar objective values. The trend, however, remains consistent, with smaller relative differences
for the induced \RFMs. 

   \newcommand{\gtpbox}{30cm}
\setlength{\tabcolsep}{4pt}
\begin{table*}[h!]
\begin{center}
\begin{small}
\begin{sc}
\begin{tabular}{|l||c|c|c|c|c|c|c|}
\hline   & Subspace & Sparse & Elastic Net &  Elastic net &  Coupled cols & Coupled cols \\
 & \scriptsize $\fcol^2 = \| \cdot \|_2^2$ & \scriptsize $\fcol^2 = \| \cdot \|_1^2$  & \scriptsize $\nu \| \cdot \|_2^2 + (1-\nu) \|\cdot \|_1^2$ & (non-norm) & \scriptsize $\sum_{i=1}^\xdim \| \Dmat_{i:} \|_1^2$ & \scriptsize $\sum_{i=1}^\xdim \| \Dmat_{i:} \|_2$\\
 & \scriptsize $\frow^2 = \| \cdot \|_2^2$ &\scriptsize $\frow^2 = \| \cdot \|^2_1$  & \scriptsize $\nud = 0.5 = \nuh$ & \scriptsize $\nud = 0.5 = \nuh$ & \scriptsize $\frow^2 = \| \cdot \|_2^2$& \scriptsize $\frow^2 = \| \cdot \|_2^2$  \\
\hline
\pbox{\gtpbox}{\pboxspace\shortstack{Min\\ Difference} \pboxspace} & \pbox{\gtpbox}{\pboxspace\pbox{\gtpbox}{\pboxspace\shortstack{0.000000 \\ \scriptsize$(0.5,  5,  10)$}\pboxspace}\pboxspace} & \pbox{\gtpbox}{\pboxspace\shortstack{0.000000 \\ \scriptsize$(0.05,  5,  3)$}\pboxspace} & \pbox{\gtpbox}{\pboxspace\shortstack{0.000000 \\ \scriptsize$(0.5,  5,  3)$}\pboxspace} & \pbox{\gtpbox}{\pboxspace\shortstack{0.000007 \\ \scriptsize$(0.05,  50,  3)$}\pboxspace} & \pbox{\gtpbox}{\pboxspace\shortstack{0.004586 \\ \scriptsize$(0.5,  10,  3)$}\pboxspace} & \pbox{\gtpbox}{\pboxspace\shortstack{0.466634 \\ \scriptsize$(0.005,  5,  3)$}\pboxspace}\\
\hline
\pbox{\gtpbox}{\pboxspace\shortstack{Max\\ Difference} \pboxspace}&  \pbox{\gtpbox}{\pboxspace\shortstack{0.000785 \\ \scriptsize$(0.005,  5,  5)$}\pboxspace} & \pbox{\gtpbox}{\pboxspace\shortstack{0.000136 \\ \scriptsize$(0.5, 50, 10)$}\pboxspace} & \pbox{\gtpbox}{\pboxspace\shortstack{0.001269 \\ \scriptsize$(0.05, 10, 10)$}\pboxspace} & \pbox{\gtpbox}{\pboxspace\shortstack{1.535013 \\ \scriptsize$(0.005, 50, 3)$}\pboxspace} & \pbox{\gtpbox}{\pboxspace\shortstack{0.096123 \\ \scriptsize$(0.005,  5,  3)$}\pboxspace} & \pbox{\gtpbox}{\pboxspace\shortstack{0.566597 \\ \scriptsize$(0.05,  10,  3)$}\pboxspace}\\
\hline
\end{tabular}
\end{sc}
\end{small}
\caption{The minimum and maximum relative differences between objective values for the solutions found by alternating minimization,
with a least-squares loss and the given regularizers. 
The relative difference between a pair of solutions is computed by taking the absolute difference between their objective values divided by
the expected value of the objective across the solutions (to normalize the magnitude for changing $\xdim, \ldim$). 
The parameter settings that resulted in the minimum and maximum relative difference
are reported in brackets below the values with the order $(\regwgt, \xdim, \ldim)$. 
} \label{table_batch} 
\end{center}
\end{table*}

\setlength{\tabcolsep}{4pt}
\begin{table*}[h!]
\begin{center}
\begin{small}
\begin{sc}
\begin{tabular}{|l||c|c|c|c|c|c|c|}
\hline   & Subspace & Sparse & Elastic Net &  Elastic net &  Coupled cols & Coupled cols \\
 & \scriptsize $\fcol^2 = \| \cdot \|_2^2$ & \scriptsize $\fcol^2 = \| \cdot \|_1^2$  & \scriptsize $\nu \| \cdot \|_2^2 + (1-\nu) \|\cdot \|_1^2$ & (non-norm) & \scriptsize $\sum_{i=1}^\xdim \| \Dmat_{i:} \|_1^2$ & \scriptsize $\sum_{i=1}^\xdim \| \Dmat_{i:} \|_2$\\
 & \scriptsize $\frow^2 = \| \cdot \|_2^2$ &\scriptsize $\frow^2 = \| \cdot \|^2_1$  & \scriptsize $\nud = 0.5 = \nuh$ & \scriptsize $\nud = 0.5 = \nuh$ & \scriptsize $\frow^2 = \| \cdot \|_2^2$& \scriptsize $\frow^2 = \| \cdot \|_2^2$  \\
\hline
\pbox{\gtpbox}{\pboxspace\shortstack{Min \\ Difference} \pboxspace} & \pbox{\gtpbox}{\pboxspace\pbox{\gtpbox}{\pboxspace\shortstack{0.000000 \\ \scriptsize$(0.005,  5,  3)$}\pboxspace}\pboxspace} & \pbox{\gtpbox}{\pboxspace\shortstack{0.000000 \\ \scriptsize$(0.05, 5, 3)$}\pboxspace} & \pbox{\gtpbox}{\pboxspace\shortstack{0.000000 \\ \scriptsize$(0.5,  5,  3)$}\pboxspace} & \pbox{\gtpbox}{\pboxspace\shortstack{0.000000 \\ \scriptsize$(0.5, 5, 3)$}\pboxspace} & \pbox{\gtpbox}{\pboxspace\shortstack{0.007800 \\ \scriptsize$(0.005, 50, 3)$}\pboxspace} & \pbox{\gtpbox}{\pboxspace\shortstack{0.004000 \\ \scriptsize$(0.05, 5, 10)$}\pboxspace}\\
\hline
\pbox{\gtpbox}{\pboxspace\shortstack{Max\\ Difference} \pboxspace}  &  \pbox{\gtpbox}{\pboxspace\shortstack{0.005000\\ \scriptsize$(0.05, 10, 10)$}\pboxspace} & \pbox{\gtpbox}{\pboxspace\shortstack{0.023400 \\ \scriptsize$(0.05, 50, 3)$}\pboxspace}  & \pbox{\gtpbox}{\pboxspace\shortstack{0.124800 \\ \scriptsize$(0.5, 50, 10)$}\pboxspace} & \pbox{\gtpbox}{\pboxspace\shortstack{0.994400 \\ \scriptsize$(0.5, 50, 3)$}\pboxspace} & \pbox{\gtpbox}{\pboxspace\shortstack{0.631000 \\ \scriptsize$(0.5, 10, 10)$}\pboxspace} & \pbox{\gtpbox}{\pboxspace\shortstack{0.592000 \\ \scriptsize$(0.005, 50, 3)$}\pboxspace}\\
\hline
\end{tabular}
\end{sc}
\end{small}
\caption{The minimum and maximum relative norm differences between solutions found by alternating minimization,
with a least-squares loss and the given regularizers. 
The relative difference is a thresholded relative difference, illustrating the percentage of significantly different values.
The relative difference between two solutions is the number of entries above the threshold 0.05, in terms of
the absolute value of the difference between the two solutions $\joint_1 = \Dmat_1 \Hmat_1$
and $\joint_2 = \Dmat_2 \Hmat_2$, divided by the total number of elements. 
The parameter settings that resulted in the minimum and maximum are reported in brackets below the values with the order $(\regwgt, \xdim, \ldim)$.      
} \label{table_batch_supplement} 
\end{center}
\end{table*}

\subsection{Selecting the inner dimension}

The optimality of induced \RFMs\ has been predicated on selecting a sufficiently large inner dimension $\ldim$. 
Because $\ldim$ can be critical, it is important to understand how to set this parameter. 
For certain models, the selection of $\ldim$ is intuitive. 
For example, for subspace dictionary learning, $\ldim$ is the size of the desired latent rank, where $\ldim$ can be chosen smaller
for larger values of $\regwgt$. 
For sparse coding, the increase in the $\regwgt$ does not naturally lead to a decrease in $\ldim$,
but rather to an increase in the level of sparsity. 
For the elastic net regularizer, however, we have less intuition; 
there is likely a preference for $\xdim < \ldim < \nsamples$, but this remains unknown.

There are several strategies that can be pursued to facilitate choice of $\ldim$.
A simple approach is to allow the optimization to select $\ldim$ by setting $\ldim = \nsamples$. 
The optimal solution will have an implicit $\ldim$ that is smaller than $\nsamples$. This approach, however,
is typically not computationally feasible. To avoid setting $\ldim$ too large, a number of algorithms have been developed
that iteratively generate columns and rows to add to $\Dmat$ and $\Hmat$ respectively 
\cite{aspremont2004direct,bach2008convex,journee2010low,zhang2012accelerated,hsieh2014nuclear,mirzazadeh2015scalable}.
In practice, however, it is common to simply choose a fixed $\ldim$;
such a strategy, however, compromises some of the simplicity and efficiency of a standard alternating minimization algorithm,
initialized from a random solution. Despite the introduction of such algorithms, it remains common in practice to simply choose a fixed $\ldim$ and using alternating minimization. Further, it is less clear how to extend such algorithms, that iteratively increase $\ldim$, to the incremental setting. 

To facilitate the use of the simple alternating minimization approach used in practice, for a fixed user-specified $\ldim$,
we show the optimality properties of subspace, sparse and elastic-net \RFMs, for increasing $\ldim$. The goal of these exploratory results is to better elucidate practical choices of $\ldim$ for these induced \RFMs. The results are reported on the Extended Yale Face Database B \cite{georghiades2001from}. We provide two experiments. The first parallels the experiments in the previous section, but now we test optimality for a variety of $\ldim$, including small $\ldim$. The second experiment is to determine how much the objective can be improved by increasing $\ldim$. 
The goal of the first experiment is to examine for which user-specified $\ldim$, alternating minimization can find global solutions. 
The goal of the second experiment is to examine the inherent preference in the objective for a larger $\ldim$.

For the first comparison, we determine the optimality of alternating minimization for smaller $\ldim$, shown in Figure \ref{fig:heatmapSD}. This contrasts the results in the previous section, where we fixed $\ldim$ to one, larger value. 
The theoretical results indicate that even if $\ldim$ is less than the dimension $\ldim^*$ that gives $\regzk = \regz$,
alternating minimization can still give global solutions for $\ldim$ that results in convex $\regzk$. 
We find that for even very small $\ldim$, alternating minimization produces global solutions. 
  This is a surprising result, considering for sparse coding ($\nuh = 0$), the conventional wisdom is that $\ldim$ needs to be quite large
  (close to $\nsamples$) to get optimal solutions. Here, however, we are showing that alternating minimization, with a restricted $\ldim$, can still obtain the global solution \textit{for that objective}. When comparing to an objective where $\ldim$ is allowed to be larger, there is a bigger difference (see Figure \ref{fig:heatmap}). Therefore, if a specific $\ldim$ is desired for sparse coding, alternating minimization is likely going to reach the global minimum, despite the fact that further increasing $\ldim$ could further decrease the objective. 
This behavior suggests that either $\regzk$ is convex for these $\ldim$,
or potentially that $\regzk$ has other nice properties not currently characterized by our theory. 
We discuss this outcome further in the conclusion.

\begin{figure*}[h]
  \centering
   \subfigure[$\nuh$ = 0]{\includegraphics[width=0.32\textwidth]{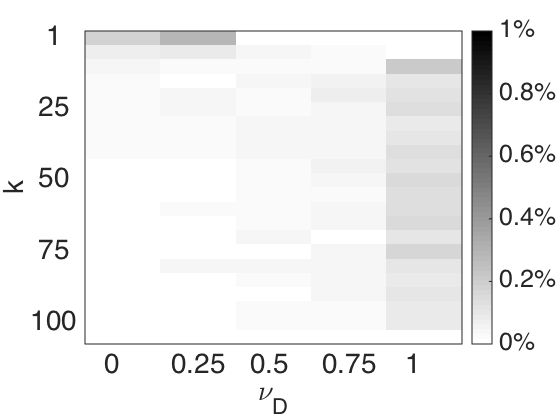}\label{fig:heatmapLSSd00a}}
   \subfigure[$\nuh$ = 0.5]{\includegraphics[width=0.32\textwidth]{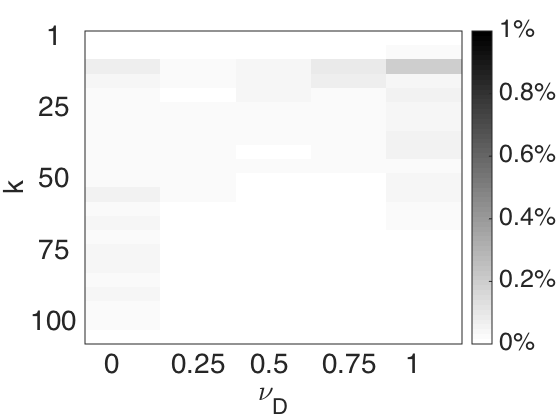}\label{fig:heatmapLSSd05a}}
   \subfigure[$\nuh$ = 1.0]{\includegraphics[width=0.32\textwidth]{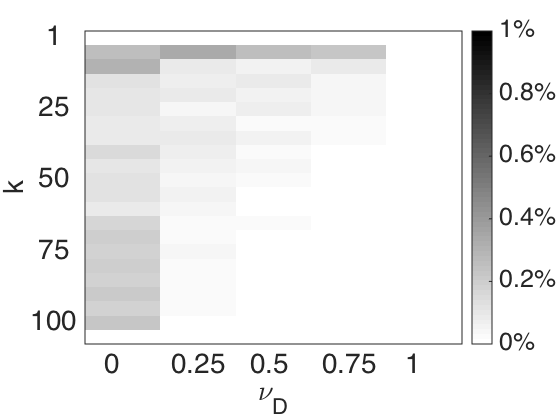}\label{fig:heatmapLSSd10a}}  
    \subfigure[$\nuh$ = 0]{\includegraphics[width=0.32\textwidth]{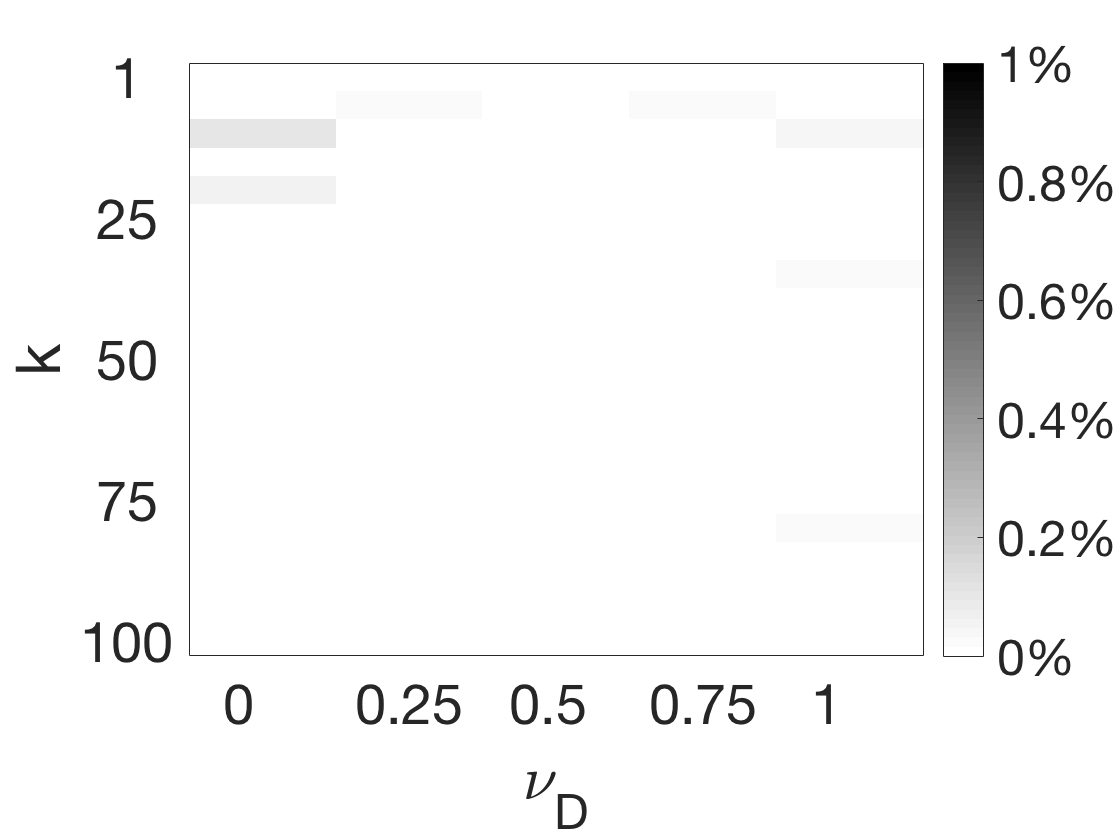}\label{fig:heatmapGasLSSd00a}}
   \subfigure[$\nuh$ = 0.5]{\includegraphics[width=0.32\textwidth]{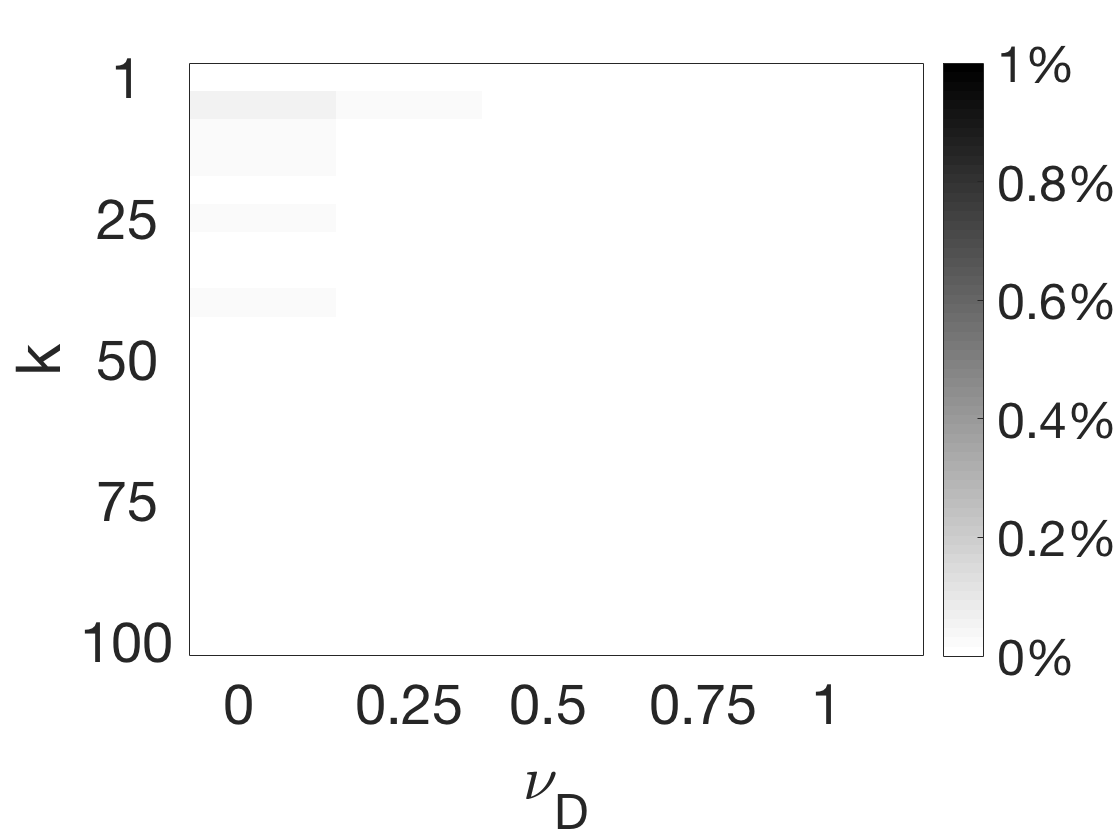}\label{fig:heatmapGasLSSd05a}}
   \subfigure[$\nuh$ = 1.0]{\includegraphics[width=0.32\textwidth]{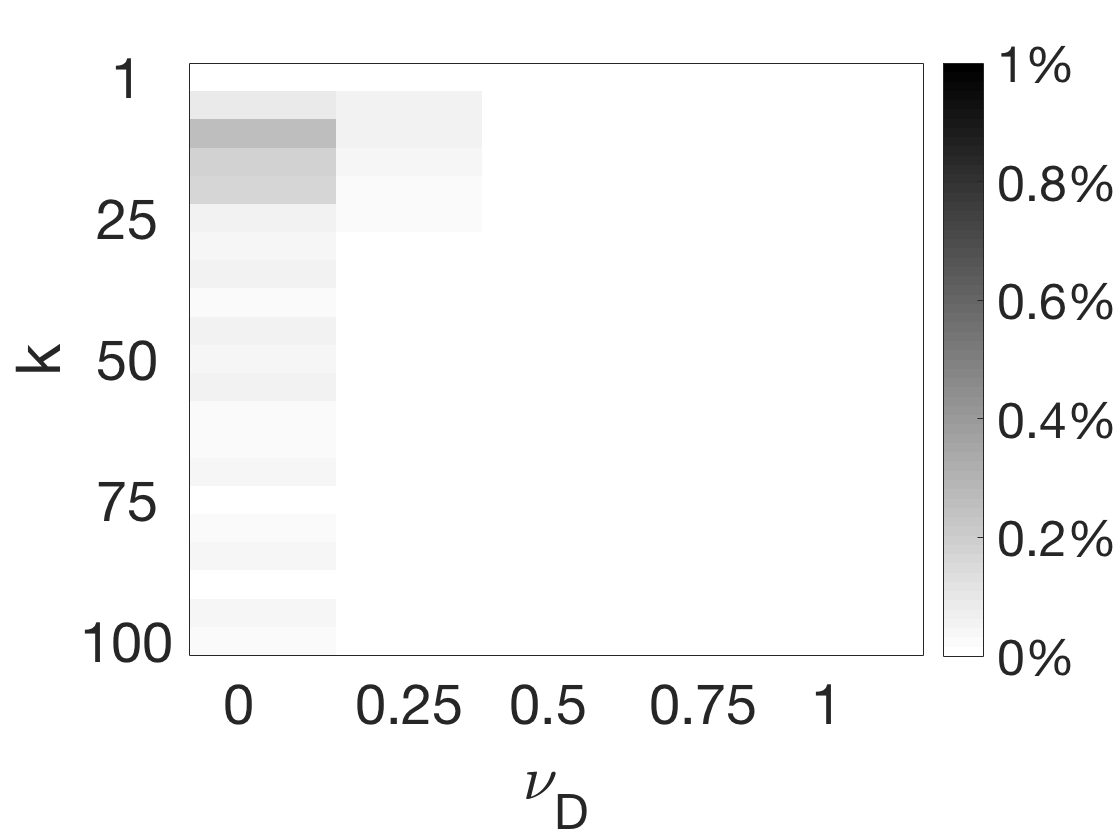}\label{fig:heatmapGasLSSd10a}}  
  \caption[caption]{ Standard deviations of the objective value for $\ldim \leq \nsamples$ for elastic net \RFMs,
  over 10 runs of different random initializations with the same regularization parameters.
  The lighter the color, the smaller the difference. 
 The loss is the least-squares loss for (a) - (c), and the logistic loss for (d) - (f), with $\nsamples = 100$ and $\xdim = 50$ on data randomly drawn from the Extended Yale Face Database B.   
 The relative differences are averaged over 10 runs, with the maximal distance reported within each run from 10 different random initializations.
  The goal is to identify if the solutions returned
  for highly different initial points result in different local minima. The standard deviation is very small for all choices of $\ldim$,
  suggesting that the AM-\RFM\  algorithm (see Algorithm \ref{alg_amrfm}) obtains a globally optimum solution even for small $\ldim$.
 }
  \label{fig:heatmapSD}
\end{figure*}

For the second experiment, we compare the objective values for increasing sizes of $\ldim$, compared to the least constrained setting of $\ldim = \nsamples$. The results are summarized in Figure \ref{fig:heatmap}. We expect that if the objective has a strong preference for larger $\ldim$,
then there will be a large discrepancy between the objective value for small $\ldim$ and the objective value for $\ldim = \nsamples$. 
 The result indicates that $\ldim$ can be much smaller than $\nsamples$ and obtain good solutions, within $1\%$. 
 Recall that $\nud = 0$ results in a sparse $\ell_1$ regularizer on $\Dmat$, and $\nud = 1$ results in the subspace $\ell_2$ regularizer on $\Dmat$. 
 The cases where $\ldim$ needs to be larger to match performance of $\ldim = \nsamples$ are when $\nud = 0, \nuh = 1$ or $\nud = 1, \nuh = 0$.
 This is consistent with previous results, where the global solution for sparse coding requires $\ldim = \nsamples$. We can see that even with a small decrease from this extreme
 setting, with $\nud < 1, \nuh = 0$, $\ldim$ can be significantly smaller without incurring much difference to the optimal solution at $\ldim = \nsamples$. 
Therefore, for sparse coding ($\nuh = 0$ and $\nud = 1$), the choice of $\ldim$ does need to be noticeably
larger than for the elastic net (i.e., all other settings). This validates the initial motivation for using the elastic-net for dictionary learning,
for learning more compact sparse representations.

\begin{figure*}[]
  \centering
   \subfigure[$\nuh$ = 0]{\includegraphics[width=0.32\textwidth]{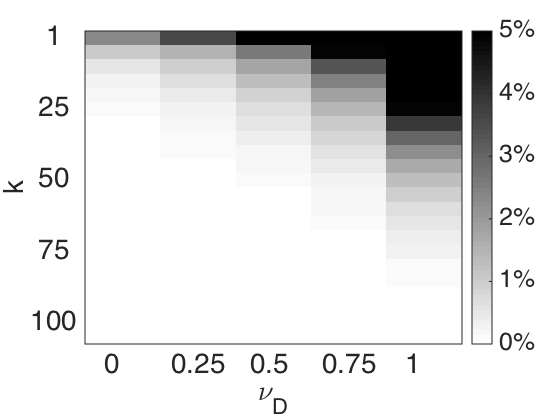}\label{fig:heatmapLS00b}}
   \subfigure[$\nuh$ = 0.5]{\includegraphics[width=0.32\textwidth]{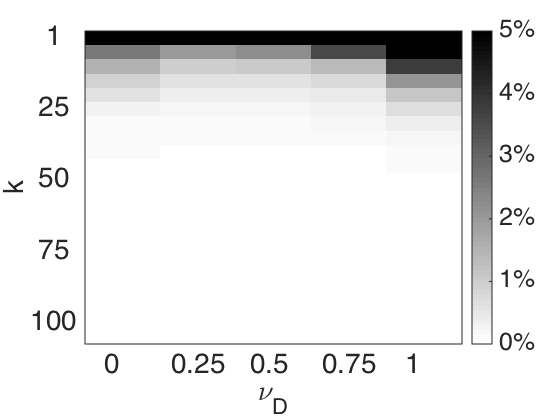}\label{fig:heatmapLS05b}}
   \subfigure[$\nuh$ = 1.0]{\includegraphics[width=0.32\textwidth]{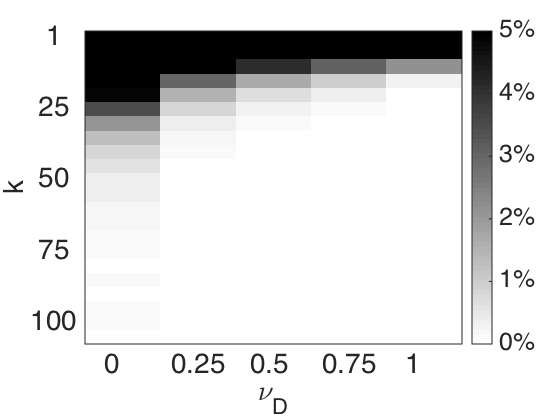}\label{fig:heatmapLS10b}}  
   \subfigure[$\nuh$ = 0]{\includegraphics[width=0.32\textwidth]{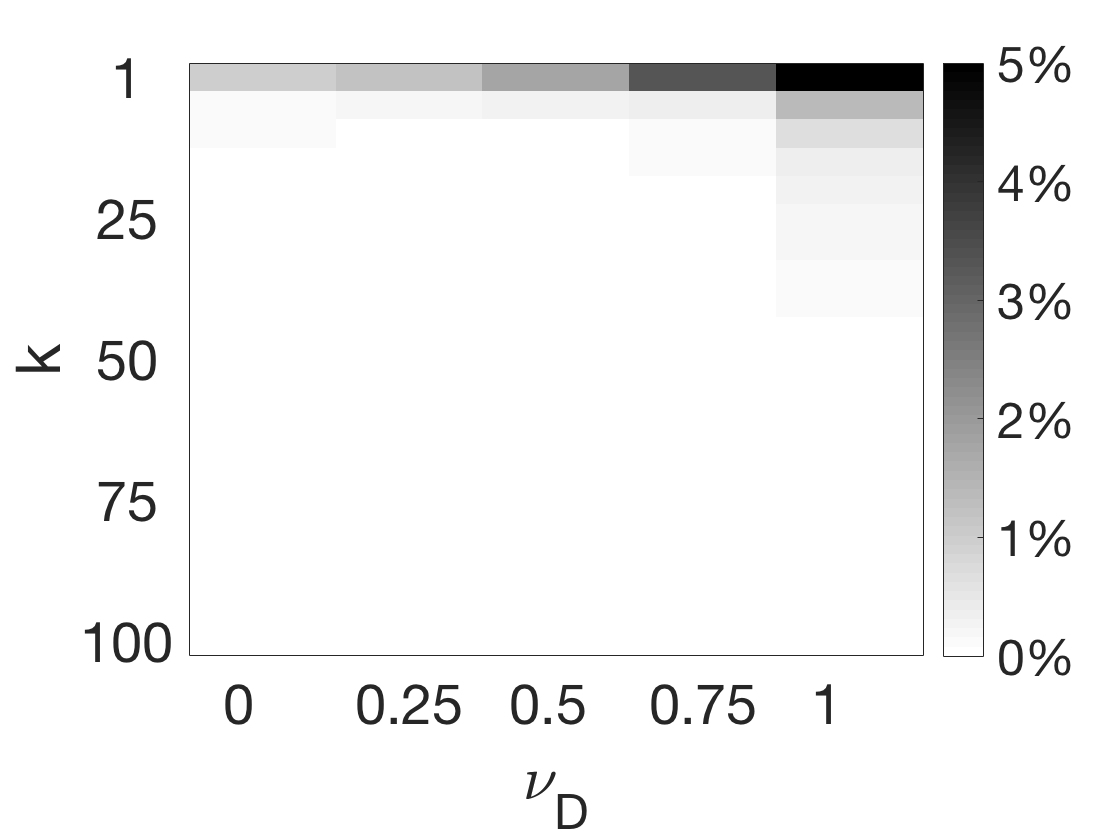}\label{fig:heatmapGasLS00b}}
   \subfigure[$\nuh$ = 0.5]{\includegraphics[width=0.32\textwidth]{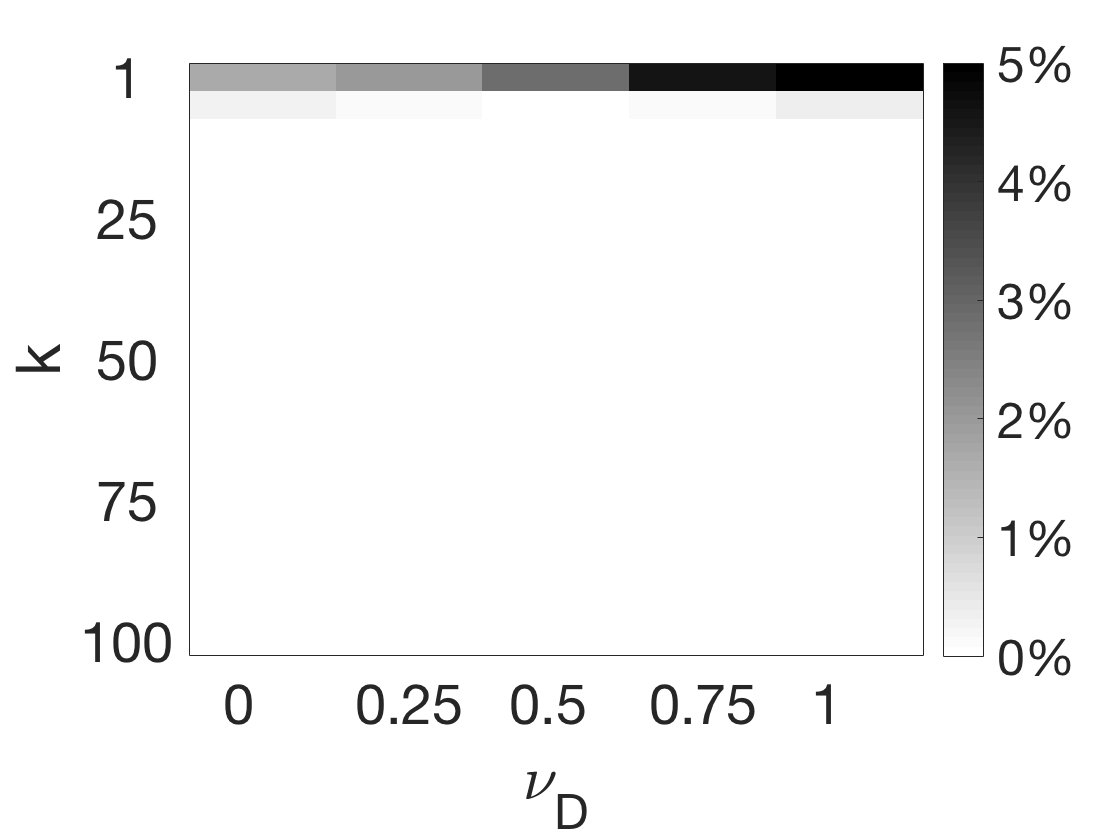}\label{fig:heatmapGasLS05b}}
   \subfigure[$\nuh$ = 1.0]{\includegraphics[width=0.32\textwidth]{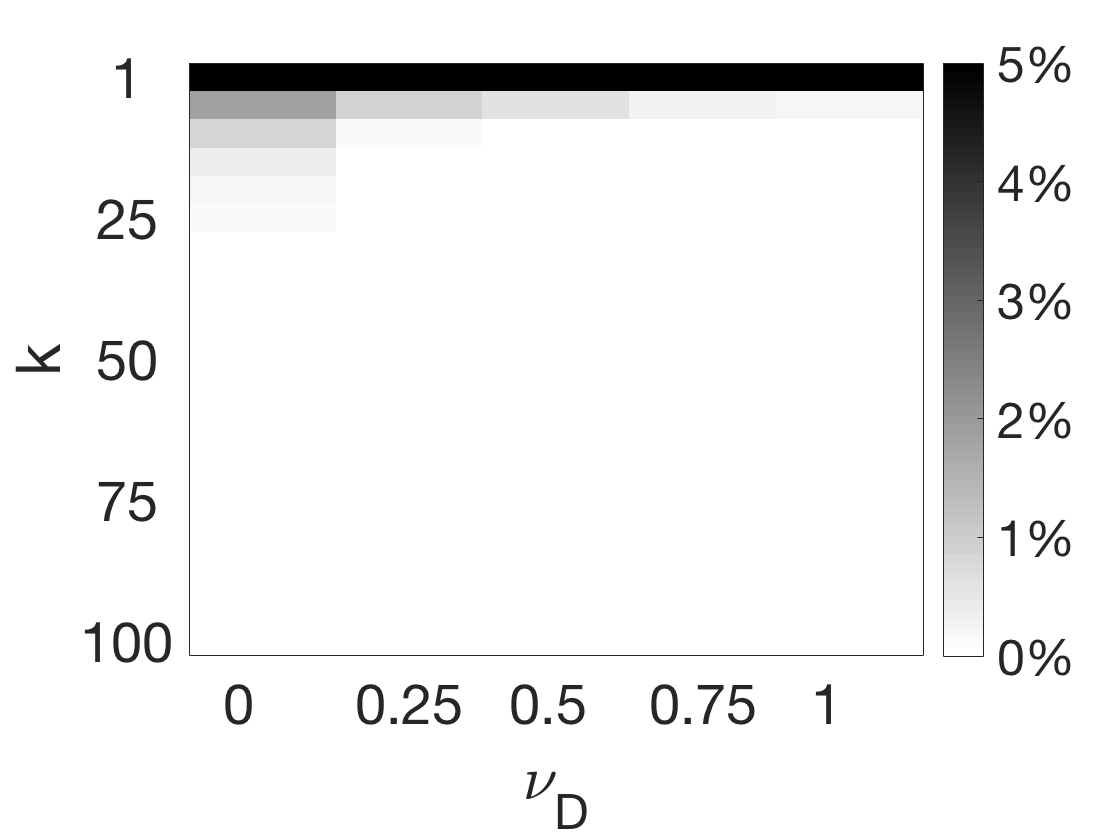}\label{fig:heatmapGasLS10b}}  
  \caption[caption]{ Comparison between the objective value for $\ldim < \nsamples$ and $\ldim = \nsamples$ for  the elastic net \RFMs.
  The reported relative differences are computed between the objective value for $\ldim$ and the objective values at $\ldim = \nsamples$, normalized by
  the objective value at $\ldim = \nsamples$ to give a percentage difference. 
  The lighter the color, the smaller the difference. 
 The loss is the least-squares loss for (a) - (c), and the logistic loss for (d) - (f), with $\nsamples = 100$ and $\xdim = 50$ on data randomly drawn from the Extended Yale Face Database B.   
 The relative differences are averaged over 10 runs, with the maximal distance reported within each run from 10 different random initializations.
 }
  \label{fig:heatmap}
\end{figure*}

\subsection{Batch optimization algorithm for induced \RFMs}

In this section, we include algorithmic details on the alternating minimization. 
We opt for a basic alternating minimization
strategy, without specific initialization or explicit strategies to escape saddle points. 
In particular, we focus only on making the alternating minimization
more effective by improving the convergence rate in two ways: (a) using an incomplete
alternating minimization and (b) deriving a proximal operator for the squared $\ell_2$ norm, required
for both the sparse coding and elastic net \RFMs. 


The alternating minimization algorithm for \RFMs---summarized in Algorithm \ref{alg_amrfm}---
is a standard block coordinate descent algorithm, with a few specific choices that we found effective. 
The standard approach involves descending in one variable with the other fixed, and then alternating. 
The main algorithmic choices are to use inexact updates for the alternating minimization,
with proximal gradient updates for non-smooth regularizers.
Empirically, we found both these modifications significantly speed convergence.
We opt for proximal gradient approaches rather than other approaches, such as 
the alternating direction method of multipliers, because it maintains sparse variables which can be
more efficiently stored and because the convergence results for proximal gradient
approaches are well understood, even under approximate updates \cite{machart2012optimal}.

\paragraph{Inexact alternating step.}
To alternate between $\Dmat$ and $\Hmat$ in the optimization, one can completely solve for each variable with the other fixed (exact) or alternate between single gradient descent steps (inexact). Both approaches converge under general conditions, proven as part of more general results about the convergence of block coordinate descent for multi-convex problems using exact updates \cite{xu2013ablock} and inexact updates \cite{tappenden2014inexact}. In our own experiments, we found the exact updates significantly slower and only marginally less sensitive to parameter choices, such as the step-size. We therefore adopt the inexact method, which is significantly faster. 

\newcommand{\useregwgt}{}
\newcommand{\RComment}[1]{\hfill $\triangleright$ #1}

\begin{algorithm}[htp]
\caption{Alternating Minimization for Dictionary Learning Models (AM-\RFM)}
\label{alg_amrfm}
\begin{algorithmic}
\State $\triangleright$ Input loss components $\jointloss$, $\regd$, $\regh$, $\ldim$
\State $\Dmat, \Hmat \gets $ random full-rank matrices with inner dimension $\ldim$
\State prevobj $\gets \infty$
\Repeat
\State Update $\Dmat$ using one step of gradient descent (such as in Algorithms \ref{alg_gradsmooth}, \ref{alg_one} or \ref{alg_elastic})
\State Update $\Hmat$ using one step of gradient descent (such as in Algorithms \ref{alg_gradsmooth}, \ref{alg_one} or \ref{alg_elastic})
\State currentobj $\gets \jointloss(\Dmat\Hmat) + \regd(\Dmat) + \regh(\Hmat)$
\If {$| \text{currentobj} - \text{prevobj} |< $ tolerance} \textbf{break}
\EndIf
\State prevobj $\gets$ currentobj
\Until{convergence within tolerance or reach maximum number of iterations}\\
\Return $\Dmat, \Hmat$
\end{algorithmic}
\end{algorithm}

\begin{algorithm}[htp]
\caption{Standard gradient descent step for smooth regularizers}
\label{alg_gradsmooth}
\begin{algorithmic}
\State $\Dmat \gets \Dmat - \stepsize_D (\nabla_Z \jointloss(\Dmat\Hmat) \Hmat^\top +  \nabla_D \regd(\Dmat))$  \RComment{$\Hmat$ update:
$\Hmat \gets \Hmat - \stepsize_H (\Dmat^\top \nabla_Z \jointloss(\Dmat\Hmat) +  \nabla_H \regh(\Hmat))$}
\end{algorithmic}
\end{algorithm}

\begin{algorithm}[]
\caption{Proximal gradient descent step for $\ell_1$ regularizer}
\label{alg_one}
\begin{algorithmic}
\State $\triangleright$  The steps are written for updating $\Dmat$ with fixed $\Hmat$. Differences in the updates to $\Hmat$ are in comments.
\State $\Delta \gets \nabla \jointloss(\Dmat\Hmat) \Hmat^\top$ \RComment{$\Hmat$ update:  $\Delta \gets \Dmat^\top\nabla \jointloss(\Dmat\Hmat)$}
\State $l \gets$ the Lipschitz constant of the gradient (or an upper bound to the constant)
\For{all $j$}
	\For{all $i$}  \RComment{$\Hmat$ update: swap indices to update one row at a time}
	\State $\Dmat_{ij} = \sign{\Dmat_{ij} - \Delta_{ij}/l} \max(| \Dmat_{ij} - \Delta_{ij}/l  | - \frac{\regwgt}{l} , 0)$	
	\EndFor
	\State $\Delta \gets \nabla \jointloss(\Dmat\Hmat) \Hmat^\top$  \RComment{$\Hmat$ update:  $\Delta \gets \Dmat^\top\nabla \jointloss(\Dmat\Hmat)$}
\EndFor
\end{algorithmic}
\end{algorithm}
\begin{algorithm}[]
\caption{Proximal gradient descent step for elastic net regularizer with $\ell_1^2$}
\label{alg_elastic}
\begin{algorithmic}
\Repeat
\For{$i = 1$ to $\ldim$}
	\State $\triangleright$ Calculate Lipschitz constant. For example, if $\jointloss$ is the least-squares loss, then   
	\State $\triangleright$  for the $\Dmat$ update: $l \gets 2(\| \Hmat_{i:} \|_1^2 + \nud \regwgt)$ or for the $\Hmat$ update:  $l \gets 2(\| \Dmat_{:i} \|_1^2 + \nuh \regwgt/\hscale)$
	\State $l \gets$ the Lipschitz constant of the gradient (or an upper bound to the constant)
	\State Update $\Dmat_{:i}$ (or $\Hmat_{i:}$) using Algorithm \ref{alg_st} with current $\Dmat$ (or $\Hmat$)
\EndFor
\Until{difference between $\Dmat$ (or $\Hmat$) on successive steps within tolerance or maximum number of iterations}
\end{algorithmic}
\end{algorithm}

\newcommand{\sparsity}{r}

\begin{algorithm}
\caption{Soft thresholding algorithm for updating column vectors of $\Dmat$}
\label{alg_st}
\begin{algorithmic}
\State $\lambda \gets \tfrac{(1-\nud)\regwgt}{l}$ \RComment{To update $\Hmat$:  $\lambda \gets \tfrac{(1-\nuh)\regwgt}{sl}$}
\Repeat
\State $\uvec \gets \Dmat_{:i} - \tfrac{1}{l}[\nabla\jointloss(\Dmat\Hmat)\Hmat_{i:}^\top+2\nud\regwgt \Dmat_{:i}]$ \RComment{$\Hmat$ update:  $\uvec \gets \Hmat_{i:} - \tfrac{1}{l}[\Dmat_{:i}^\top\nabla\jointloss(\Dmat\Hmat)+2\nuh\regwgt/\hscale \Hmat_{i:}]$}
\State Sort $\uvec = [u_1, \ldots, u_\ldim]$ according to $\left| u_i\right|$ in descending order such that $\left| u_{m_1}\right| \geq \left|u_{m_2}\right|\geq ... \geq \left|u_{m_\ldim}\right|$
\State $C \gets 0, \sparsity \gets 0$, 
\While{$\sparsity < \ldim$ and $\left|u_{m_{\sparsity+1}}\right| > \tfrac{2\lambda C}{1+2\lambda \sparsity}$}
	\State $C \gets C + \left|u_{m_{\sparsity+1}}\right|$ 
	\State $\sparsity \gets \sparsity+1$
\EndWhile

\For{$j = 1$ to $\ldim$}
\State $\Dmat_{ij} \gets \sign{u_{j}}\max\left(\left|u_{j}\right|-\tfrac{2\lambda C}{1+2\lambda \sparsity},0\right)$ \RComment{$\Hmat$ update: flip indices to update $\Hmat_{ji}$}
\EndFor

\Until{difference between $\Dmat$ on successive steps within tolerance or reach maximum number of iterations}
\end{algorithmic}
\end{algorithm}

%

\paragraph{Proximal gradient updates.} 
A standard gradient descent step is problematic when the regularizer, or a component of the regularizer, is non-differentiable. For example, $\ell_1$ has a non-differentiable point at $0$. 
Though one could simply choose a subgradient at $0$ and apply subgradient descent, in practice for batch optimization, the convergence properties are poor. 
For alternating minimization on induced \RFMs, we found that the descent would converge to a point and then very slowly decrease over a large number of iterations. 

Proximal gradient methods, on the other hand, use a proximal operator that avoids computation of the subgradient of the non-differentiable component. 
For example, for the $\ell_1$ regularizer, with Lipschitz constant $l$ for the gradient of the loss, the proximal operator is the soft-thresholding operator
\begin{align*}
\text{prox}(\uvec) = \argmin_{\uvec \in \RR^{\nsamples}} \frac{1}{2} \| \uvec - \hvec \|_2^2 + \frac{\regwgt}{l} \| \hvec \|_1 = \sign{\uvec} \circ \max(| \uvec | - \frac{\regwgt}{l} , \zerovec)
\end{align*}
where the multiplication $\circ$ is element-wise. 
Proximal operators are well-known for common non-smooth regularizers such as $\|\cdot\|_1$ and $\|\cdot\|_2$ \cite{bach2011convex}.
Some regularizers for induced \RFMs, however, involve squares of convex functions, which is atypical. For example, $\|\cdot\|_1^2$ is used in the elastic net \RFM. 
To the best of our knowledge, proximal operators have not been derived for $\|\cdot\|_1^2$. 
We provide the proximal updates for the elastic net \RFM\   in Algorithm \ref{alg_elastic} and \ref{alg_st}, with proof of the validity of the operator in Proposition \ref{prop_prox} in Appendix \ref{app_prox}.

\section{Incremental estimation}

In this section, we explore how to effectively learn $\Dmat$ incrementally. We discuss an incremental algorithm for least-squares losses that summarizes
past data (online AM-\RFM) and discuss how to use stochastic gradient descent updates.
We provide experiments demonstrating the different properties of the online versus stochastic incremental algorithms,
particularly providing insights into step-size selection.

\subsection{Online AM-\RFM\   algorithm}
For a least-squares loss, a more sample efficient incremental algorithm
can be used to optimize induced \RFMs. This algorithm is summarized in Algorithm \ref{alg_oamrfm}.
It is a straightforward modification of \citet[Algorithm 1]{mairal2010online},
which was introduced for sparse coding and for which they proved convergence
to stationary points \citep[Proposition 1]{mairal2010online}.
The main modification is to use a regularizer on $\Dmat$ instead of a projection onto a constraint set.  
Note that \citet[Section 3.4]{mairal2010online} discuss several algorithmic improvements,
including down-weighting past samples and using mini-batches; these ideas extend
directly to this algorithm and so we do not repeat them here. 

  

\begin{algorithm}[htp!]
\caption{Online AM-\RFM}
\label{alg_oamrfm}
\begin{algorithmic}
\State $\triangleright$ Input full rank initial dictionary $\Dmat_0 \in \RR^{\xdim \times \ldim}$, $\nsamples$ as number of iterations
\State $\Amat_0 \gets \zerovec \in \RR^{\ldim \times \ldim}$, $\Bmat_0 \gets \zerovec \in \RR^{\xdim \times \ldim}$
\For{$t=1$ to $\nsamples$}
\State $\triangleright$ If multiple passes over a dataset, reshuffle the data for each pass
\State $\triangleright$ As before, if there is a non-smooth regularizer on $\Dmat$ or $\Hmat$, then use proximal updates
\State $\hvec_t \gets \argmin_{\hvec \in \RR^\ldim} \frac{1}{2} \sqnorm{\Dmat_{t-1} \hvec - \xvec_t}^2 + \regh(\hvec)$
\State $\beta_t \gets \max(\tfrac{1}{t},0.01)$
\State $\Amat_t \gets (1-\beta_t) \Amat_{t-1} + \beta_t \hvec_t \hvec_t^\top$
\State $\Bmat_t \gets (1-\beta_t) \Bmat_{t-1} + \beta_t \xvec_t \hvec_t^\top$
\State Compute $\Dmat_t$ using $\Dmat_{t-1}$ as initial point (warm restart)
\begin{align*}
\Dmat_t &= \argmin_{\Dmat \in \RR^{\xdim \times \ldim}} \frac{1}{t} \sum_{i=1}^t \left(\frac{1}{2}\sqnorm{\Dmat \hvec_i - \xvec_i}^2 +  \regh(\hvec_i) \right ) + \regd(\Dmat)\\
&= \argmin_{\Dmat \in \RR^{\xdim \times \ldim}} \frac{1}{2} \trace{\Dmat^\top \Dmat \Amat_t} - \trace{\Dmat^\top \Bmat_t} +  \regd(\Dmat) 
\end{align*}
\EndFor\\
\Return $\Dmat_\nsamples$
\State $\triangleright$ If it takes at most $S$ descent iterations to get $\hvec_t$ or $\Dmat_t$, then the time complexity for each $t$ is $O(S \xdim \ldim^2)$ 
\end{algorithmic}
\end{algorithm}

\begin{algorithm}[htp!]
\caption{SGD AM-\RFM}
\label{alg_sgdrfm}
\begin{algorithmic}
\State $\triangleright$ Input full rank initial dictionary $\Dmat_0 \in \RR^{\xdim \times \ldim}$, $\nsamples$ as number of iterations
\For{$t=1$ to $\nsamples$}
\State $\triangleright$ If multiple passes over a dataset, reshuffle the data for each pass
\State $\triangleright$ If there is a non-smooth regularizer on $\Hmat$, use sub-gradient descent to get $\hvec$
\State $\hvec_t \gets \argmin_{\hvec \in \RR^\ldim} L_x(\Dmat_{t-1} \hvec, \xvec_t) +  \regh(\hvec)$
\State $\triangleright$ Update $\Dmat_t$ using one step of stochastic subgradient descent step
\State Select step-size $\stepsize_t$ as described in Section \ref{sec_sgd}
\State $\Dmat_{t} \gets \Dmat_{t-1} - \stepsize_t (\nabla l_\sampiter(\Dmat_{t-1}) + \nabla \regd(\Dmat_{t-1}))$ \RComment{$\nabla l_t(\Dmat_{t-1}) = \nabla L_x(\Dmat_{t-1}\hvec,\xvec_t) \hvec^\top$}
\EndFor\\
\Return $\Dmat_\nsamples$
\State $\triangleright$ If it takes at most $S$ descent iterations to get $\hvec_t$ or $\Dmat_t$, then the time complexity for each $t$  is $O(S\xdim \ldim)$. The time complexity for the step-size selection strategies, is dominated by this; for accelerated SG, the time complexity is $O(\xdim + \ldim)$, and for the others is constant. 
\end{algorithmic}
\end{algorithm}

\subsection{Stochastic gradient descent for \RFMs}\label{sec_sgd}

Stochastic gradient descent (SGD) is another approach to incremental estimation of \RFMs. 
As discussed in Section \ref{sec_decoupled}, only $\Dmat$ is updated incrementally, according to the loss
\begin{align}
l_\sampiter(\Dmat) =  \regd(\Dmat) + \big(\min_{\hvec} L_x(\Dmat \hvec, \xvec_\sampiter) + \regh(\hvec) \big)
.
\end{align}
Therefore, on each step, the optimal $\hvec$ is computed for the current $\Dmat$ and data point $\xvec_t$, and
then a gradient step is performed for $\Dmat$ using $l_{\sampiter}(\Dmat)$.
Unlike online AM-\RFM, we used subgradient descent updates for both $\hvec$ and $\Dmat$ for SGD AM-\RFM. Using proximal gradient updates for either $\hvec$ or $\Dmat$ resulted in $\hvec$ becoming progressively more sparse and the resulting solution was poor. 
This remains an important open question for future work in using stochastic gradient descent for induced \RFMs. 
We summarize the SGD choices we found to be effective, in Algorithm \ref{alg_sgdrfm}. 

We will empirically investigate which of these two incremental algorithms is generally more effective for optimizing induced \RFMs\ incrementally. The SGD algorithm is significantly faster per-step, scaling as $O(\xdim \ldim)$, compared to online AM-\RFM, which scales as $O(\xdim \ldim^2)$. Additionally, the SGD algorithm allows for more general losses, whereas online AM-\RFM\ is restricted to the least-squares loss. However, we expect that online AM-\RFM\ should be more sample efficient. Further, it should be simpler to select step-sizes for online AM-\RFM, because line search can be used for the two optimizations within the loop. For SGD, we will need to explore alternative heuristics for setting the step-size.

The key factor for making SGD effective for induced \RFMs\ is to incorporate recent accelerations to SGD. 
\citet{mairal2009online} found that for sparse coding, their online AM-\RFM\   algorithm converged more quickly in terms of samples
than SGD. We find, however, that this is not always the case, particularly by taking advantage of the recent understanding of accelerations to stochastic gradient descent (\citet{roux2012astochastic} provide a nice summary). 
Several of these accelerations rely on strategies for selecting the step-size, $\stepsize_t$. We explored several of these accelerations,
and report results for the two that were the most effective with induced \RFMs\ and that did not involve storing gradients. 

\paragraph{Accelerated SG.} The first acceleration involves an aggressive step-size selection strategy that enables a constant
step-size for several iterations. This strategy gives a more aggressive step-size, that can speed convergence, but that adaptively is decreased
to ensure convergence \cite{kesten1958accelerated}. On each step, the step-size is only decreased when the inner-product between two successive gradient estimates are negative. Specifically, a list of strictly decreasing step sizes is predefined, with $\stepsize_0$ initialized to the first step size in this list. If $\tr(\nabla l_{\sampiter}(\Dmat_{t-1})^\top\nabla l_{\sampiter}(\Dmat_{t-2})) \geq 0$, then $\stepsize_t = \stepsize_{t-1}$; otherwise, the step-size is chosen such that $\stepsize_t < \stepsize_{t-1}$, from the list of step-sizes.

\paragraph{Momentum.} The second acceleration uses a momentum term, which is the difference between two successive iterations (see for example the algorithm by \citet{tseng1998anincremental}). The update for $\Dmat$ is instead
\begin{align*}
\Dmat_{t} \gets \Dmat_{t-1} - \stepsize_t (\nabla l_\sampiter(\Dmat_{t-1}) + \nabla \regd(\Dmat_{t-1})) + \beta_t (\Dmat_{t-1} - \Dmat_{t-2})
\end{align*}
Without the momentum term, $\Dmat_{t-1} - \Dmat_{t-2} =  -\eta_{t-1} \nabla l(\Dmat_{t-2})$. The momentum term is stepping further
along the direction of this previous gradient. Momentum's acceleration effects are mainly supported by theories in batch optimization,
and has not been theoretically characterized for stochastic settings \cite{sutskever2013ontheimportance}. In practice, however, momentum is typically combined with SGD. Our experiments in Section \ref{sec_experiments} show that adding momentum term can accelerate SGD AM-\RFM, but not always.

\begin{figure*}[htp!]
  \centering
 	\subfigure[Subspace]{\includegraphics[width=0.48\textwidth]{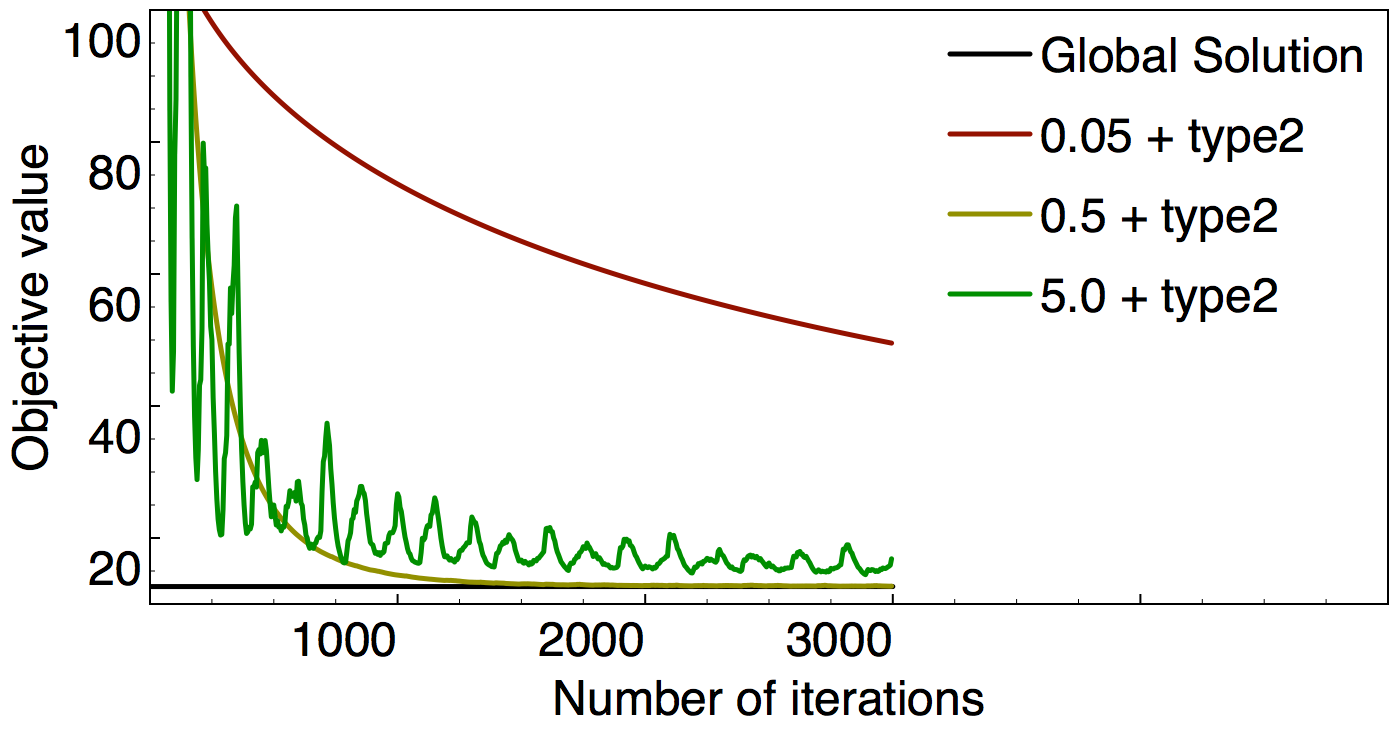}\label{fig:subspace1}}
  	\subfigure[Elastic Net ($\nuh = 1.0, \nud = 0.5$)]{\includegraphics[width=0.48\textwidth]{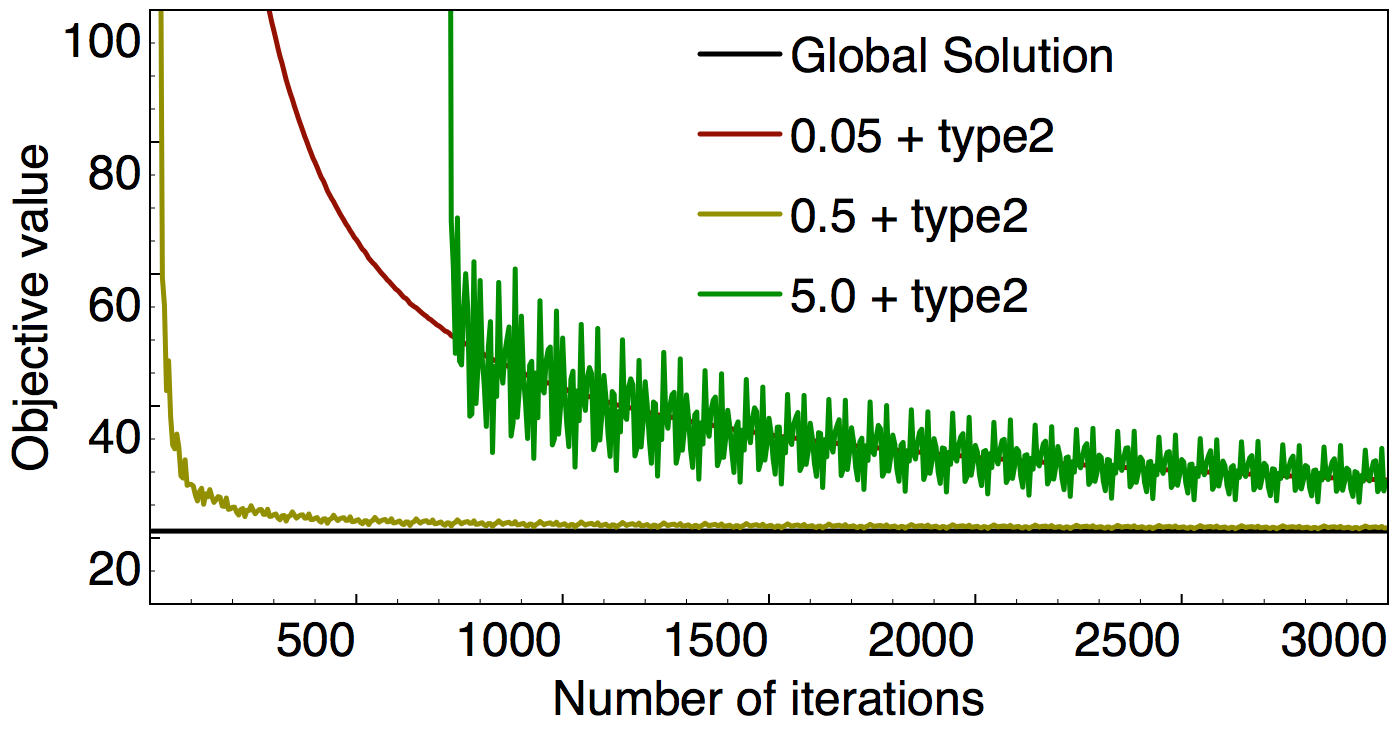}\label{fig:combined1}}
  	\subfigure[Subspace]{\includegraphics[width=0.48\textwidth]{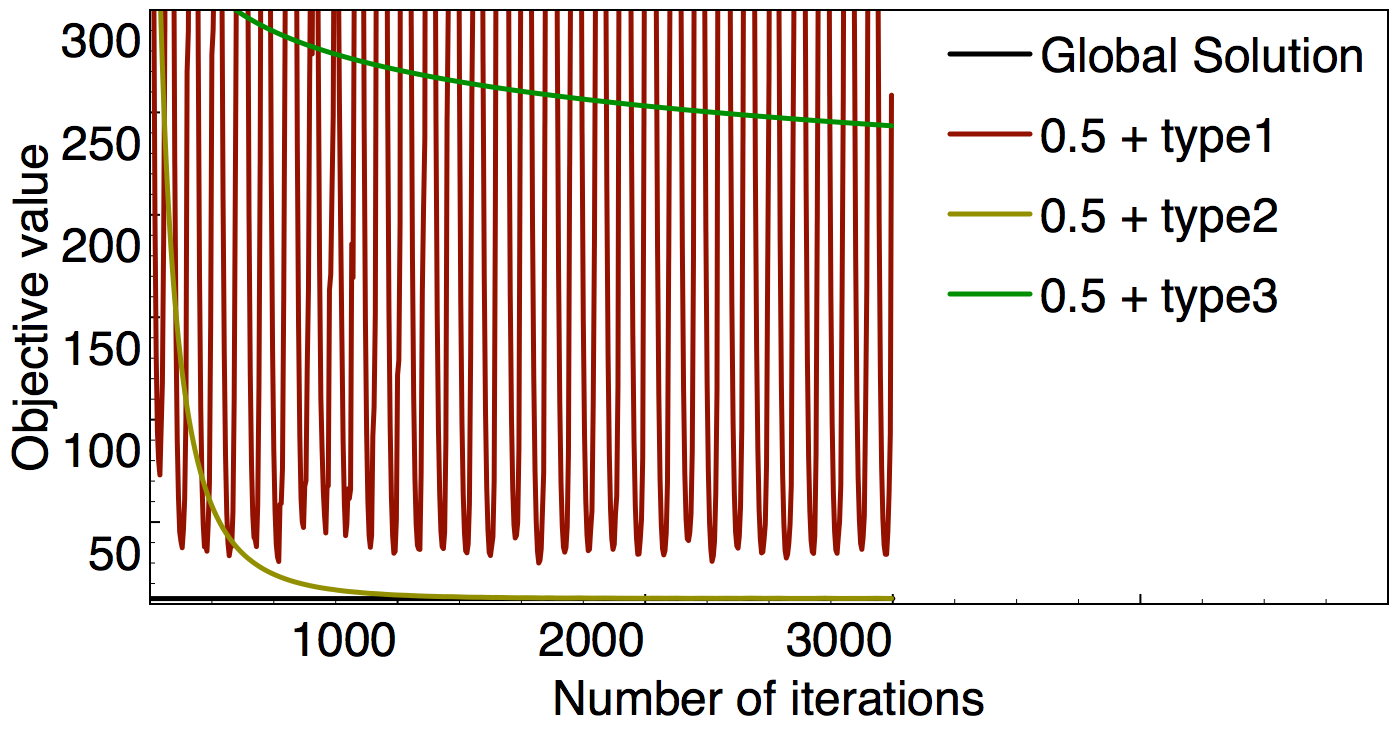}\label{fig:subspace2}}
  	\subfigure[Elastic net ($\nuh = 1.0, \nud = 0.5$)]{\includegraphics[width=0.48\textwidth]{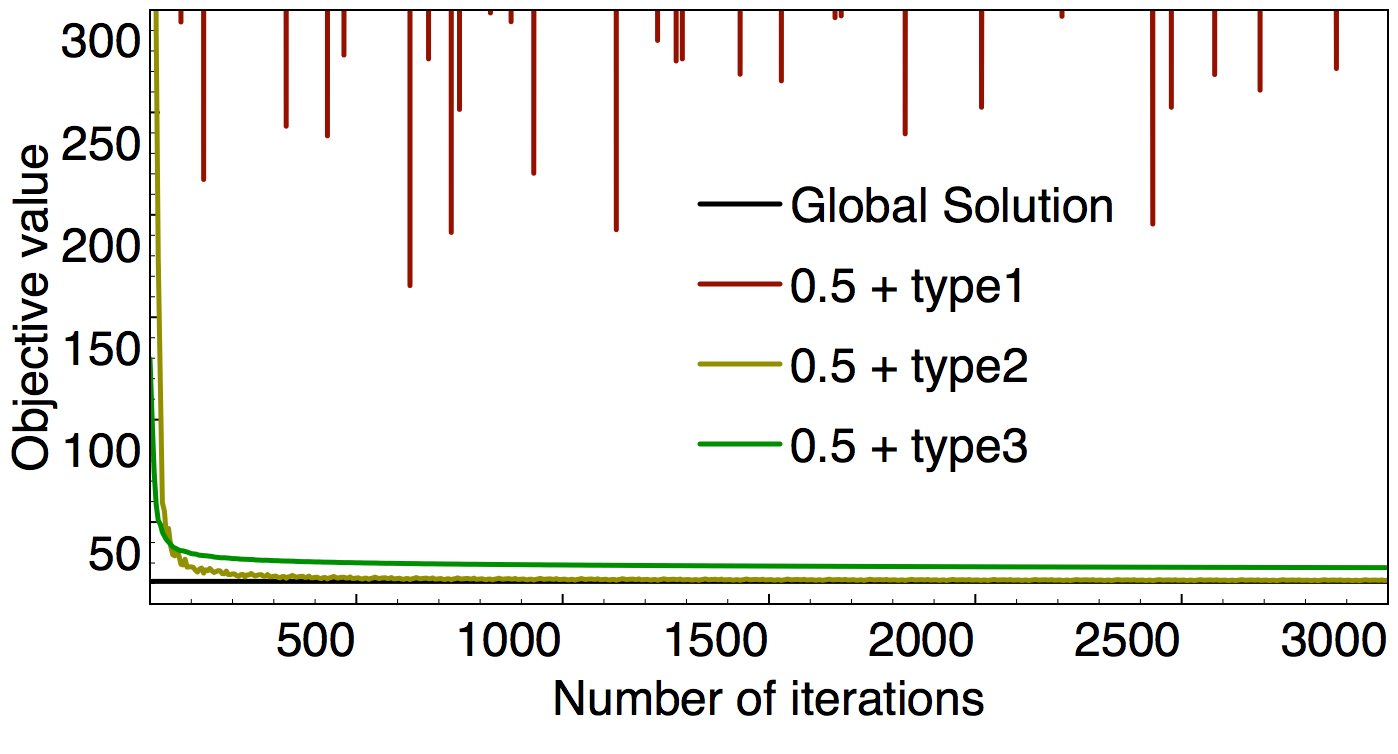}\label{fig:combined2}}
  \caption{ Comparison of SGD AM-\RFM\   to the global solution for varying $\alpha_0$ and step-size decays. 
  The initial step-size is $\stepsize_0 \in \{0.05,0.5,5.0\}$.
  There are three time-based decays to get $\stepsize_t$: $\stepsize_0$ (type 1), $\stepsize_0/\sqrt{t}$ (type 2) and $\stepsize_0/t$ (type 3).
  In general, we found type 2 to perform the best across $\stepsize_0$, and found $\stepsize_0 = 0.5$ to produce the fastest, most stable convergence.
  We therefore more specifically report results for these two settings, with (a), (b) for type 2 with varying $\stepsize_0$ and (c), (d) for $\stepsize_0 = 0.5$ and varying types. 
  These results indicate that learning can be significantly different across these choices. 
As expected, an aggressive initial step-size of $5.0$ produces oscillations, though it does in fact converge with the type 2 decay schedule. 
A conservative step-size of $0.05$ converges too slowly, and an overly fast decay (type 3) may converge before reaching global optima. 
Setting $\stepsize_0 = 0.5$ and using type 2 decay converges to a global solution quickly and smoothly.       
 }
  \label{fig:comparison1}
\end{figure*}

\begin{figure}[htp!]
  \centering
   \subfigure[Subspace]{\includegraphics[width=0.48\textwidth]{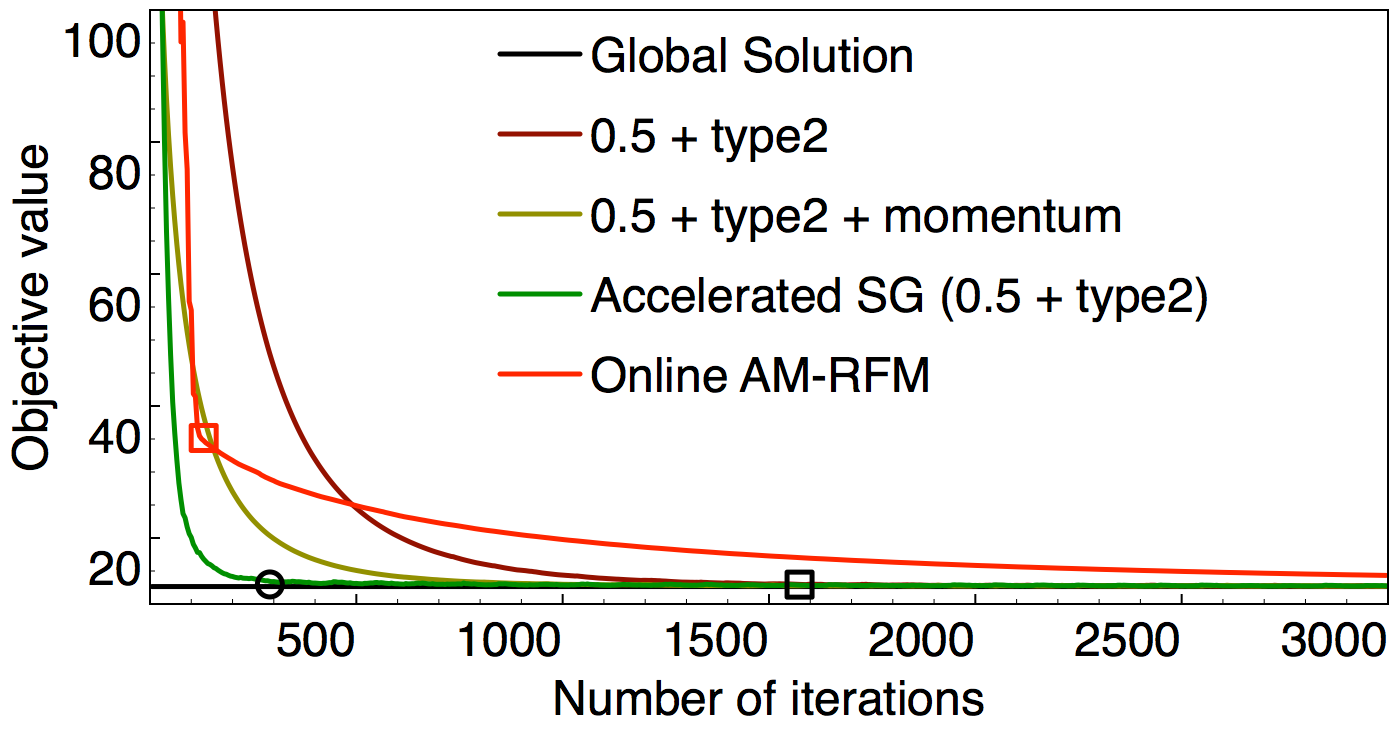}\label{fig: comparison of different SGD methods1}}
   \subfigure[Elastic net ($\nuh = 1.0, \nud = 0.5$)]{\includegraphics[width=0.48\textwidth]{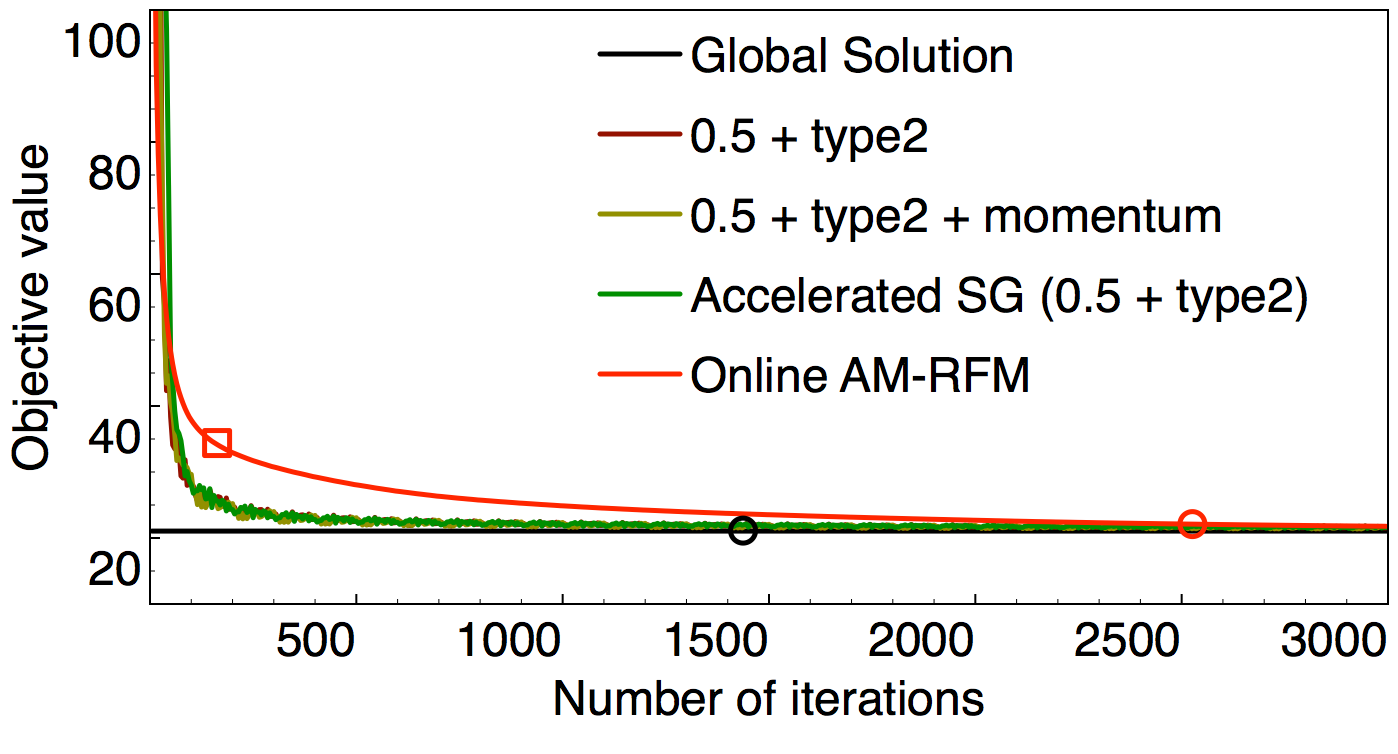}\label{fig: comparison of different SGD methods2}}
   \subfigure[Subspace]{\includegraphics[width=0.48\textwidth]{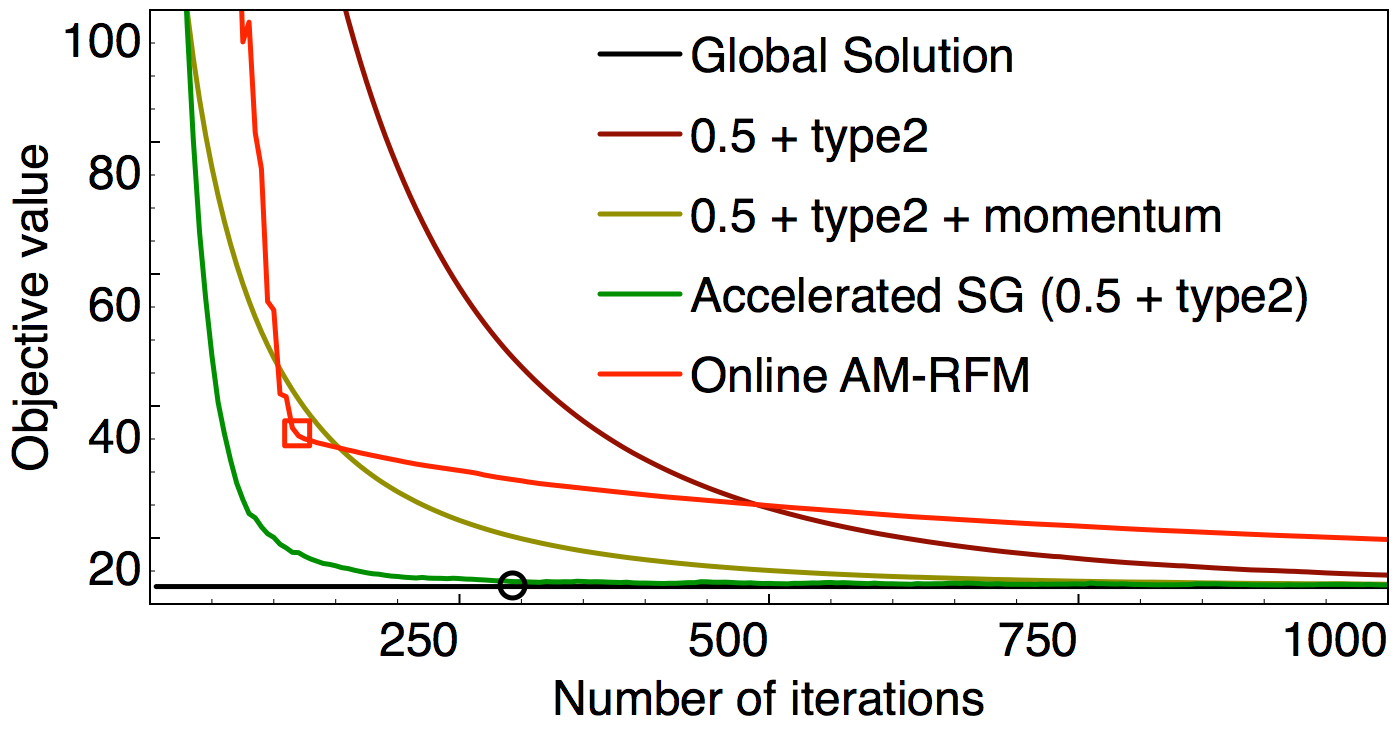}\label{fig: comparison of different SGD methods3}}
   \subfigure[Elastic net ($\nuh = 1.0, \nud = 0.5$)]{\includegraphics[width=0.48\textwidth]{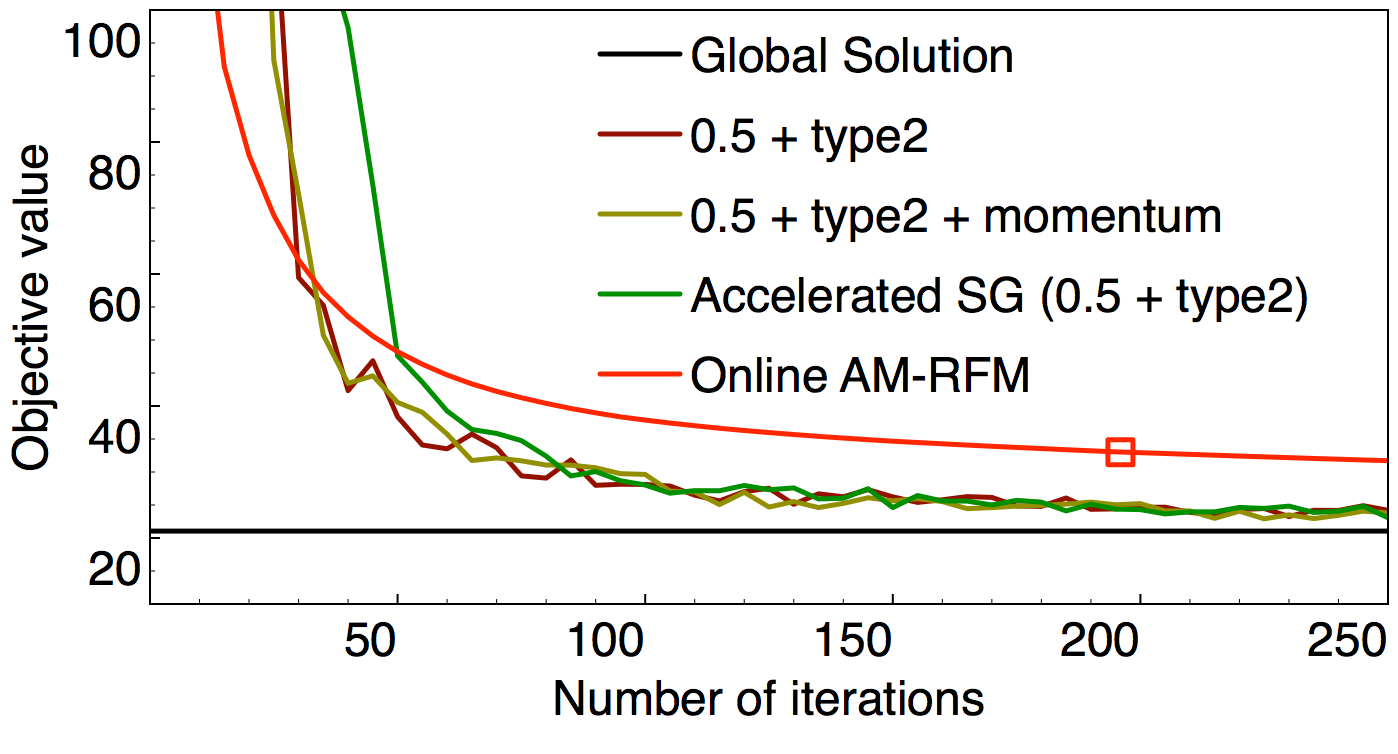}\label{fig: comparison of different SGD methods4}}
  \caption{Comparison of different incremental algorithms for induced \RFMs.
    The plots in (c) and (d) depict early learning for (a) and (b) respectively (i.e., they are zoomed into early learning in those plots).
    Accelerated SG is SGD AM-\RFM\   with $\stepsize_t = \stepsize_{t-1}$ if $\tr(\nabla l_{\sampiter}(\Dmat_{t-1})^\top\nabla l_{\sampiter}(\Dmat_{t-2})) \geq 0$
    and otherwise selects $\stepsize_t = \stepsize_0/\sqrt{t}$ (type 2) or $\stepsize_t = \stepsize_0/t$ (type 3). For momentum, $\beta_t = 0.01$. 
    Both of these accelerations outperformed the best setting for standard SGD AM-\RFM\   (with $\stepsize_0 = 0.5$ and type 2 decay), for subspace dictionary learning.
    For elastic net, however, the accelerations appeared to have little impact, potentially due to the fact that convergence of the standard SGD AM-\RFM\   was
    already fast.
    A surprising observation is that online AM-\RFM\   decreases the function value very fast at the beginning, but then slows before converging to the global solution. 
    Finally, we report the runtime of SGD AM-\RFM. The two marks in each plots denote convergence and runtime information based on Accelerated SG. 
    The black circle denotes the step where Accelerated SG is within 5\% of the batch solution; the black square denotes the step where Accelerated SG spends as much cumulative time as the batch one. If a mark is missing in a plot, then the mark is beyond the maximum iterations. These black marks are similar for the other stochastic approaches.
    This convergence is within $30\%-40\%$ of the time needed for the global solution.
    This is a serious underestimate of the true gains of stochastic gradient descent, because $\nsamples = 100$ was chosen to make the comparison to
    the global batch solution computationally feasible. In practice, this value is much larger and stochastic gradient descent becomes the obvious choice
    for estimating induced \RFMs\ if efficiency is a concern. The red circle and square have the same meanings as the black ones except for that they are for online AM-\RFM. In contrast to black marks, this convergence is more than 10 times of the time needed for the global solution, demonstrating that online AM-\RFM\   is computationally costly and slow convergence while it is sample efficient and with no need to select step-size explicitly.  
    }         
  \label{fig:comparison2}
\end{figure} 

\begin{figure}[htp!]
  \centering
   \subfigure[Subspace]{\includegraphics[width=0.48\textwidth]{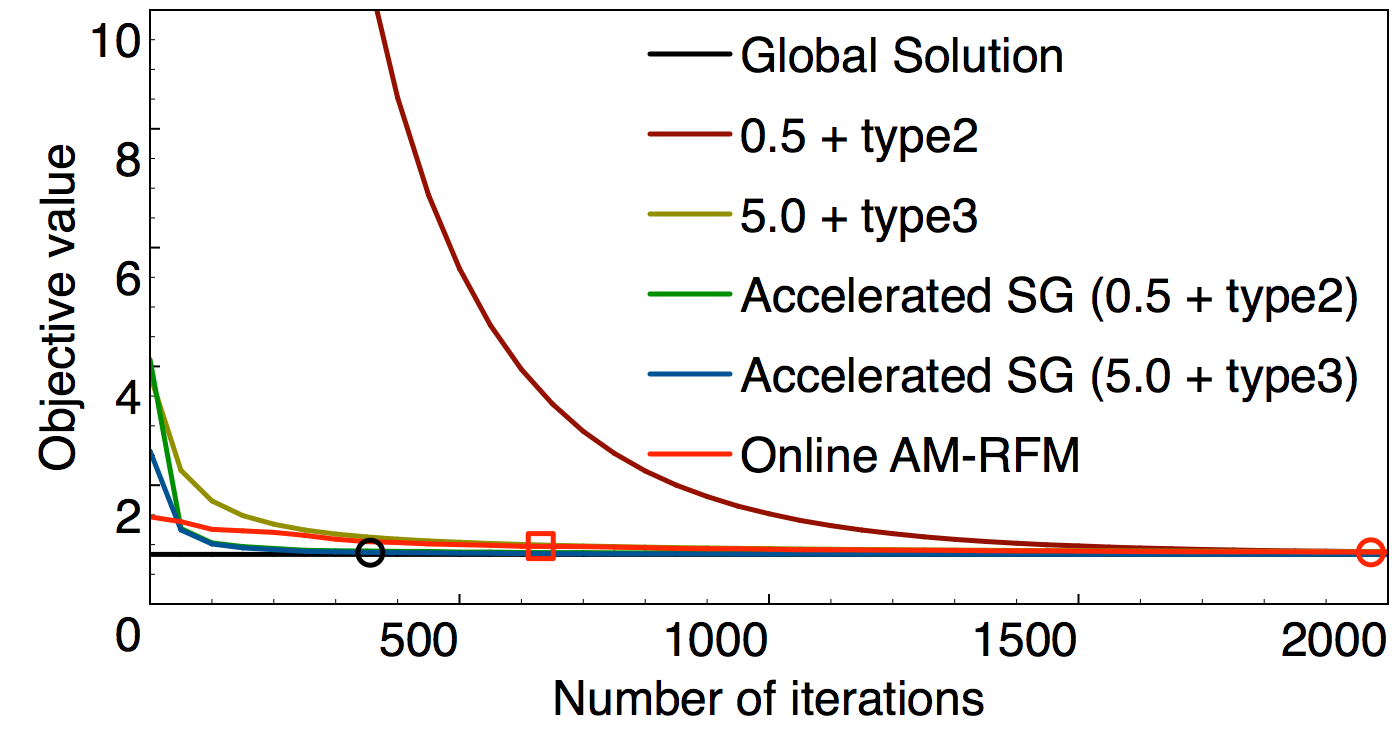}\label{fig: comparison of different SGD methods5}}
   \subfigure[Elastic net ($\nuh = 1.0, \nud = 0.5$)]{\includegraphics[width=0.48\textwidth]{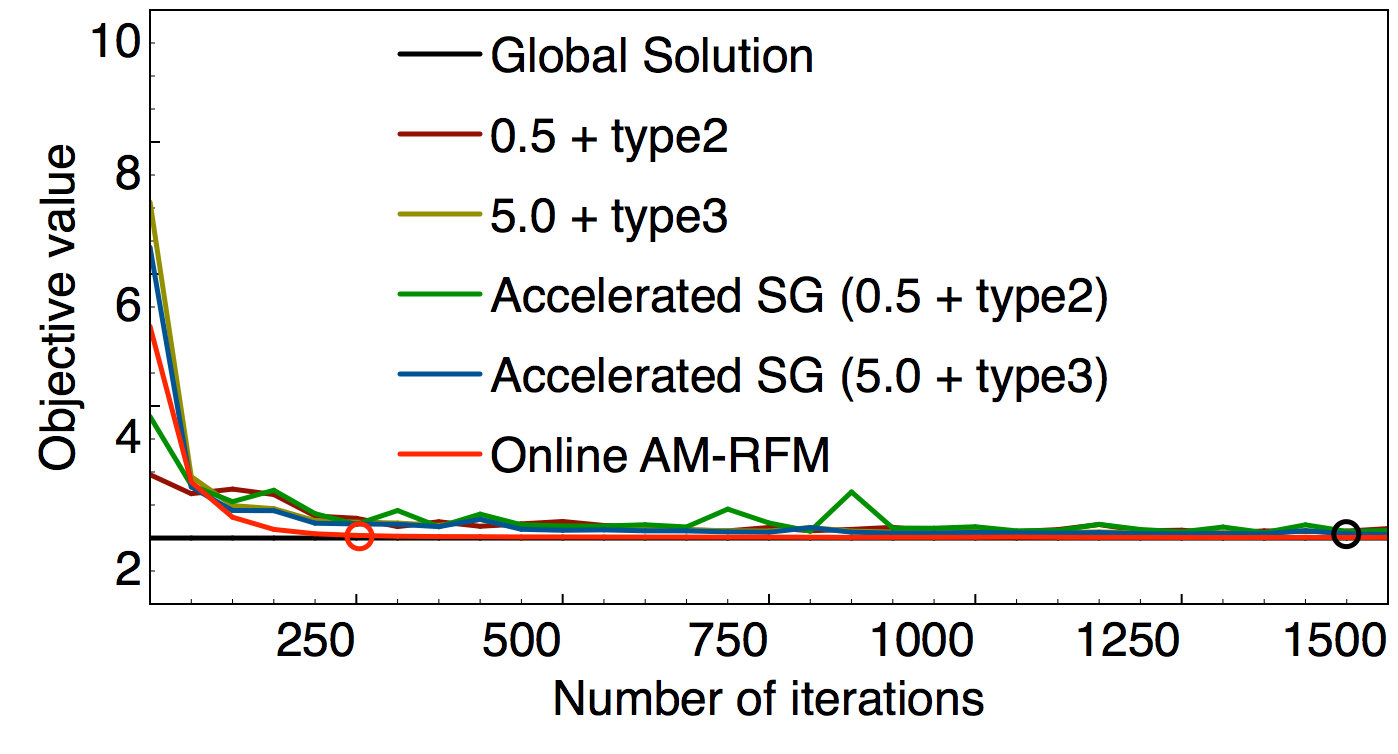}\label{fig: comparison of different SGD methods6}}
  \caption{ Comparison of different incremental algorithms for induced \RFMs\ on data randomly drawn from the Extended Yale Face Database B, with $\nsamples = 1000$ and $\xdim = 100$. All the settings are the same as in Figure \ref{fig:comparison2} except for that the black marks are for Accelerated SG (5.0 + type3) rather than Accelerated SG (0.5 + type2). In this experiment, setting $\stepsize_0 = 5.0$ and using type 3 decay converges faster than setting $\stepsize_0 = 0.5$ and using type 2. Accelerated SG still accelerates SGD AM-\RFM\   in (a), and appears to have little impact in (b) where SGD AM-\RFM\   already converges fast. The methods with momentum term are not shown in the plots, because they are almost the same as or even a bit worse than the ones without momentum, denoting that adding momentum is not guaranteed to improve the convergence of the SGD algorithm. We also notice that online AM-\RFM\   always needs much more time to converge than all the stochastic methods, even though it converges within fewer steps. For example, in (b), the time needed for online AM-\RFM\   to reach the red circle is around 1400\% of the time for Accelerated SG (5.0 + type3) to reach the black circle, where the solutions are within 5\% of the batch solution, despite that the red circle is within much fewer steps than the black one. 
  }         
  \label{fig:comparison3}
\end{figure}

\subsection{Experimental results}\label{sec_experiments}

In this section, we empirically investigate the properties of these incremental estimation strategies for induced \RFMs. 
Previous sections indicated the global optimality of the alternating minimization
for the batch setting; our goal in this section is to empirically demonstrate that this result holds for incremental estimation, and provide insight
into which strategies are most effective for incremental estimation. In addition to online AM-\RFM\   and the accelerations to SGD AM-\RFM,
we explore more basic step-size selection strategies, including constant step-sizes and standard decay schedules. 
We include results for poorly selected step-sizes, and do not necessarily advocate for any one algorithm.
Our goal is to provide preliminary insights into the incremental properties for induced \RFMs, 
to make it simpler for others to practically select and use these algorithms. 

The experiments compare the objective value of the incremental approaches to the global batch solution.
At each iteration, the objective value for the current $\Dmat$ is computed as 
$$\regwgt \sum_{i=1}^\ldim \| \Dmat_{:i} \|_c^2 + \frac{1}{\nsamples} \min_{\Hmat \in \RR^{\ldim \times \nsamples}} \sum_{\sampiter=1}^\nsamples \| \Dmat \Hmat_{:\sampiter} - \Xmat_{:\sampiter} \|_2^2 + \regwgt \sum_{i=1}^\ldim \| \Hmat_{i:} \|_r^2
.$$ 
The loss function is a least-squares loss, to enable comparison with online AM-\RFM.
The regularizer weight is $\regwgt = 0.05$, where the regularizer on $\Hmat$ is appropriately scaled
with the number of samples (as discussed in Section \ref{sec_samples}). 
The data is synthetically generated from a unit-variance Gaussian, with $\xdim = 50$ and $\nsamples = 100$. A
smaller $\nsamples$ was chosen to make it computationally feasible to run the batch algorithm, to give the global solution for comparison. 
The incremental algorithms iterate over the dataset multiple times, with a random reshuffling of data each time. 
The results are summarized in Figures \ref{fig:comparison1}, \ref{fig:comparison2} and \ref{fig:comparison3}.
Figure \ref{fig:comparison1} investigates properties of fixed step-sizes and decay schedules for SGD AM-\RFM. 
Figure \ref{fig:comparison2} and \ref{fig:comparison3} investigate the properties of all the algorithms, including SGD AM-\RFM\   with accelerations and online AM-\RFM, on synthetic and real data. 

The results are described more fully in the figure captions, but we provide the overall conclusions here. 1) SGD converges to a good solution much more efficiently than batch AM-\RFM; 2) both accelerations to SGD AM-\RFM\   improve the convergence rate and 3) online AM-\RFM\ converges in only a few epochs and is effective with line-search to select step-sizes, but is quite slow per step. The selection of step-sizes for SGD is not particularly sensitive, and even a relatively aggressive step size with decay converged after some oscillation. The step-size selection is even less sensitive with acceleration, because a more conservative list of step-sizes can be chosen. In this way, the algorithm can keep a larger step size for a number of steps until the sign changes and then pick
a more conservative step-size from the list. However, if one wants to avoid picking a step-size and more memory and computation is available, online AM-\RFM\   is a reasonable alternative.

\section{Related work}\label{sec_related}

There has been a recent focus on understanding alternating minimization for
problems in machine learning, and on effective strategies for escaping saddle points, particularly for neural networks.
We summarize much of this work here, to better complete the emerging picture for strategies
to optimize such biconvex objectives.

The most general and relevant related work is the treatise by \cite{haeffele2015global} on global optimality
for a wide class of models, including dictionary learning models and neural networks with rectified linear activations.
They show that when the solutions are rank-deficient, in their case specifically having a column of all zeros
in their factors, that local minima are in fact global minima. 
This work complements the results in this paper, because we focused on the overcomplete setting, where the solution is not rank deficient. 
Other more specific global optimality results, that largely fit within the results by \cite{haeffele2015global}, include
a setting with a least-squares loss,
unweighted trace norm and conditions on the value of the loss at the given stationary point \citep{mardani2013decentralized};
an unregularized setting \citep[Proposition 5]{abernethy2009anew};
 an online robust PCA setting where the true rank is assumed to be known \citep{feng2013online};
 and a recent result on local minima for deep linear neural networks \cite{kawaguchi2016deep}.  

A common strategy in previous efforts to obtain global solutions for biconvex problems
has been to provide careful initialization strategies and assume sparsity or coherence properties on the matrices. 
\citet{jain2013low}
showed that alternating minimization for matrix completion provides global solutions, assuming
that the given matrix and optimal factors are incoherent matrices \citep[Definition 2.4]{jain2013low}, 
which enables them to define a suitable initial point using a singular value decomposition. 
\citet{gunasekar2013noisy} and \citet{hardt2014ontheprovable} further generalize these assumptions, particularly to the noisy setting
and \citet{netrapalli2015phase} to the setting for phase retrieval. 
For sparse coding, \citet{spielman2012exact} showed how to recover the optimal dictionary exactly, 
assuming it has full column rank and assuming a maximum sparsity level on the basis. Since then,
most work has focused on the overcomplete case, where the dictionary is not full column rank. 
\citet{agarwal2017aclustering} showed that alternating minimization
for sparse coding with thresholding gives global solutions, if the dictionary satisfies the restricted
isometry property \cite{candes2005decoding} and using a provable strategy to initialize the dictionary to be near optimal.
Previously, both \citet{agarwal2017aclustering} and \citet{arora2014new} had independently provided a similar algorithm, with incoherence requirements
on the matrices, with slightly different requirements on the sparsity level. 
\citet{arora2015simple} built on this work and expanded the level of possible sparsity.

A separate line of work has required fewer assumptions on the true matrices and pursued initialization strategies by generating candidate columns, typically using
singular value decompositions \cite{aspremont2004direct,bach2008convex,journee2010low,zhang2012accelerated,mirzazadeh2015scalable}.
The bulk of this work has been on the semi-definite case, including seminal work by \citet{burer2005local},
where the two factors are the same, $\joint = \Dmat \Dmat^\top$.
The exception is the GCG algorithm \cite{zhang2012accelerated,yu2014generalized}, which generates columns and rows for $\Dmat$ and $\Hmat$ for some of the induced \RFMs\ considered in this work. Though the boosting algorithm is quite general, it requires
an oracle solution for a problem that is intractable except for in a small number of cases, which mostly correspond to those with known induced regularizers. 
Most of this work has not characterized the saddle-points of the non-convex optimization, 
with the exception of a recent work for the semi-definite setting which identified a smoothness condition on the
gradient of the loss that ensured convergence to global solutions \citep[Proposition 2]{mirzazadeh2015scalable}.
The results in this work suggest that it is likely unnecessary to carefully generate candidate columns and rows,
though for large $\xdim$, progressively increasing $\ldim$ could be computationally attractive. 

Another strategy has been to use the convex reformulations obtained for induced \RFMs.
For example, instead of alternating on $\Dmat$ and $\Hmat$ in the subspace formulation,
$\joint$ is solved for directly in trace norm regularized problem. From there, the
factors $\Dmat$ and $\Hmat$ could typically be recovered, usually with iterative methods. 
The problems tackled with this approach include
relaxed rank exponential family PCA \cite{bach2008convex,zhang2011convex}, 
matrix completion \cite{cai2010singular},
(semi-)supervised dictionary learning \cite{goldberg2010transduction,zhang2011convex},
multi-view learning \cite{white2012convex},
co-embedding \cite{mirzazadeh2014convex}
and
autoregressive moving average models \cite{white2015optimal}.
These approaches, however, are often not as scalable as learning the factors $\Dmat$ and $\Hmat$, which
can be smaller. Further, because $\joint$ grows with the size of the data, and the induced regularizer
often couples entries across $\joint$, the induced convex reformulation is typically not amenable to incremental estimation.
As an additional problem, 
for the general class of induced \RFMs, the induced convex formulation
is only of theoretical value, since there is no known form for the induced norm
and so the optimization over $\joint$ cannot be performed. 
There has been some effort to obtain approximations to the convex form for the elastic net, which does not have a known induced form \citep{bach2008convex};
the proposed approach, however, uses a convex lower bound and the relation to the elastic net solution is unclear. 
The alternating minimization approach avoids the need to have an explicit
form for the induced norm, and rather only uses the property that this class has induced forms to identify promising situations for global optimality.
  

For incremental estimation of induced \RFMs, there are significantly fewer approaches.
Outside of incremental principal components analysis \cite{warmuth2008randomized,feng2013online} and partial least-squares \cite{arora2012stochastic},
little is known about how to obtain optimal incremental \RFM\   estimation
and several approaches have settled for local solutions \cite{mairal2009online,mairal2009supervised,mairal2010online}.
\citet{desa2015global} prove convergence to a global solution, with random initialization, using incremental, stochastic gradient descent algorithms
for an interesting subset of induced \RFMs, including for matrix completion, phase retrieval and subspace tracking. 
These problems, however, are formulated as semi-definite problems, again by learning $\joint = \Dmat \Dmat^\top$. 
The incremental estimation results in this work apply to a broader class of models---induced \RFMs---
and provide justification for a simple descent approach with stochastic gradient descent. 

 
\section{Conclusion and Discussion}

In this work, we addressed tractable estimation of induced dictionary learning models (\RFMs),
which constitute a broad
and widely used formalism for unsupervised, semi-supervised
and supervised learning problems. 
We generalized the class of induced \RFMs\ and provided several
new objectives and characterizations of stationary points. 
We proved that local minima are global minima
for a subclass of induced \RFMs\ and showed that there are no degenerate saddlepoints for one setting. 
Further, we empirically
demonstrated that alternating minimization results in global minima for an even broader class. 
Motivated by these insights, we then empirically investigated the practical estimation properties
of simple alternating minimization algorithms, particularly for incremental estimation. 
The main theoretical novelty in this work is to characterize the stationary points where the learned factors are full rank,
as opposed to the rank-deficient setting which has been more thoroughly characterized. 
 
 A key insight from this paper is that
that we can  
identify tractable dictionary learning objectives
using the class of induced \RFMs. For those dictionary learning methods that
can be specified as induced \RFMs, we provide compelling evidence that local minima are global minima and, in some cases, that there are no degenerate saddlepoints. Consequently, a simple alternating minimization scheme is effective for obtaining global solutions.
To be more precise, we show a general result for the optimality of stationary points for the induced form, and highlight a necessary condition---the induced regularization property---to prove a similar result for the factored form. We then show one setting that satisfies the induced regularization property, i.e., for the case of strictly convex regularizers and Holder continuous losses. These general results provide insights into specific settings of interest, 
showing that local minima are global minima for 
\begin{enumerate}[leftmargin=0cm,itemindent=0.5cm]
\item matrix completion using squared regularizers on the factors, with any invertible weighting on the squared regularizer on $\Dmat$, where the inner dimension $\ldim$ can be set less than $\xdim$ (Theorem \ref{thm_weighted})
\item sparse coding with a pseudo-Huber smooth approximation to the $\ell_1$ regularizer, for sufficiently large $\ldim$ (Corollary \ref{cor_en})
\item dictionary learning with the elastic net on either $\Dmat$ and $\Hmat$, where again the $\ell_1$ regularizer is approximated with the pseudo-Huber function and $\ldim$ is sufficiently large (Corollary \ref{cor_en}).
\end{enumerate}
The result is proved for the more general class of induced \RFMs,
as opposed to previous results which were proved for each specific setting. 
Therefore, though the results from this paper provide novel insights for
each of these specific settings, the larger contribution
is to provide a new proof strategy for characterizing stationary points
of dictionary learning methods. 

The generality of this result, however, is limited in two important ways. 
First, to characterize saddlepoints using directional derivatives, we have assumed that
the losses and regularizers are differentiable. Consequently,
the results do not directly apply to standard sparse regularizers, such as $\ell_1$.
Nonetheless, they do apply to smooth approximations, such as the pseudo-Huber smooth approximation.
The pseudo-Huber function has a parameter $\sigma$; when $\sigma \rightarrow 0$, the pseudo-Huber
approaches $\ell_1$. The results in this paper apply to arbitrarily small $\sigma$;
a limit argument should extend the results to these non-smooth regularizers. An important next step, therefore,
is to
formally characterize this intuition, so that these results directly apply to settings such as sparse coding.

A second limitation is that the results do not yet provide a precise specification for the inner dimension, $\ldim$. 
The proof requires $\ldim$ to be large enough, to ensure that $\regzk$ is convex.
We know that for $\ldim$ at least as large as the true induced dimension,
then $\regzk$ is equal to the induced regularizer $\regz$ which is guaranteed to be convex.
For smaller $\ldim$, however, $\regzk$ could still be convex,
or potentially have nice properties such as pseudo-convexity. 
The increase of $\ldim$ has a diminishing returns property,
and we expect $\regzk$ to behave similarly to $\regz$ before $\ldim$ actually reaches the magnitude 
of the true induced dimension.
In fact, we find that $\ldim$ can be quite a bit
smaller than the true induced dimension and alternating minimization still provides global solutions.
We hypothesize that this is in fact because $\regzk$ even for small $\ldim$ is already close to being a convex function (e.g., it could be pseudoconvex).
To provide an even more general result, and to characterize what we see in practice,
a useful direction would be to extend the proof to pseudo-convex induced regularizers
and to investigate specific choices of regularizers and $\ldim$ that produce pseudo-convex induced regularizers.
\bibliography{paper}

\newpage
\appendix

%

\section{Additional information about regularizers} \label{app_regularizers}

\begin{proposition} \label{prop_convex}
For any $m$ non-negative convex functions $g_i: \RR^\xdim \rightarrow \RR^+$,  
$g(\dvec) = \sqrt{g_1^2(\dvec) + \ldots g_m^2(\dvec)}$ is convex.
Further, if the $g_i$ are all norms, then $g$ is a norm. 
\end{proposition}
\begin{proof}
Notice that $g$ can be seen as the composition of the $\ell_2$ norm on the
vector of $m$ values of $g_i$ applied to $\dvec$. 
We can use the convexity properties of both to obtain the result.
Take any $\eta \in [0,1]$ and any $\avec, \bvec \in \RR^{\xdim}$.
\begin{align*}
g(\eta \avec + (1-\eta) \bvec) 
&= \ell_2( [g_1(\eta \avec + (1-\eta) \bvec), \ldots, g_m(\eta \avec + (1-\eta) \bvec)] )\\
&\le \ell_2( [\eta g_1(\avec) + (1-\eta) g_1(\bvec), \ldots, \eta g_m(\avec) + (1-\eta) g_m(\bvec)] )\\
&= \ell_2( \eta [g_1(\avec), \ldots, g_m(\avec)]  + (1-\eta) [ g_1(\bvec), \ldots, g_m(\bvec)] )\\
&\le \eta \ell_2( [g_1(\avec), \ldots, g_m(\avec)])  + (1-\eta) \ell_2([ g_1(\bvec), \ldots, g_m(\bvec)] )\\
&= \eta g(\avec) + (1-\eta) g(\bvec) 
\end{align*}
The first inequality follows from the fact that $\ell_2$ is monotonically increasing on the set of non-negative values
and the $g_i$ are all non-negative, so for any $\dvec \in \RR^\xdim$, $b \in \RR^+$ such that, 
$g_i(\dvec) \le b$ implies $g_i^2(\dvec) \le b^2$.
Therefore, $g$ is convex.

If the $g_i$ are norms, then clearly $g$ satisfies absolute homogeneity $g(\eta \dvec) = \eta g(\dvec)$
and $g(\dvec) = 0$ implies $\dvec = \zerovec$. The triangle inequality follows similarly to above, obtained by setting $\eta = 0.5$
and using the absolute homogeneity of $g$. 
\end{proof}

We can use this result to use separate norms on different parts of the columns or rows of $\Dmat$ and $\Hmat$. 
This is particularly useful for supervised learning, where $\Dmat$ is partitioned into two components, one for
the unsupervised recovery and the other for supervised learning. 
\begin{corollary}\label{cor_regularizer_separate}
For some $m < \xdim$ with two functions $g_1: \RR^m \rightarrow \RR^+$ and $g_2: \RR^{\xdim - m} \rightarrow \RR^+$ which act on the two partitions of $\dvec \in \RR^{\xdim}$, 
the regularizer $g(\dvec) = \sqrt{g_1^2(\dvec_{1:m}) + g_2^2(\dvec_{(m+1):\xdim})}$ is a convex, non-negative function in $\RR^\xdim$. 
\end{corollary}
\begin{proof}
We can extend both functions $g_1$ and $g_2$ to the domain $\RR^{\xdim}$, simply by having those parts not influence the function value.
These extended functions are still non-negative, and so satisfy the conditions of Proposition \ref{prop_convex}.
The function $g$ defined with the extended functions is exactly the same, and so we obtain the desired result. 
\end{proof}

\section{Proof of Global Optimality for Weighted Frobenius Norm}\label{app_weighted}

We prove the result for Theorem \ref{thm_weighted} in this section. We do so by proving the more general results in Theorem \ref{thm_weighted_general}, where in the main text we presented a simpler statement that highlighted the properties of saddlepoints. 

We consider a general optimization, with a row weighting $\Diagmat$ on $\Dmat$. 
\begin{align*}
\min_{\genrankset} \jointloss(\Dmat\Hmat) + \tfrac{\regwgt}{2} \frobsq{\Diagmat\Dmat} + \tfrac{\regwgt}{2 \hscale^2}  \frobsq{\Hmat} 
\end{align*}
Without loss of generality, we can adjust the diagonal weighting $\Diagmat$, and avoid the cluttered additional weighting
using $s^2$. This can be done by setting $\regwgt = \frac{\regwgt_{\text{original}}}{\hscale^2}$
and $\Diagmat = \hscale^2 \Diagmat_{\text{original}}$. We do not absorb $\regwgt$ into $\Diagmat$, because for simplicity of the proof,
we require an invertible $\Diagmat$, but allow $\regwgt = 0$ to cover the unregularized setting.

\begin{theorem} \label{thm_weighted_general}
For invertible $\Diagmat \in \RR^{\xdim \times \xdim}$ and $\regwgt \ge 0$, 
let $(\Dalt,\Halt)$ be a stationary point of
%
\begin{align}
\min_{\genrankset} \jointloss(\Dmat\Hmat) + \tfrac{\regwgt}{2} \frobsq{\Diagmat\Dmat} + \tfrac{\regwgt}{2}  \frobsq{\Hmat} 
\tag{\ref{eq_localweighted}}
\end{align}
%
%
If either 
\begin{enumerate}
\item $\| \Diagmat^{-\top}  \nabla  \jointloss(\Dalt \Halt) \|_2 \le \regwgt$
(i.e., the maximum singular value is $\regwgt$); or
\item the Hessian is positive semi-definite 
with $(\Dalt,\Halt)$ rank-deficient
\end{enumerate}
then $\joint = \Dalt \Halt$ is a globally optimal solution to 
\begin{align}
\argmin_{\joint \in \RR^{\xdim \times \nsamples}} \jointloss(\joint) + \regwgt \tracenorm{\Diagmat \joint}
. \label{eq_globalweighted}
\end{align}
\end{theorem}
\begin{proof}


Since $\Dalt$ and $\Halt$ are stationary points of \eqref{eq_localweighted}, we know
\begin{align}
\nabla_\Dmat = \nabla \jointloss(\Dalt \Halt) \Halt^\top + \regwgt \Diagmat^\top\Diagmat\Dalt = \zerovec \label{eq_Fmat}\\
\nabla_\Hmat = \Dalt^\top \nabla \jointloss(\Dalt \Halt) + \regwgt \Halt = \zerovec \label{eq_Hmat}
\end{align}
By multiplying both sides by $\Dalt^\top$ or $\Halt^\top$ respectively, and
taking the trace, we get
\begin{align}
\trace{\nabla \jointloss(\Dalt \Halt) \Halt^\top\Dalt^\top} + \regwgt \trace{\Diagmat^\top \Diagmat \Dalt \Dalt^\top} = \zerovec \label{eq_localone}\\
\trace{\Halt^\top\Dalt^\top \nabla \jointloss(\Dalt \Halt)} + \regwgt\trace{\Halt^\top \Halt} = \zerovec \label{eq_localtwo}
\end{align}
Now, we know that the trace norm can be characterized as follows \citep{recht2010guaranteed}
\begin{align*}
\tracenorm{\joint} &= \min_{\Wmat_1, \Wmat_2} \tfrac{1}{2}\trace{\Wmat_1} + \tfrac{1}{2} \trace{\Wmat_2} 
\text{ s.t. } \Wmat = 
\left(\begin{array}{cc} 
\Wmat_1 & \joint\\
\joint^\top & \Wmat_2
\end{array}\right) \succeq 0
\end{align*}
giving weighted norm
\begin{align*}
\tracenorm{\Diagmat\joint} &= \min_{\Wmat_1, \Wmat_2}  \tfrac{1}{2}\trace{\Wmat_1} + \tfrac{1}{2}  \trace{\Wmat_2}
\text{ s.t. } \Wmat = 
\left(\begin{array}{cc} 
\Wmat_1 & \Diagmat\joint\\
\joint^\top \Diagmat^\top & \Wmat_2
\end{array}\right) \succeq 0
\end{align*}

A Lagrange multiplier $\Mmat \in \RR^{(\xdim+\nsamples) \times (\xdim+\nsamples)}$ can be used to bring the constraint 
$\Wmat \succeq 0$ into the objective and consider the
Lagrangian with $\Mmat \succeq 0$:
\begin{align}
\Lg(\joint, \Wmat_1, \Wmat_2, \Mmat) 
&= \jointloss(\joint) +  
\tfrac{\regwgt}{2}\trace{\Wmat_1} + \tfrac{\regwgt}{2}  \trace{\Wmat_2} - 
\trace{\Mmat^\top\Wmat}
\label{eq_lagrange}
\end{align}

Notice first that, writing  $\Mmat$ in a block structure, 
\begin{align*}
\Wmat^\top \Mmat = 
\left(\begin{array}{cc} 
\Wmat_1^\top & \Diagmat\joint\\
\joint^\top\Diagmat^\top & \Wmat_2^\top
\end{array}\right) 
\left(\begin{array}{cc} 
\Mmat_1 & \Mmat_2\\
\Mmat_3  & \Mmat_4
\end{array}\right) 
=
\left(\begin{array}{cc} 
\Wmat_1^\top \Mmat_1 + \Diagmat \joint \Mmat_3 & \Wmat_1^\top \Mmat_2 + \Diagmat \joint \Mmat_4\\
\joint^\top \Diagmat^\top \Mmat_1 + \Wmat_2^\top \Mmat_3 & \joint^\top \Diagmat^\top \Mmat_2 + \Wmat_2^\top \Mmat_4\\
\end{array}\right) 
\end{align*}
we get
$\trace{\Mmat^\top\Wmat} 
=\trace{\Wmat^\top\Mmat}
=\trace{\Wmat_1^\top \Mmat_1} 
+ \trace{\Diagmat \joint \Mmat_3} 
+ \trace{\joint^\top \Diagmat^\top \Mmat_2} 
+ \trace{\Wmat_2^\top \Mmat_4} $.

The KKT conditions for this optimization require
\begin{align}
\textbf{Gradient conditions:}\nonumber\\
\nabla \Lg(\joint, \Wmat_1, \Wmat_2, \Mmat)  
= \nabla \jointloss(\joint) - \Diagmat^\top \Mmat_3^\top -  \Diagmat^\top \Mmat_2= \zerovec\\
\nabla_{\Wmat_1} \Lg(\joint, \Wmat_1, \Wmat_2, \Mmat)  
= \tfrac{\regwgt}{2} \eye_{\xdim} - \Mmat_1= \zerovec\\
\nabla_{\Wmat_2} \Lg(\joint, \Wmat_1, \Wmat_2, \Mmat)  
= \tfrac{\regwgt}{2} \eye_{\nsamples} - \Mmat_4= \zerovec\\
\textbf{Complementary slackness:}\nonumber\\
\trace{\Mmat^\top \Wmat} = 0\\
\textbf{Primal feasibility:}\nonumber\\
 \Wmat \succeq 0\\
 \textbf{Dual feasibility:}\nonumber\\
 \Mmat \succeq 0
\end{align}
We introduce candidate variables using the stationary points $\Dalt$
and $\Halt$ of the
nonconvex objective \eqref{eq_localweighted}:
$\joint =\Dalt \Halt$,
$\Wmat_1 =  \Diagmat \Dalt \Dalt^\top \Diagmat^\top$,
$\Wmat_2 = \Halt^\top \Halt$,
$\Mmat_1 = \tfrac{\regwgt}{2} \eye_{\xdim}$,
$\Mmat_4 = \tfrac{\regwgt}{2} \eye_{\nsamples}$,
$\Mmat_2 =  \tfrac{1}{2} \Diagmat^{-\top}\nabla \jointloss(\joint)$
and
$\Mmat_3= \Mmat_2^\top$.

These variables clearly satisfy all the gradient conditions.
Complementary slackness holds since
\begin{align*}
&\trace{\Mmat^\top\Wmat} 
=\trace{\Diagmat \Dalt \Dalt^\top \Diagmat^\top \tfrac{\regwgt}{2} \eye_{\xdim}} 
+ \trace{\Diagmat \Dalt \Halt \tfrac{1}{2} \nabla \jointloss(\joint)^\top \Diagmat^{-1}} + \\
&\hspace{3cm} \trace{ \Halt^\top \Dalt^\top \Diagmat^\top\tfrac{1}{2} \Diagmat^{-\top}\nabla \jointloss(\joint)} 
+ \trace{\Halt^\top \Halt \tfrac{\regwgt}{2} \eye_{\nsamples}} \\
&=\tfrac{\regwgt}{2} \trace{\Diagmat \Dalt \Dalt^\top \Diagmat^\top} 
+ \tfrac{1}{2}  \trace{ \Dalt \Halt \nabla \jointloss(\joint)^\top }  
+ \tfrac{1}{2} \trace{ \Halt^\top \Dalt^\top \nabla \jointloss(\joint)}
+ \tfrac{\regwgt}{2} \trace{\Halt^\top \Halt} \\
&= 0
\end{align*}
because of \eqref{eq_localone} and \eqref{eq_localtwo}
and because the
trace satisfies a circular property, $\trace{\Amat\Bmat\Cmat} = \trace{\Bmat\Cmat\Amat} = \trace{\Cmat\Amat\Bmat}$
and equality under transpose, $\trace{\Amat} = \trace{\Amat^\top}$.

To show primal feasibility, notice that
\begin{align*}
\Wmat = 
\left(\begin{array}{cc} 
\Diagmat \Dalt \Dalt^\top \Diagmat^\top & \Diagmat\Dalt \Halt\\
 \Halt^\top \Dalt^\top \Diagmat^\top & \Halt^\top \Halt
\end{array}\right)
=
\left(\begin{array}{c} 
\Diagmat \Dalt \\
 \Halt^\top 
\end{array}\right)
\left(\begin{array}{c} 
\Diagmat \Dalt \\
 \Halt^\top 
\end{array}\right)^\top
\succeq 0
\end{align*}
 
To show dual feasibility,
we can use the Schur complement condition of block
matrices: $\Mmat \succeq 0$ iff 
$\Mmat_1 \succ 0$ and
 $\Mmat_4-   \Mmat_3\Mmat_1^{-1} \Mmat_2 \succeq 0$.
Clearly, 
$\Mmat_1 = \tfrac{\regwgt}{2 } \eye_{\xdim} \succ 0$, since $\regwgt > 0$. 
Now,
 \begin{align*}
\Mmat_4-   \Mmat_3\Mmat_1^{-1} \Mmat_2 &= \tfrac{\regwgt}{2} \eye_{\nsamples} - \tfrac{1}{2} \nabla \jointloss(\joint)^\top \Diagmat^{-1}(\tfrac{\regwgt}{2} \eye_{\xdim})^{-1} \tfrac{1}{2} \Diagmat^{-\top}\nabla \jointloss(\joint) \\
 &= \tfrac{\regwgt}{2} \eye_{\nsamples} - \tfrac{1}{2\regwgt} \nabla \jointloss(\joint)^\top \Diagmat^{-1} \Diagmat^{-\top}\nabla \jointloss(\joint)
 \end{align*}

 \paragraph{Case 1: general stationary point.}  For $\regwgt > 0$, if $\| \Diagmat^\inv \nabla \jointloss(\joint) \|_2 \le \regwgt$, then
 %
  \begin{align*}
 & \nabla  \jointloss(\joint)^\top \Diagmat^{-1} \Diagmat^{-\top}\nabla  \jointloss(\joint)
 \preceq \regwgt^2 \eye\\
&\implies \regwgt^2 \eye_{\nsamples} -  \nabla  \jointloss(\joint)^\top \Diagmat^{-1} \Diagmat^{-\top}\nabla  \jointloss(\joint) \succeq 0\\
&\implies \tfrac{\regwgt}{2} \eye_{\nsamples} -  \tfrac{1}{2\regwgt} \nabla  \jointloss(\joint)^\top \Diagmat^{-1} \Diagmat^{-\top}\nabla  \jointloss(\joint) \succeq 0
.
 \end{align*}
 satisfying the dual feasibility requirement for this case.
 
 If $\regwgt = 0$, then $\| \Diagmat^\inv \nabla  \jointloss(\joint) \|_2 \le 0$,
 and so $\nabla  \jointloss(\joint) = \zerovec$, again indicating $\joint = \Dalt \Halt$ is a global 
 optimum.

  \paragraph{Case 2: rank-deficient stationary point has a positive semi-definite Hessian.}\mbox{}\\
%
  We will take advantage of the fact that the bottom $\ldim  - \rdim$ singular values of $\Dalt$ and $\Halt$ are zero.
 %
  To conveniently write the Hessian
  for matrix-valued variables, we will use the directional derivative
  \begin{align*}
  D_\pairvar \closs(\pairvar_0)[d \pairvar] = \lim_{t \rightarrow 0} \frac{\closs(\pairvar_0 + t d\pairvar) - \closs(\pairvar_0)}{t} = \langle \nabla_\pairvar \closs(\pairvar_0), d \pairvar \rangle
  \end{align*}
  where $\langle \Amat, \Bmat\rangle = \tr(\Amat^\top \Bmat)$. 
  For the Hessian, we get
  \begin{align*}
  D^2_\pairvar \closs(\pairvar_0)[d \pairvar,d \pairvar] = D_\pairvar \langle \nabla_\pairvar \closs(\pairvar_0), d \pairvar \rangle [d \pairvar]
  = \langle \nabla_\pairvar \langle \nabla_\pairvar \closs(\pairvar_0), d \pairvar \rangle,  d \pairvar \rangle
  .
  \end{align*}
  In our setting, $\pairvar_0 = [\Dalt; \Halt^\top] \in \RR^{(\xdim+\nsamples) \times \ldim}$ and
  $\closs([\Dalt;\Halt^\top]) = \jointloss(\Dalt \Halt) + \tfrac{\regwgt}{2} \frobsq{\Diagmat\Dalt} + \tfrac{\regwgt}{2}  \frobsq{\Halt}$.
  Let $\Gmat = \nabla  \jointloss(\joint)$ for $\joint = \Dalt \Halt$, giving
    \begin{align*}
  D_\pairvar \closs(\pairvar_0)[d \pairvar] 
  &= \left\langle \inlinevec{\Gmat\Halt^\top + \regwgt \Diagmat^\top \Diagmat \Dalt}{ \Gmat^\top \Dalt + \regwgt \Halt^\top}, \inlinevec{d \Dmat}{ d \Hmat^\top} \right\rangle \\
  &= \tr( \Halt \Gmat^\top  d \Dmat) 
  + \regwgt \tr( \Dalt^\top \Diagmat^\top \Diagmat d\Dmat) 
  + \tr( \Dalt^\top \Gmat d\Hmat^\top) 
  + \regwgt \tr(\Halt d\Hmat^\top)
  .
  \end{align*}
Similarly, for the Hessian, 
    \begin{align*}
  D^2_\pairvar \closs(\pairvar_0)[d \pairvar,d \pairvar] 
  &= \left\langle \nabla_\pairvar \left(\tr( \Halt \Gmat^\top  d \Dmat) + \regwgt \tr( \Dalt^\top \Diagmat^\top \Diagmat d\Dmat) + 
  \tr( \Dalt^\top \Gmat d\Hmat^\top) + \regwgt \tr(\Halt d\Hmat^\top)\right),  d \pairvar \right\rangle\\
  &= \left\langle  \inlinevec{\regwgt \Diagmat^\top \Diagmat d\Dmat + \Gmat d\Hmat^\top}{ \Gmat^\top d \Dmat + \regwgt d\Hmat^\top} ,  \inlinevec{d \Dmat}{ d \Hmat^\top} \right\rangle  
  + D^2_\joint \jointloss(\joint)[\Dalt d\Hmat + d\Dmat\Halt,\Dalt d\Hmat + d\Dmat\Halt]\\
    &= 2\tr(d\Dmat^\top \Gmat d\Hmat^\top) + \regwgt \tr(d\Dmat^\top\Diagmat^\top \Diagmat d\Dmat) + \regwgt \tr(d\Hmat^\top d\Hmat) \\
  &\hspace{1.0cm}+ D^2_\joint \jointloss(\joint)[\Dalt d\Hmat + d\Dmat\Halt,\Dalt d\Hmat + d\Dmat\Halt]
  .
  \end{align*}
Now we can make the last term, the Hessian w.r.t. $\jointloss$, disappear by carefully choosing
$d\Hmat$ and $d\Dmat$.
Let $\Dmat = \Umat_1 \Sigmamat_1 \Vmat_1^\top$ and $\Hmat = \Umat_2 \Sigmamat_2 \Vmat_2^\top$. 
Take 
\begin{align*}
d\Hmat &= \Vmat_1 [\zerovec, \ldots, \zerovec, \hvec_1, \ldots, \hvec_{\ldim - \rdim}]^\top\\
d \Dmat &= [\zerovec, \ldots, \zerovec, \dvec_1, \ldots, \dvec_{\ldim - \rdim}] \Umat_2^\top
\end{align*}
for any vectors $\hvec_i, \dvec_i$. 
The rotation in $d\Hmat$ results in 
$\Dalt d\Hmat = \Umat \Sigmamat_\rdim [\zerovec, \ldots, \zerovec, \hvec_1, \ldots, \hvec_{\xdim - \ldim}]^\top = \zerovec$.
Similarly, $d \Dmat \Halt = \zerovec$. Therefore, $D^2_\joint \jointloss(\joint)[\Dalt d\Hmat^\top + d\Dmat\Halt^\top,\Dalt d\Hmat^\top + d\Dmat\Halt^\top] = D^2_\joint \jointloss(\joint)[\zerovec,\zerovec] = \zerovec$.

Now, for the first term, because $\hvec_i$ and $\dvec_i$ are arbitrary vectors in the last $\ldim-\rdim$
columns , and because we have a positive
semi-definite Hessian, we get the inequality for any $\hvec \in \RR^\nsamples$, $\dvec \in \RR^\xdim$
    \begin{align*}
2\dvec^\top \Gmat \hvec + \regwgt \dvec^\top\Diagmat^\top \Diagmat \dvec + \regwgt \hvec^\top \hvec \ge 0
.
  \end{align*}
  Since this is true for any $\hvec$ and $\dvec$, we can choose
  $\dvec = -\Diagmat^{-1} \Diagmat^{-\top} \Gmat \vvec$ and $\hvec = \regwgt \vvec$ for any $\vvec \in \RR^\nsamples$.
  This gives
    \begin{align*}
  0 &\le \dvec^\top \Gmat \hvec + \tfrac{\regwgt}{2} \dvec^\top\Diagmat^\top \Diagmat \dvec + \regwgt \hvec^\top \hvec \hspace{2.0cm} \triangleright \Diagmat \dvec =  -\Diagmat \Diagmat^{-1} \Diagmat^{-\top} \Gmat \vvec = - \Diagmat^{-\top} \Gmat \vvec\\
&= - \regwgt \vvec^\top \Gmat^\top \Diagmat^{-1} \Diagmat^{-\top} \Gmat \vvec + \tfrac{\regwgt}{2}\vvec^\top \Gmat^\top \Diagmat^{-1} \Diagmat^{-\top} \Gmat \vvec + \tfrac{\regwgt^3}{2} \vvec^\top \vvec \\
&= - \tfrac{\regwgt}{2} \vvec^\top \Gmat^\top \Diagmat^{-1} \Diagmat^{-\top} \Gmat \vvec + \tfrac{\regwgt^3}{2} \vvec^\top \vvec \\
&= \vvec^\top ( \tfrac{\regwgt^3}{2} \eye - \tfrac{\regwgt}{2}\Gmat^\top \Diagmat^{-1} \Diagmat^{-\top} \Gmat ) \vvec
  \end{align*}
  for all $\vvec$, which implies 
  $\tfrac{\regwgt^3}{2} \eye - \tfrac{\regwgt}{2}\Gmat^\top \Diagmat^{-1} \Diagmat^{-\top} \Gmat \succeq 0$
  (by the definition of positive definite matrices). 
  If $\regwgt > 0$, then 
  by dividing both sides by $\regwgt^2$ and recalling
  that $\Gmat = \nabla  \jointloss(\joint)$,
    \begin{align*}
 \tfrac{\regwgt}{2} \eye_{\nsamples} -  \tfrac{1}{2\regwgt} \nabla  \jointloss(\joint)^\top \Diagmat^{-1} \Diagmat^{-\top}\nabla  \jointloss(\joint) \succeq 0
 \end{align*}
which completes the proof for dual feasibility for the second case.

If $\regwgt = 0$, then above we can simply choose 
  $\dvec = -\Diagmat^{-1} \Diagmat^{-\top} \Gmat \vvec$ and $\hvec = \vvec$ for any $\vvec \in \RR^\nsamples$,
  to obtain $\nabla  \jointloss(\joint) = \zerovec$, again indicating $\joint = \Dalt \Halt$ is a global 
 optimum for this second case. 
\end{proof}

\section{Proof of local strong convexity for regularizers with full rank variables}\label{app_strongly}


The below lemma is stated for $\regd$, but similarly applies to $\regh$ for full rank $\Halt$. 
\begin{lemma}\label{lem_strongly} 
For $\regd(\Dalt) = \sum_{i=1}^\ldim \fcol^2(\Dalt_{:i})$ for $\fcol$ a strictly convex, non-negative centered function
and $\Dalt \in \RR^{\xdim \times \ldim}$ full rank with no columns of $\Dalt$ strictly zero, then there exists $m_d > 0$ such that
\begin{align}
\regd((1-t)\Dalt + t\Dmat) &\le
(1-t)\regd(\Dalt) + t\regd(\Dmat) - \frac{m_d}{2} t (1-t) \| \Dalt - \Dmat\|_F^2 \label{eq_sc_proof}
\end{align}
for any $\Dmat$ and for a sufficiently small $t$.
\end{lemma}
\begin{proof}
Let $\epsilon > 0$ be such that each column of $\Dalt$ is outside of an $\epsilon$-ball around zero.
Such an $\epsilon$ exists because no column of $\Dalt$ is zero. 
We need to characterize the local strong convexity of $\regd$ around $\Dalt$. 

For non-negative, centered, strictly convex
functions $\fcol$, squaring can only strictly increase the curvature once sufficiently far from zero. To see why, recall that for the standard definition of strong convexity for smooth functions, $m$ is a lower bound on the eigenvalues of the Hessian at every point, and so reflects a minimal curvature of the function. For non-negative, centered, strictly convex
functions, squaring can decrease the curvature near zero; once outside a sufficiently large ball around zero, the function becomes larger than $1$ and so squaring can only strictly increase the curvature. For those values inside the ball, as long as they are constrained to be outside a small $\epsilon$-ball around $\zerovec$, we can still ensure that the curvature remains non-negligible. 

We know there exists $m_i$ and a sufficiently small $t$ (to remain local to $\Dalt_{:i}$) such that the function satisfies
\begin{align*}
\fcol^2((1-t)\Dalt_{:i} + t\Dmat_{:i}) &\le
(1-t)\fcol^2(\Dalt_{:i}) + t\fcol^2(\Dmat_{:i}) - \frac{m_i}{2} t (1-t) \| \Dalt_{:i} - \Dmat_{:i}\|_2^2
\end{align*}
Taking $m_d = \min_i m_i$, we get
\begin{align*}
\regd((1-t)\Dalt + t \Dmat)) &= 
\sum_{i=1}^\ldim \fcol^2((1-t)\Dalt_{:i} + t\Dmat_{:i}) \\
&\le
\sum_{i=1}^\ldim  \left[(1-t)\fcol^2(\Dalt_{:i}) + t\fcol^2(\Dmat_{:i}) - \frac{m_i}{2} t (1-t) \| \Dalt_{:i} - \Dmat_{:i}\|_2^2 \right]\\
&=
(1-t) \regd(\Dalt) + t \regd(\Dmat) - t (1-t) \sum_{i=1}^\ldim  \frac{m_i}{2}  \| \Dalt_{:i} - \Dmat_{:i}\|_2^2 \\
&\le
(1-t) \regd(\Dalt) + t \regd(\Dmat) - \frac{m_d}{2} t (1-t) \sum_{i=1}^\ldim \| \Dalt_{:i} - \Dmat_{:i}\|_2^2 \\
&\le
(1-t) \regd(\Dalt) + t \regd(\Dmat) - \frac{m_d}{2} t (1-t)  \| \Dalt - \Dmat\|_F^2
\end{align*}
%

\end{proof}

\section{Proof of induced regularization property for an additional setting}\label{app_prop}
In the main text, we demonstrated that for induced \RFMs\ with strictly convex regularizers and a Holder continuous loss, full-rank local minima satisfy the induced regularization property. We show an additional setting that satisfies the induced regularization property,
in Proposition \ref{prop_frob}. We first provide a lemma needed for that proof, and then proof the proposition. 

\begin{lemma}\label{lem_qconvex}
For any convex function $\regh: \RR^{\ldim \times \nsamples}$ that is separable across rows $\regh(\Hmat) = \sum_{i=1}^\ldim f_r(\Hmat_{i:})$
with non-negative convex function $f_r: \RR^{\nsamples} \rightarrow \RR^+$,
the following function is convex in $\Qmat \in \RR^{\ldim \times \ldim}$.
\begin{align*}
\| \Dmat \Qmat^{-1} \|_F^2 + \regh(\Qmat\Hmat)
\end{align*}
\end{lemma}
\begin{proof}
We will use directional derivatives. Let $\Qmat(t) = \Qmat + t \Bmat$,
where we step some amount in the direction of $\Bmat$, which need not
be a full rank matrix. 
Then, using the Taylor series expansion, we get
\begin{align*}
\Qmat(t)^\inv &= (\Qmat (\eye + t \Qmat^\inv \Bmat))^\inv\\
&= \Qmat^\inv - t \Qmat^\inv \Bmat \Qmat^\inv + t^2 \Qmat^\inv \Bmat \Qmat^\inv \Bmat \Qmat^\inv - \ldots
\end{align*}
giving
\begin{align*}
\Qmat(t)^\inv \Qmat^\tinv 
&= \Qmat^\inv\Qmat^\tinv - t \Qmat^\inv \Bmat \Qmat^\inv \Qmat^\tinv + t^2 \Qmat^\inv \Bmat \Qmat^\inv \Bmat \Qmat^\inv  \Qmat^\tinv + \ldots\\
&+  (-t)\Qmat^\inv \Qmat^\tinv \Bmat^\top \Qmat^\tinv + t^2 \Qmat^\inv \Bmat \Qmat^\inv \Qmat^\tinv \Bmat^\top \Qmat^\tinv - \ldots\\
&+ t^2\Qmat^\inv \Qmat^\tinv \Bmat^\top \Qmat^\tinv \Bmat^\top \Qmat^\tinv + \ldots
\end{align*}
Now using the fact that
\begin{align*}
 \frobsq{\Dmat \Qmat(t)^{\inv}} &= \trace{\Qmat(t)^\tinv \Dmat^\top \Dmat \Qmat(t)^\inv}\\
 &= \trace{\Qmat(t)^\inv\Qmat(t)^\tinv \Dmat^\top \Dmat}
\end{align*}
we can take the second derivative with respect to $t$, and letting $t \rightarrow 0$, to obtain
\begin{align*}
\frac{d^2}{dt^2} \frobsq{\Dmat \Qmat(t)^{\inv}} |_{t=0}
 &= \frac{d^2}{dt^2} \trace{\Qmat(t)^\inv\Qmat(t)^\tinv \Dmat^\top \Dmat} |_{t=0}\\ 
 &= 2 \trace{\Qmat^\inv \Bmat \Qmat^\inv \Bmat \Qmat^\inv  \Qmat^\tinv \Dmat^\top \Dmat}\\
 &+ 2 \trace{\Qmat^\inv \Bmat \Qmat^\inv \Qmat^\tinv \Bmat^\top \Qmat^\tinv\Dmat^\top \Dmat}\\
 &+ 2 \trace{\Qmat^\inv \Qmat^\tinv \Bmat^\top \Qmat^\tinv \Bmat^\top \Qmat^\tinv\Dmat^\top \Dmat}
\end{align*}
Now if this sum is greater than or equal to zero, because this was for any full rank $\Qmat$,
then we know that $\frobsq{\Dmat \Qmat(t)^{\inv}}$ is convex over the set $\Qmat \in \RR^{\ldim \times \ldim}: \Qmat \text{ full rank}$.

First, the above traces are all on products of positive semi-definite matrices. 
The product of symmetric positive semi-definite matrices is again positive semi-definite
if and only if that product is symmetric \cite{meenakshi1999onaproduct}.
Therefore, if the above products are symmetric, then the trace is
on positive semi-definite matrices and so must be non-negative. 
Using the circular property of trace and re-ordering the last two traces, we can rewrite the above as
\begin{align}
& \trace{\Dmat\Qmat^\inv \Bmat \Qmat^\inv \Bmat \Qmat^\inv  \Qmat^\tinv \Dmat^\top}
+ \trace{\Dmat\Qmat^\inv \Qmat^\tinv \Bmat^\top \Qmat^\tinv \Bmat^\top \Qmat^\tinv\Dmat^\top}\nonumber\\
&\hspace{1.0cm} +  \trace{\Dmat\Qmat^\inv \Bmat \Qmat^\inv \Qmat^\tinv \Bmat^\top \Qmat^\tinv\Dmat^\top}\nonumber\\
 &=\trace{\Dmat\Qmat^\inv (\Bmat \Qmat^\inv \Bmat \Qmat^\inv + \Qmat^\tinv \Bmat^\top \Qmat^\tinv \Bmat^\top) \Qmat^\tinv \Dmat^\top}\label{eq_matone}\\
&\hspace{1.0cm}  +  \trace{\Dmat\Qmat^\inv \Bmat \Qmat^\inv \Qmat^\tinv \Bmat^\top \Qmat^\tinv\Dmat^\top} 
.
\label{eq_mattwo}
\end{align}
Recall that for any matrix $\Amat$, the inner product $\Amat^\top \Amat$ and outer product $\Amat\Amat^\top$ are both symmetric.
The matrix in \eqref{eq_matone} is symmetric, because the addition $\Bmat \Qmat^\inv \Bmat \Qmat^\inv + \Qmat^\tinv \Bmat^\top \Qmat^\tinv \Bmat^\top$
is symmetric and diagonalizable. The matrix in \eqref{eq_mattwo} also clearly has this property.
Therefore, these two traces are non-negative, and so the result follows. 
\end{proof}

\begin{proposition}[Induced regularization property for $\regd = \| \cdot \|_F^2$]\label{prop_frob}
Assume that 
\begin{enumerate}
\item the column regularizer $\fcol = \ell_2$, i.e., $\regd(\Dmat) = \sum_{i=1}^k \|\Dmat_{:i} \|_2^2$; 
\item $\regh(\Hmat) = \sum_{i=1}^k \frow^2(\Hmat_{i:} )$ is strictly convex;
\item $\Bmatopt_1 = \zerovec$, from Equation \eqref{eq_nullspace}.
\end{enumerate}
If 
 \begin{enumerate}[itemindent=0.5cm]
 \item $(\Dalt,\Halt)$ is a local minimum of \eqref{eq_generic}
 \item $\Dalt$ and $\Halt$ are full rank
 \end{enumerate}
then $(\Dalt, \Halt)$ satisfies the induced regularization property.
\end{proposition}
\begin{proof}
The same argument applies as in Proposition \ref{prop_strongly}
and we can rewrite the optimization in terms of $\Qmat, \Bmat$. 
If $\Qmatopt, \Bmatopt_1$ strictly decreases the regularization component, 
we know $\regd(\Dalt \Qmatopt^\inv) + \regh(\Qmatopt \Halt + \Qmatopt\Bmatopt_1\Hmat) < \regd(\Dalt) + \regh(\Halt)$.
We will show that this leads to a contradiction, and so $\Qmatopt = \eye, \Bmatopt_1 = \zerovec$ must be a solution. 

For a sufficiently small $t > 0$, $(1-t) \eye + t \Qmatopt$ is invertible. 
Because $\regd(\Dalt \Qmat^\inv)$ is locally convex in $\Qmat$, around $\Qmat = \eye$ by Lemma \ref{lem_qconvex}, 
we have that $\regd(\Dalt ((1-t) \eye + t \Qmatopt)^\inv) \le (1-t) \regd(\Dalt) + t \regd(\Dalt \Qmatopt^\inv)$. 

We can define a small step $t$ in a direction from $\Dalt, \Halt$ to get new points
\begin{align*}
&((1-t) \eye + t \Qmatopt + t \Qmatopt\Bmatopt_1) \Halt\\
&\Dalt ((1-t) \eye + t \Qmatopt)^\inv 
\end{align*}
where the product of these new points equals  
\begin{align*}
\Dalt ((1-t) \eye + t \Qmatopt)^\inv ((1-t) \eye + t \Qmatopt + t \Qmatopt\Bmatopt_1) \Halt
&= \Dalt \Halt + t \Dalt ((1-t) \eye + t \Qmatopt)^\inv \Qmatopt\Bmatopt_1 \Halt
.
\end{align*}
We can explicitly write these new points as small steps from $\Dalt, \Halt$, for $\Halt$ with some rearrangement
\begin{align}
((1-t) \eye + t \Qmatopt + t \Bmatopt_1) \Halt &=  (1-t) \Halt + t \Qmatopt \Halt + t\Qmatopt\Bmatopt_1 \Halt \nonumber\\
&= \Halt + t (\Qmatopt + \Qmatopt\Bmatopt_1 - \eye)\Halt \label{eq_descenth}
\end{align}
and for $\Dalt$, using the Woodbury identity, 
\begin{equation}
\Dalt ((1-t) \eye + t \Qmatopt)^\inv = \Dalt + t \Dalt \left(\frac{1}{1-t} \eye - \frac{1}{(1-t)^2} \left(\Qmatopt^\inv + t(1-t)^\inv \eye\right)^\inv \right)
. \label{eq_descentd}
\end{equation}
%
We first examine the impact of stepping in these directions on the regularizers
\begin{align*}
\regd(\Dalt ((1-t) \eye &+ t \Qmatopt)^\inv) + \regh(((1-t) \eye + t (\Qmatopt + \Qmatopt\Bmatopt_1)) \Halt) \\
&\le (1-t) \regd(\Dalt) + t \regd(\Dalt \Qmatopt^\inv) +  (1-t)\regh(\Halt) + t \regh((\Qmatopt + \Qmatopt\Bmatopt_1)\Halt) \\
&= (1-t) \left(\regd(\Dalt) + \regh(\Halt) \right)+  t \left(\regd(\Dalt \Qmatopt^\inv) + \regh(\Qmatopt \Halt + \Qmatopt\Bmatopt_1 \Halt) \right)\\
&< \regd(\Dalt)  + \regh(\Halt) 
\end{align*}
where the first inequality also follows because $\regh( \Bmat \Halt)$ is convex in $\Bmat$
and the last inequality from the fact that $\regd(\Dalt \Qmatopt^\inv) + \regh(\Qmatopt \Halt + \Qmatopt\Bmatopt_1 \Halt) < \regd(\Dalt) + \regh(\Halt)$. 

Now, because $\Bmatopt_1 = \zerovec$,
\begin{align*}
\Dalt ((1-t) \eye + t \Qmatopt)^\inv ((1-t) \eye + t \Qmatopt + t \Qmatopt\Bmatopt_1) \Halt
&= \Dalt \Halt
\end{align*}
and so the loss value is unchanged. Clearly, therefore, stepping in these directions  in \eqref{eq_descenth} and \eqref{eq_descentd}  strictly decreases the regularizers, but does not modify the loss, and so would contradict the fact that $\Dalt, \Halt$ is a local minimum.

Therefore, $\Qmat = \eye$ and $\Bmatopt_1 = \zerovec$ is one solution to \eqref{eq_frobform}.  
It is possible that there could be other $\Qmatopt$ that do not decrease the regularizer. By the same argument as in Proposition \ref{prop_strongly}, these $\Qmat$ would also give stationary points.
 
 Therefore, $\Dalt, \Halt$ clearly satisfies the induced regularization property, because $\regd(\Dalt) + \regh(\Halt)$ cannot
 be improved by selecting different $\Dmat_k, \Hmat_k$ that satisfy $\Dmat_k \Hmat_k = \Dalt \Halt$
 and all possible solutions are stationary points. 
\end{proof}


\section{Proof of Proposition \ref{prop_equivalence}}\label{app_prop_equivalence}

\noindent
\textbf{Proposition \ref{prop_equivalence} }
The stationary points are equivalent for the summed form
\begin{align*}
\jointloss(\Dmat \Hmat) + \frac{\alpha}{2} \sum_{i=1}^\ldim \| \Dmat_{:i} \|_c^2 + \frac{\alpha}{2} \sum_{i=1}^\ldim \| \Hmat_{i:} \|_r^2
\end{align*}
and the producted form
\begin{align*}
\jointloss(\Dmat \Hmat) + \alpha \sum_{i=1}^\ldim \| \Dmat_{:i} \|_c \| \Hmat_{i:} \|_c
\end{align*}
The stationary points of the summed and producted forms 
are also stationary points of the constrained form.
\begin{proof}
 To simplify notation, let 
 \begin{align*}
 \hstack(\dvec_i) &= [\dvec_1, \ldots, \dvec_\ldim]\\
\vstack(\hvec_i) &= [\hvec_1; \ldots; \dvec_\ldim]
 \end{align*}
 where this notation will always implicitly be from 1 to $\ldim$. 
The stationary points 
$\Ds, \Hs$ of the summed objective satisfy
\begin{align*}
\nabla \jointloss(\Ds \Hs) \Hs^\top + \alpha \hstack( \| \Ds_{:i} \|_c \nabla \| \Ds_{:i} \|_c) = \zerovec\\
\Ds^\top \nabla \jointloss(\Ds \Hs) + \alpha \vstack(\| \Hs_{i:} \|_r \nabla \| \Hs_{i:} \|_r) = \zerovec
\end{align*}
the stationary points 
$\Dp, \Hp$ of the producted objective satisfy
\begin{align*}
\nabla \jointloss(\Dp \Hp) \Hp^\top + \alpha \hstack(  \| \Hp_{i:} \|_r \nabla \| \Dp_{:i} \|_c) = \zerovec\\
\Dp^\top \nabla \jointloss(\Dp \Hp) + \alpha \vstack( \| \Dp_{:i} \|_c \nabla \| \Hp_{i:} \|_r) = \zerovec
\end{align*}

 We can reweight $\Ds, \Hs$ to get a stationary point to the producted form, and vice-versa.
 Let $\Gammamat = \diag( \| \Ds_{:1} \|_c^{-1} \| \Hs_{1:} \|_r, \ldots,  \| \Ds_{:\ldim} \|_c^{-1} \| \Hs_{\ldim:} \|_r)$.
We require $\Gammamat$ to be invertible. This will always be true
unless a row or column are zero in $\Hmat$ and $\Dmat$; in this case, we can always
simply remove these.

 For any stationary point $\Ds, \Hs$, 
 let $\Dp = \Ds \Gammamat^{-1}, \Hp = \Gammamat \Hs$.
 Then $\Dp \Hp = \Ds \Hs$ and
 \begin{align*}
&\nabla \jointloss(\Dp \Hp) \Hp^\top + \alpha \hstack( \nabla \| \Dp_{:i} \|_c \| \Hp_{i:} \|_r)  \\
&=
\nabla \jointloss(\Ds\Hs) \Hs^\top \Gammamat + \alpha \hstack( \nabla \| \Ds_{:i} \|_c \Gammamat_i^{-1} \Gammamat_i \| \Hs_{i:} \|_r ) 
\end{align*}
giving
\begin{align*}
&\nabla \jointloss(\Ds\Hs) \Hs^\top \Gammamat \Gammamat^{-1} + \alpha \hstack( \nabla \| \Ds_{:i} \|_c \| \Hs_{i:} \|_r )  \Gammamat^{-1}  \\
&=
\nabla \jointloss(\Ds\Hs) \Hs^\top + \alpha \hstack(\nabla \| \Ds_{:i} \|_c \| \Hs_{i:} \|_r \Gammamat_i^{-1}) \\
&=
\nabla \jointloss(\Ds\Hs) \Hs^\top + \alpha \hstack(\nabla \| \Ds_{:i} \|_c \| \Ds_{:i} \|_c)  \\
&= \zerovec 
\end{align*}
because $\Ds$ is a stationary points of the summed form.
Therefore, because $\Gammamat$ is invertible, this means that
$\nabla \jointloss(\Dp \Hp) \Hp^\top + \alpha \hstack( \nabla \| \Dp_{:i} \|_c \| \Hp_{i:} \|_r) = 0$,
and so $\Dp$ is a stationary point of the producted form.
For $\Hp$,  
 \begin{align*}
&\Dp^\top \nabla \jointloss(\Dp \Hp) + \alpha \vstack( \| \Dp_{:i} \|_c \nabla \| \Hp_{i:} \|_r)  \\
&=
 \Gammamat^{-1} \Ds^\top \nabla \jointloss(\Ds\Hs)  + \alpha \vstack(  \| \Ds_{:i} \|_c \Gammamat_i^{-1} \Gammamat_i \nabla \| \Hs_{i:} \|_r ) 
\end{align*}
giving
\begin{align*}
&\Gammamat \Gammamat^{-1} \Ds^\top \nabla \jointloss(\Ds\Hs)  + \alpha \Gammamat \vstack( \| \Ds_{:i} \|_c  \nabla\| \Hs_{i:} \|_r )  \\
&=
\Ds^\top \nabla \jointloss(\Ds\Hs)  + \alpha \vstack( \Gammamat_i \| \Ds_{:i} \|_c \nabla\| \Hs_{i:} \|_r ) \\
&=
\Ds^\top \nabla \jointloss(\Ds\Hs) + \alpha \vstack(  \| \Hs_{i:} \|_r \nabla\| \Hs_{i:} \|_r)  \\
&= \zerovec 
\end{align*}
%
%
 
 \noindent
 In the other direction, 
  for any stationary point $\Dp, \Hp$, 
   let 
   $$\Gammamat = \diag( \| \Dp_{:1} \|_c^{-1} \| \Hp_{1:} \|_r, \ldots,  \| \Dp_{:\ldim} \|_c^{-1} \| \Hp_{\ldim:} \|_r)$$
and
 let $\Ds = \Dp \Gammamat^{1/3}, \Hs = \Gammamat^{-1/3} \Hp$.
 Then $\Ds \Hs = \Dp \Hp$ and
 \begin{align*}
&\nabla \jointloss(\Ds \Hs) \Hs^\top + \alpha \hstack( \| \Ds_{:i} \|_c \nabla \| \Ds_{:i} \|_c)\\  
&=
\nabla \jointloss(\Dp\Hp) \Hp^\top \Gammamat^{-1/3} + \alpha \hstack( \| \Dp_{:i} \|_c \Gammamat_i^{1/3} \Gammamat_i^{1/3} \nabla \| \Dp_{:i} \|_c) \\
&=
\nabla \jointloss(\Dp\Hp) \Hp^\top \Gammamat^{-1/3} + \alpha \hstack( \| \Dp_{:i} \|_c \nabla \| \Dp_{:i} \|_c)  \Gammamat^{2/3}
\end{align*}
giving
\begin{align*}
&\nabla \jointloss(\Dp\Hp) \Hp^\top \Gammamat^{-1/3} \Gammamat^{1/3} + \alpha \hstack( \| \Dp_{:i} \|_c \nabla \| \Dp_{:i} \|_c)  \Gammamat_i^{2/3} \Gammamat^{1/3}\\
&=
\nabla \jointloss(\Dp\Hp) \Hp^\top  + \alpha \hstack( \| \Dp_{:i} \|_c \nabla \| \Dp_{:i} \|_c)  \Gammamat\\
&=
\nabla \jointloss(\Dp\Hp) \Hp^\top  + \alpha \hstack( \| \Dp_{:i} \|_c \nabla \| \Dp_{:i} \|_c \| \Dp_{:i} \|_c^{-1} \| \Hp_{i:} \|_r)  \\
&=
\nabla \jointloss(\Dp\Hp) \Hp^\top  + \alpha \hstack( \nabla \| \Dp_{:i} \|_c \| \Hp_{i:} \|_r) \\
&= \zerovec 
\end{align*}
and so $\Ds$ is a stationary point of the summed form,
again assuming $\Gammamat$ is full rank.
This is similarly the case for $\Hs$. 

For the constrained form, to satisfy the KKT conditions, the Lagrange multipliers can be specified
to reweight the solution similarly. Because there is a separate Lagrange multiplier for each constraint $\| \Dmat_{:i} \|_c \le 1$,
there are enough degrees of freedom to easily reweight a stationary from the producted form to produce
a stationary point that satisfies the KKT conditions. 
 \end{proof}

\section{Proximal operator for the elastic net}\label{app_prox}

\begin{proposition} \label{prop_prox}
For a given $\uvec = [u_1, \ldots, u_\ldim] \in \RR^{\ldim}$, let the indices $m_1, \ldots, m_\ldim$ indicate a sorted descending order such that $\left| u_{m_1}\right| \geq \left|u_{m_2}\right|\geq ... \geq \left|u_{m_\ldim}\right|$. Let $C(\uvec, \sparsity) = \sum_{j=1}^\sparsity |u_{m_j}|$. 
Then the following $\zvec^*$ is an optimal solution of 
$\min\limits_{\zvec \in \RR^{\ldim}} \tfrac{1}{2}\|\uvec-\zvec\|_2^2+\lambda \|\zvec\|_1^2$
\begin{align*}
z^*_{i} = \sign{u_{i}}\max\left(\left|u_i\right|-\tfrac{2\lambda C(\uvec, \sparsity^*) }{1+2\lambda \sparsity^*},0\right)
\end{align*}
where
\begin{align*}
	 \sparsity^* = \left\{ \begin{array}{ll}
         0 \ \ \ \ \ \ & \mbox{if $u_{m_1} = 0$}\\
         \ldim & \mbox{if $|u_{m_\ldim}| > \tfrac{2\lambda C(\uvec, \ldim)}{1+2\lambda \ldim}$}\\        
          \sparsity & \mbox{such that $|u_{m_\sparsity}| > \tfrac{2\lambda C(\uvec, \sparsity)}{1+2\lambda \sparsity}$ and $|u_{m_{\sparsity+1}}| \le \tfrac{2\lambda C(\uvec, \sparsity)}{1+2\lambda \sparsity}$}.\end{array} \right. 
  \end{align*}
\end{proposition}

\begin{proof}
Let $\partial f(\cdot)$ denote the subdifferential of a function $f(\cdot)$. If $\dfrac{\uvec-\zvec^*}{\lambda} \in \partial \|\zvec^*\|_1^2$, then $\zvec^*$ is the optimum solution of $\min\limits_{\zvec \in \RR^{\ldim}} \tfrac{1}{2}\|\uvec-\zvec\|_2^2+\lambda \|\zvec\|_1^2$. Because $\partial \|\zvec\|_1^2 = 2\|\zvec\|_1 \partial \|\zvec\|_1$, we need to show that 
\begin{align*}
\dfrac{\uvec-\zvec^*}{2 \lambda \|\zvec^*\|_1} \in \partial \|\zvec^*\|_1
.
\end{align*}
Let $\lambda^* = 2 \lambda \|\zvec^*\|_1$. Using the known form for the $\ell_1$ proximal operator,
if $z^*_i = \sign{u_i}\max(\left|u_i\right|-\lambda^*,0)$, then $\dfrac{\uvec-\zvec^*}{\lambda^*} \in \partial \|\zvec^*\|_1$. 
The following cases verify that $\zvec^*$ satisfies $z^*_i = \sign{u_i}\max(\left|u_i\right|-\lambda^*,0)$.

\noindent
\textbf{Case 1.} If $\sparsity^* = 0$, then $\uvec = 0$ and so $\zvec^*=0$.

\noindent
\textbf{Case 2.}
If $\sparsity^* = k$, then for all $i$, $| u_i| > \tfrac{2\lambda C(\uvec, \ldim)}{1+2\lambda\ldim}$ and so 
$$z^*_{i} = \sign{u_{i}}\left(\left|u_i\right|-\tfrac{2\lambda C(\uvec, \ldim)}{1+2\lambda\ldim}\right) = \sign{u_{i}}\left(\left|u_i\right|-\tfrac{2\lambda\|\uvec\|_1}{1+2\lambda\ldim}\right).$$ 
Hence, 
\begin{align*}
\|\zvec^*\|_1 &= \sum_i \left| \left|u_i\right|-\tfrac{2\lambda\|\uvec\|_1}{1+2\lambda\ldim} \right|\\
&=  \sum_i \left(\left|u_i\right|-\tfrac{2\lambda\|\uvec\|_1}{1+2\lambda\ldim} \right) && \triangleright | u_i| > \tfrac{2\lambda C(\uvec, \ldim)}{1+2\lambda\ldim} \\
&=\|\uvec\|_1-\tfrac{2\lambda\ldim\|\uvec\|_1}{1+2\lambda \ldim} \\
&= \tfrac{\|\uvec\|_1}{1+2\lambda \ldim}
.
\end{align*}
Therefore, $\lambda^* = 2 \lambda \|\zvec^*\|_1 = \tfrac{2\lambda\|\uvec\|_1}{1+2\lambda\ldim}$, and so $z^*_i = \sign{u_i}\max(\left|u_i\right|-\lambda^*,0)$.

\noindent
\textbf{Case 3.}
If $0<\sparsity^*<\ldim$, then for $i>\sparsity^*$, $z^*_{m_i} = 0$ and for $i \leq \sparsity^*$, $z^*_{m_i} = \sign{u_{m_i}}\left(\left|u_{m_i}\right|-\tfrac{2\lambda C(\uvec, \sparsity^*)}{1+2\lambda\sparsity^*}\right)$. 
Similarly to above, $\|\zvec^*\|_1 =C(\uvec, \sparsity^*) -\tfrac{2\lambda\sparsity^*C(\uvec, \sparsity^*)}{1+2\lambda \sparsity^*} = \tfrac{C(\uvec, \sparsity^*)}{1+2\lambda \sparsity^*}$. and so $\lambda^* = \tfrac{2\lambda C(\uvec, \sparsity^*)}{1+2\lambda\sparsity^*}$. 
\end{proof}

\end{document}